\documentclass[twoside,11pt]{article}

\usepackage[abbrvbib]{jmlr2e}
\usepackage{algorithm}
\usepackage{algorithmic}
\usepackage{bm,color}
\usepackage{hyperref}
\usepackage{url}
\usepackage{amsmath}

\usepackage{amsthm}

\usepackage{xspace} 
\usepackage[T1]{fontenc} 
\graphicspath{ {./}{./plots/} }
\usepackage{booktabs}
\usepackage{subcaption}
\usepackage{nicefrac}
\usepackage{wrapfig}
\graphicspath{ {./}{./plots/} }
\usepackage{caption}
\usepackage{capt-of}
\DeclareCaptionLabelFormat{andtable}{#1~#2  \&  \tablename~\thetable}

\usepackage[titletoc,title]{appendix}

\newcommand{\cocoa}{\textsc{CoCoA}\xspace} 
\newcommand{\cocoap}{\textsc{CoCoA$\!^{\bf \textbf{\footnotesize+}}$}\xspace}

\newcommand{\proxcocoa}{\textsc{CoCoA}\xspace}
\newcommand{\glmnet}{\textsc{glmnet}\xspace}

\newtheorem*{rep@theorem}{\rep@title}
\newcommand{\newreptheorem}[2]{%
\newenvironment{rep#1}[1]{%
 \def\rep@title{#2 \ref{##1}}%
 \begin{rep@theorem}}%
 {\end{rep@theorem}}}
\newreptheorem{lemma}{Lemma'}
\newreptheorem{definition}{Definition'}
\newreptheorem{proposition}{Proposition'}
\newreptheorem{theorem}{Theorem'}

\makeatletter
\def\renewtheorem#1{%
  \expandafter\let\csname#1\endcsname\relax
  \expandafter\let\csname c@#1\endcsname\relax
  \gdef\renewtheorem@envname{#1}
  \renewtheorem@secpar
}
\def\renewtheorem@secpar{\@ifnextchar[{\renewtheorem@numberedlike}{\renewtheorem@nonumberedlike}}
\def\renewtheorem@numberedlike[#1]#2{\newtheorem{\renewtheorem@envname}[#1]{#2}}
\def\renewtheorem@nonumberedlike#1{  
\def\renewtheorem@caption{#1}
\edef\renewtheorem@nowithin{\noexpand\newtheorem{\renewtheorem@envname}{\renewtheorem@caption}}
\renewtheorem@thirdpar
}
\def\renewtheorem@thirdpar{\@ifnextchar[{\renewtheorem@within}{\renewtheorem@nowithin}}
\def\renewtheorem@within[#1]{\renewtheorem@nowithin[#1]}
\makeatother

\renewtheorem{theorem}{Theorem}
\renewtheorem{lemma}[theorem]{Lemma}
\newtheorem{assumption}{Assumption}
\renewtheorem{remark}{Remark}
\renewtheorem{definition}{Definition}

\DeclareMathOperator*{\argmin}{arg\,min}

\newcommand{\R}{\mathbb{R}}
\newcommand{\alphav}{ {\boldsymbol \alpha}}
\newcommand{\wv}{ {\bf w}}
\newcommand{\xv}{ {\bf x}}
\newcommand{\vv}{ {\bf v}}
\newcommand{\uv}{ {\bf u}}
\newcommand{\bv}{ {\bf b}}
\newcommand{\dv}{ {\bf d}}
\newcommand{\yv}{ {\bf y}}
\newcommand{\OB}{\mathcal{O}_{\hspace{-1pt}B}}
\newcommand{\OA}{\mathcal{O}_{\hspace{-1pt}A}}
\newcommand{\Pk}{\mathcal{P}_k}

\newcommand{\Ip}{\mathcal{I}_p}
\newcommand{\Ggk}{\mathcal{G}^{\sigma'}_k\hspace{-0.08em}}
\newcommand{\vsubset}[2]{#1_{[#2]}}
\newcommand{\eqdef}{:=}
\newcommand\tagthis{\addtocounter{equation}{1}\tag{\theequation}}
\newcommand{\aggpar}{\gamma}
\newcommand{\vc}[2]{#1^{(#2)}}                   
\newcommand{\0}{ {\bf 0}}
\newcommand{\Exp}{\mathbb{E}}
\newcommand{\E}{\mathbb{E}}                                            
\newcommand{\Lone}{L_1}
\newcommand{\Ltwo}{L_2}
\newcommand{\gap}{G}

\let\oldell\ell 
\renewcommand{\d}{m}
\newcommand{\n}{n}

\usepackage{xcolor}

\newcommand{\cocoatitle}{\proxcocoa: A General Framework for Communication-Efficient Distributed Optimization}
\usepackage{lastpage}
\jmlrheading{19}{2018}{1-\pageref{LastPage}}{10/16}{7/18}{16-512}{Virginia Smith, Simone Forte, Chenxin Ma, Martin Tak{\'a}{\v c}, Michael I. Jordan, Martin Jaggi}
\ShortHeadings{\cocoatitle}{Smith, Forte, Ma, Tak{\'a}{\v c}, Jordan, and Jaggi}
\firstpageno{1}

\begin{document}

\title{\cocoatitle}

\author{\name Virginia Smith \email smithv@stanford.edu \\
       \addr Department of Computer Science\\
       Stanford University\\
       Stanford, CA 94305, USA
       \AND
       \name Simone Forte \email simone.forte@gess.ethz.ch \\
       \addr Department of Computer Science\\
       ETH Z{\"u}rich \\
       8006 Z{\"u}rich, Switzerland
       \AND
       \name Chenxin Ma \email chm514@lehigh.edu  \\
              \name Martin Tak{\'a}{\v c} \email Takac.MT@gmail.com \\
       \addr Industrial and Systems Engineering Department \\
	Lehigh University \\
	Bethlehem, PA 18015, USA
	\AND
       \name Michael I.\ Jordan \email jordan@cs.berkeley.edu \\
       \addr Division of Computer Science and Department of Statistics\\
       University of California\\
       Berkeley, CA 94720, USA
       \AND
       \name Martin Jaggi \email martin.jaggi@epfl.ch \\
       \addr School of Computer and Communication Sciences \\
       EPFL \\
       1015 Lausanne, Switzerland}

\editor{Yoram Singer}

\maketitle

\begin{abstract}
The scale of modern datasets necessitates the development of efficient distributed optimization methods for machine learning. We present a general-purpose framework for distributed computing environments, \proxcocoa, that has an efficient communication scheme and is applicable to a wide variety of problems in machine learning and signal processing.
We extend the framework to cover general non-strongly-convex regularizers, including L1-regularized problems like lasso, sparse logistic regression, and elastic net regularization, and show how earlier work can be derived as a special case. We provide convergence guarantees for the class of convex regularized loss minimization objectives, leveraging a novel approach in handling non-strongly-convex regularizers and non-smooth loss functions. The resulting framework has markedly improved performance over state-of-the-art methods, as we illustrate with an extensive set of experiments on real distributed datasets.
\end{abstract} 
\vspace{.25em}

\section{Introduction}

Distributed computing architectures have come to the fore in modern machine learning, in response to the challenges arising from a wide range of large-scale learning applications.  Distributed architectures offer the promise of scalability by increasing both computational and storage capacities.  A critical challenge in realizing this promise of scalability is to develop efficient methods for communicating and coordinating information between distributed machines, taking into account the specific needs of machine-learning algorithms.  

On most distributed systems, the communication of data between machines is vastly more expensive than reading data from main memory and performing local computation.  Moreover, the optimal trade-off between communication and computation can vary widely depending on the dataset being processed, the system being used, and the objective being optimized.  It is therefore essential for distributed methods to accommodate flexible communication-computation profiles while still providing convergence guarantees.

Although numerous distributed optimization methods have been proposed, the mini-batch optimization approach has emerged as one of the most popular paradigms for tackling this communication-computation tradeoff~\citep[e.g.,][]{Dekel:2012wm, shalev2013accelerated,Shamir:2014tp,qu2015quartz, richtarik2016distributed}. Mini-batch methods are often developed by generalizing stochastic methods to process multiple data points at a time, which helps to alleviate the communication bottleneck by enabling more distributed computation per round of communication. However, while the need to reduce communication would suggest large mini-batch sizes, the theoretical convergence rates of these methods tend to degrade with increased mini-batch size, reverting to the rates of classical (batch) gradient methods. Empirical results corroborate these theoretical rates, and in practice, mini-batch methods have limited flexibility to adapt to the communication-computation tradeoffs that would maximally leverage parallel execution. Moreover, because mini-batch methods are typically derived from a specific single-machine solver, these methods and their associated analyses are often tailored to specific problem instances and can suffer both theoretically and practically when applied outside of their restricted setting.

In this work, we propose a framework, \proxcocoa%
\footnote{\cocoa-v1~\citep{Jaggi:2014vi} and \cocoap~\citep{Ma:2017dx,Ma:2015ti} are predecessors of this work. We continue to use the name \proxcocoa for the more general framework proposed here, and show how earlier work can be derived as a special case (Section~\ref{sec:convergence}). Portions of this newer work additionally appear in SF's master's thesis \citep{Forte:2015wv} and \cite{Smith:2015ua}.
}%
, that addresses these two fundamental limitations. First, we allow arbitrary local solvers to be used on each machine in parallel. This allows the framework to directly incorporate state-of-the-art, application-specific single-machine solvers in the distributed setting. Second, the framework shares information between machines through a highly flexible communication scheme. This allows the amount of communication to be easily tailored to the problem and system at hand, in particular allowing for the case of significantly reduced communication in the distributed environment. 

A key step in providing these features in the framework is to first define meaningful subproblems for each machine to solve in parallel, and to then combine updates from the subproblems in an efficient manner. Our method and convergence results rely on noting that, depending on the distribution of the data (e.g., by feature or by training point), and whether we solve the problem in the primal or the dual, certain machine learning objectives can be more easily decomposed into subproblems in the distributed setting. In particular, we categorize common machine learning objectives into several cases, and use duality to help decompose these objectives. Using primal-dual information in this manner not only allows for efficient methods (achieving, e.g., up to 50x speedups compared to state-of-the-art), but also allows for strong primal-dual convergence guarantees and practical benefits such as computation of the duality gap for use as an accuracy certificate and stopping criterion.

\subsection{Contributions}

\paragraph{General framework.}
We develop a communication-efficient primal-dual framework that is applicable to a broad class of convex optimization problems. Notably, in contrast to earlier work of  \cite{Yang:2013vl}, \cite{Jaggi:2014vi}, \cite{Ma:2017dx} and \cite{Ma:2015ti}, our generalized, cohesive framework: (1) specifically incorporates difficult cases of $\Lone$ regularization and other non-strongly-convex regularizers; (2) allows for the flexibility of distributing the data by either feature or training point; and (3) can be run in either a primal or dual formulation, which we show to have significant theoretical and practical implications.

\vspace{-.25em}
\paragraph{Flexible communication and local solvers.} Two key advantages of the proposed framework are its communication efficiency and ability to employ off-the-shelf single-machine solvers internally. On real-world systems, the cost of communication versus computation can vary widely, and it is thus advantageous to permit a flexible amount of communication depending on the setting at hand. Our framework provides exactly such control. Moreover, we allow arbitrary solvers to be used on each machine, which permits the reuse of existing code and the benefits from multi-core or other optimizations therein. We note that beyond the selection of the local solver and communication vs.\ computation profile, there are no required hyperparameters to tune; the provided default parameters ensure convergence and are used throughout our experiments to achieve state-of-the-art performance.

\vspace{-.25em}
\paragraph{Convergence guarantees.} 
We derive convergence rates for the framework, guaranteeing, e.g., a $\mathcal{O}(1/t)$ rate of convergence in terms of communication rounds for convex objectives with Lipschitz continuous losses, and a faster linear rate for strongly convex losses. Importantly, our convergence guarantees do not degrade with the number of machines, $K$, and allow for subproblems to be solved to arbitrary accuracies, which allows for highly flexible computation vs.\ communication profiles. Additionally, we leverage a novel approach in the analysis of primal-dual rates for non-strongly-convex regularizers. The proposed technique is an improvement over simple smoothing techniques used in, e.g., \cite{Nesterov:2005ic}, \cite{ShalevShwartz:2014dy} and \cite{Zhang:2015vj} that enforce strong convexity by adding a small $L_2$ term to the objective. Our results include primal-dual rates and certificates for the general class of linear regularized loss minimization, and we show how earlier work can be derived as a special case of our more general approach.

\vspace{-.25em}
\paragraph{Experimental comparison.}
The proposed framework yields order-of-magnitude speedups (as much as 50$\times$ faster) compared to state-of-the-art methods for large-scale machine learning. We demonstrate these gains with an extensive experimental comparison on real-world distributed datasets. We additionally explore properties of the framework itself, including the effect of running the framework in the primal or the dual, and the impact of subproblem accuracy on convergence. All algorithms for comparison are implemented in \textsf{\small Apache Spark} and run on Amazon EC2 clusters. Our code is available at: \href{http://gingsmith.github.io/cocoa/}{\texttt{gingsmith.github.io/cocoa/}}.

\section{Background and Setup}
\label{sec:setup}
 
In this paper we develop a general framework for minimizing problems of the following form:\vspace{-1mm}
\begin{equation}
\label{eq:generalobj}\tag{I}
   \ell(\uv) 
    + r(\uv) \,,
\end{equation}
for convex functions $\ell$ and $r$. Frequently the term $\ell$ is a loss function, taking the form $\sum_{i }\ell_i(\uv)$, and the term $r$ is 
a regularizer, e.g., $r(\uv) = \lambda \| \uv \|_p$. This formulation includes many popular methods in machine learning and signal processing, such as support vector machines, linear and logistic regression, lasso and sparse logistic regression, and many others.

\subsection{Definitions} The following standard definitions will be used throughout the paper.

\begin{definition}[$L$-Lipschitz Continuity]
A function $h: \R^m \to \R$ is \emph{$L$-Lipschitz continuous} if $\forall \uv,\vv \in \R^m$, we have
\begin{equation}
 | h(\uv) - h(\vv) | \leq L \| \uv-\vv \| \, .
\end{equation}
\end{definition}

\begin{definition}[$L$-Bounded Support]\label{def:lbounded}
A function $h: \R^m \to \R\cup\{+\infty\}$ has \emph{$L$-bounded support} 
if its effective domain is bounded by $L$, i.e.,
\begin{equation}
  h(\uv) < + \infty  \ \Rightarrow \  \|\uv\| \le L \, .
\end{equation}

\end{definition}

\begin{definition}[$(1/\mu)$-Smoothness]
A function $h:\R^m\rightarrow\R$ is \emph{$(1/\mu)$-smooth} if
it is differentiable and its derivative is $(1/\mu)$-Lipschitz continuous,
or equivalently
\begin{equation}
h(\uv) \leq h(\vv) + \langle \nabla h(\vv), \uv-\vv \rangle + \frac{1}{2\mu} \| \uv-\vv \|^2  \qquad\forall \uv,\vv\in\R^m \, .
\label{eq:smooth}
\end{equation}
\end{definition}

\begin{definition}[$\mu$-Strong Convexity]
A function $h:\R^m\rightarrow\R$ is \emph{$\mu$-strongly convex} for $\mu\ge0$ if\vspace{-2mm}
\begin{equation}
h(\uv) \geq h(\vv) + \langle s, \uv-\vv \rangle + \frac{\mu}{2} \| \uv-\vv \|^2  \qquad\forall \uv,\vv\in\R^m \, ,
\label{eq:strongconv}
\end{equation}
for any $s \in \partial h(\vv)$, where $\partial h(\vv)$ denotes the subdifferential of $h$ at $\vv$.
\end{definition}

\subsection{Primal-Dual Setting} Numerous methods have been proposed to solve~\eqref{eq:generalobj}, and these methods generally fall into two categories: \textit{primal methods}, which run directly on the primal objective, and \textit{dual methods}, which instead run on the dual formulation of the primal objective. In developing our framework, we present an abstraction that allows for either a primal or a dual variant of our framework to be run. In particular, to solve the input problem~\eqref{eq:generalobj}, we consider mapping the problem to one of the following two general problems:
\begin{align}
    \label{eq:primal}\tag{A}
    &\hspace{1cm}&\min_{\alphav \in \R^{n}} \quad& \Big[ \ \ 
    \OA(\alphav) \,:= \ \ f(A\alphav )
    \ +\ g(\alphav) &\Big]&&\hspace{1cm}
\\
    \label{eq:dual}\tag{B}
    &\hspace{1cm}&\min_{\wv \in \R^{\d}} \quad& \Big[ \ \ 
    \OB(\wv) :=  \ \ f^*(\wv )
    \ +\ g^*(-A^\top\wv) &\Big] \, . &&\hspace{1cm}
\end{align}
In general, our aim will be to compute a minimizer of the problem~\eqref{eq:primal} in a distributed fashion; the main difference will be whether we initially map the primal~\eqref{eq:generalobj} to~\eqref{eq:primal} or~\eqref{eq:dual}.

Here $\alphav \in \R^{\n}$ and $\wv \in \R^{\d}$ are parameter vectors, $A := [ \xv_1; \dots; \xv_{\n} ] \in \R^{\d \times \n}$
is a data matrix with column vectors $\xv_i\in\R^{\d}$, $i\in \{ 1, \dots, \n \}$, and the functions $f^*$ and $g^*_i$ are the \textit{convex conjugates} of $f$ and $g_i$, respectively. 

The dual relationship between problems~\eqref{eq:primal} and~\eqref{eq:dual} is known as Fenchel-Rockafellar duality~\cite[Theorem 4.4.2]{Borwein:2005ge}.  We provide a self-contained derivation of the duality in Appendix~\ref{app:primaldual}. Note that while dual problems are typically presented as a pair of (min, max) problems, we have equivalently reformulated~\eqref{eq:primal} and~\eqref{eq:dual} to both be minimization problems in accordance with their roles in our framework.

Given $\alphav \in \R^{\n}$ in the context of \eqref{eq:primal}, a corresponding vector $\wv\in \R^{\d}$ for problem~\eqref{eq:dual} is obtained by: 
\begin{equation}
\label{eq:dualPdualrelation}
\wv = \wv(\alphav) := \nabla f( A\alphav ) \, .
\end{equation}
This mapping arises from first-order optimality conditions on the $f$-part of the objective.
The duality gap, given by:
\begin{equation}
\label{eq:gap}
G(\alphav) := \OA(\alphav) -[-\OB(\wv(\alphav))] \, ,
\end{equation}
is always non-negative, and under strong duality, the gap will reach zero only for an optimal pair $(\alphav^\star,\wv^\star)$. 
The duality gap at any point provides a practically computable upper bound on the unknown primal as well as dual optimization error (suboptimality), since
\[
\OA(\alphav) \geq \OA(\alphav^\star) \geq -\OB(\wv^\star) \geq -\OB(\wv(\alphav)) \ .
\]

In developing the proposed framework, noting the duality between~\eqref{eq:primal} and \eqref{eq:dual} has many benefits, including the ability to compute the duality gap, which acts as a certificate of the approximation quality. 
It is also useful as an analysis tool, helping us to present a cohesive framework and relate this work to the prior work of \cite{Yang:2013vl}, \cite{Jaggi:2014vi} and \cite{Ma:2015ti,Ma:2017dx}. As a word of caution, note that we avoid prescribing the name ``primal'' or ``dual'' directly to either of the problems~\eqref{eq:primal} or~\eqref{eq:dual}, as we demonstrate below 
that their role as primal or dual can change depending on the application problem of interest. 

\subsection{Assumptions and Problem Cases}

\paragraph{Assumptions.} Our main assumptions on problem~\eqref{eq:primal} are that $f$ is $(1/\tau)$-smooth, and the function $g$ is separable, i.e., $g(\alphav) = \sum_i g_i(\alpha_i)$, with each~$g_i$ having $L$-bounded support. Given the duality between the problems~\eqref{eq:primal} and~\eqref{eq:dual}, this can be equivalently stated as assuming that in problem~\eqref{eq:dual},~$f^*$ is $\tau$-strongly convex, and the function $g^*(-A^\top\wv) = \sum_i g_i^*(-\xv_i^\top \wv)$ is separable with each~$g_i^*$ being $L$-Lipschitz.

\paragraph{Problem cases.} Suppose, as in equation~\eqref{eq:generalobj}, 
we would like to find a minimizer of the general objective $\ell(\uv) + r(\uv)$.  Depending on the smoothness of the function $\ell$ and the strong convexity of the function $r$, we will be able to map the input function~\eqref{eq:generalobj} to one (or both) of the objectives~\eqref{eq:primal} and~\eqref{eq:dual} based on our assumptions. In particular, we outline three separate cases: Case I, in which the function $\ell$ is smooth and the function $r$ is strongly convex; case II, in which $\ell$ is smooth, and $r$ is non-strongly convex and separable; and case~III, in which $\ell$ is non-smooth and separable, and $r$ is strongly convex. These cases are summarized in Table~\ref{tab:cases}. Note that the union of these cases captures most commonly-used applications of linear regularized loss minimization problems. 

In Section~\ref{sec:proxcocoa}, we will see that different variants of the framework may be realized depending on which of these three cases we consider when solving the input problem~\eqref{eq:generalobj}.

\begin{table}[h]
\caption{Criteria for Objectives~\eqref{eq:primal} and~\eqref{eq:dual}.
}
\centering
\begin{tabular}{c  c  c }
& Smooth $\ell$ & Non-smooth, separable $\ell$ \\
\toprule
Strongly convex $r$ & Case I: Obj~\eqref{eq:primal} or~\eqref{eq:dual} & Case III:~Obj~\eqref{eq:dual}\\
\midrule
Non-strongly convex, separable $r$ & Case II: ~Obj~\eqref{eq:primal} & -- \\
\bottomrule
\end{tabular}
\label{tab:cases}
\end{table}

\vspace{-1mm}
\subsection{Running Examples} 

To illustrate the cases in Table~\ref{tab:cases}, we consider several examples below. \textit{Those interested in details of the framework itself may skip to Section~\ref{sec:proxcocoa}}. These applications will serve as running examples throughout the paper, and we will revisit them in our experiments (Section~\ref{sec:experiments}). For further applications and details, see Section~\ref{sec:applications}.
\begin{enumerate}
\item \textit{Elastic Net Regression (Case I: map to either (\ref{eq:primal}) or (\ref{eq:dual})).} We can map elastic-net regularized least squares regression, 
\begin{equation}
\label{ex:elasticnet}
\min_{\uv \in \R^p}\ \tfrac{1}{2} \|A\uv-\bv\|_2^2 + \eta\lambda \|\uv\|_1 + (1-\eta)\frac{\lambda}{2}\| \uv \|_2^2  \, ,
\end{equation}
to either objective~\eqref{eq:primal} or~\eqref{eq:dual}. 
To map to objective~\eqref{eq:primal}, we let: 
$f(A\alphav)=\frac{1}{2} \|A\alphav - \bv\|_2^2$ and $g(\alphav) = \sum_i g_i(\alpha_i)=\sum_i \eta \lambda |\alpha_i| + (1-\eta) \tfrac{\lambda}{2}\alpha_i^2$, setting $\n$ to be the number of features and $\d$ the number of training points. To map to~\eqref{eq:dual}, we let: $g(-A^\top\wv) = \sum_ig_i^*(-\xv_i^\top\wv)=\sum_i\frac{1}{2}(\xv_i^\top\wv - b_i)^2$ and $f^*(\wv)=\eta\lambda \|\wv\|_1 + (1-\eta)\frac{\lambda}{2}\| \wv \|_2^2$, setting $\d$ to be the number of features and $\n$ the number of training points. We discuss in Section~\ref{sec:proxcocoa} how the choice of mapping to either~\eqref{eq:primal} or to~\eqref{eq:dual} 
can have implications on the distribution scheme and overall performance of the framework.

\item \textit{Lasso (Case II: map to (\ref{eq:primal})).} We can represent $\Lone$-regularized least squares regression by mapping the model: 
\begin{equation}
\label{ex:lasso}
\min_{\uv \in \R^p}\ \tfrac{1}{2} \|A\uv-\bv\|_2^2 + \lambda \| \uv \|_1
\end{equation}
to objective~\eqref{eq:primal},  letting $f(A\alphav) = \frac{1}{2}\|A\alphav-\bv\|_2^2$ and $g(\alphav) = \sum_i g_i(\alpha_i) = \sum_i \lambda | \alpha_i |$. 
 In this mapping, $\n$ represents the number of features, and $\d$ the number of training points. Note that we cannot map the lasso objective to~\eqref{eq:dual} directly, as $f^*$ must be $\tau$-strongly convex and the $\Lone$-norm is non-strongly convex.
 
\item \textit{Support Vector Machine (Case III: map to (\ref{eq:dual})).} We can represent a hinge loss support vector machine (SVM) by mapping the model: 
\begin{equation}
\label{ex:svm}
\min_{\uv \in \R^p}\  \frac{1}{m} \sum_{i=1}^m \max\big\{0,1-y_i(\xv_i^\top\uv)\big\} + \tfrac\lambda2 \| \uv \|_2^2 \, ,
\end{equation}
to objective~\eqref{eq:dual}, letting $g^*(-A^\top\wv) = \sum_i g_i^*(-\xv_i^\top\wv) = \sum_i \frac{1}{\n}\max\{0, 1- y_i \xv_i^\top\wv\}$ and $f^*(\wv) = \tfrac\lambda2 \| \wv \|_2^2$. In this mapping, $\d$ represents the number of features, and $\n$ the number of training points. Note that we cannot map the hinge loss SVM primal to objective~\eqref{eq:primal} directly, as $f$ must be $(1/\tau)$-smooth and the hinge loss is non-smooth.

\end{enumerate}

\subsection{Data Partitioning}
\label{sec:datapartitioning}
In this work, we are interested in the setting where the dataset at hand is distributed across multiple machines. 
We assume that the dataset~$A$ is distributed over $K$ machines according to a partition $\{\Pk\}_{k=1}^K$ of the \emph{columns} of $A \in \R^{\d\times \n}$.
We denote the size of the partition on machine $k$ by $n_k=|\Pk|$. For machine $k\in\{1, \dots, K\}$ 
 and weight vector $\alphav\in \R^{\n}$, we define $\vsubset{\alphav}{k}\in \R^{\n}$ as the $\n$-vector with elements $(\vsubset{\alphav}{k})_i := \alpha_i$ if $i\in \Pk$ and $(\vsubset{\alphav}{k})_i := 0$ otherwise. Analogously, we write $\vsubset{A}{k}$ for the corresponding group of columns of $A$, and zeros elsewhere (note that columns can correspond to either training examples or features, depending on the application). We discuss these distribution schemes in greater detail in Section~\ref{sec:proxcocoa}. 

\section{The \proxcocoa Method}
\label{sec:proxcocoa}
In the following sections, we describe the proposed framework, \proxcocoa, at a high level, and then discuss two approaches for using the framework in practice: \proxcocoa in the primal, where we consider~\eqref{eq:primal} to be the primal objective and run the framework on this problem directly, and \proxcocoa in the dual, where we instead consider~\eqref{eq:dual} to be the primal objective, and then run the framework on the dual~\eqref{eq:primal}. 

Note that in both approaches, the aim will be to compute a minimizer of the problem~\eqref{eq:primal} in a distributed fashion; the main difference will be whether we view~\eqref{eq:primal} as the primal objective or as the dual objective.

\subsection{The Generalized Framework}
\label{subsec:framework} 
The goal of the \proxcocoa framework is to find a global minimizer of the objective~\eqref{eq:primal}, while distributing computation based on the partitioning of the dataset $A$ across machines (Section~\ref{sec:datapartitioning}). As a first step, note that distributing the update to the function $g$ in objective~\eqref{eq:primal} is straightforward, as we have required that this term is separable according to the partitioning of our data, i.e., $g(\alphav) = \sum_{i=1}^n g_i(\alpha_i)$. However, the same does not hold for the term $f(A\alphav)$. To minimize this part of the objective in a distributed fashion, we propose minimizing a quadratic approximation of the function, which allows the minimization to separate across machines. We make this approximation precise in the following subsection.

\paragraph{Data-local quadratic subproblems.}
In the general \proxcocoa framework (Algorithm~\ref{alg:generalizedcocoa}), 
\begin{algorithm}[t]
\caption{Generalized \proxcocoa Distributed Framework}
\label{alg:generalizedcocoa}
\begin{algorithmic}[1]
\STATE {\bf Input:} Data matrix $A$ distributed column-wise according to partition $\{\Pk\}_{k=1}^K$, aggregation parameter $\aggpar\!\in\!(0,1]$, 
and parameter $\sigma'$ for the local subproblems
$\Ggk(  \vsubset{\Delta \alphav}{k}; \vv, \vsubset{\alphav}{k})$.\\
Starting point $\vc{\alphav}{0} := \0 \in \R^n$, $\vc{\vv}{0}:=\0\in \R^\d$.
\FOR {$t = 0, 1, 2, \dots $}
  \FOR {$k \in \{1,2,\dots,K\}$ {\bf in parallel over computers}}
     \STATE call local solver, returning a $\Theta$-approximate solution 
     $\vsubset{\Delta \alphav}{k}$   
        of  the local subproblem~\eqref{eq:subproblem}
     \STATE update local variables $\vsubset{\vc{\alphav}{t+1}}{k} := \vsubset{\vc{\alphav}{t}}{k} + \aggpar \, \vsubset{\Delta \alphav}{k}$
     \STATE return updates to shared state $\Delta \vv_k :=  \vsubset{A}{k} 
     						\vsubset{\Delta \alphav}{k}$
  \ENDFOR
  \STATE reduce 
$\vc{\vv}{t+1}  := \vc{\vv}{t} +
  \aggpar \textstyle \sum_{k=1}^K \Delta \vv_k $
\ENDFOR 
\end{algorithmic}
\end{algorithm}
we distribute computation by defining a data-local subproblem of the optimization problem~\eqref{eq:primal} for each machine. This simpler problem can be solved on machine $k$ and only requires accessing data which is already available locally, i.e., the columns $\vsubset{A}{k}$
. More formally, each machine~$k$ is assigned the following local subproblem, which depends only on the previous shared vector $\vv := A\alphav \in\R^{\d}$, and the local data $\vsubset{A}{k}$: 
\begin{equation} 
 \label{eq:subproblem}
\min_{\vsubset{\Delta \alphav}{k}\in\R^{\n}} \ 
\Ggk(  \vsubset{\Delta \alphav}{k}; \vv, \vsubset{\alphav}{k}) \, , \vspace{-1mm}
\end{equation} 
where
\begin{align*}
\Ggk(  \vsubset{\Delta \alphav}{k}; \vv, \vsubset{\alphav}{k})
\eqdef  \frac{1}{K} f(\vv) 
 + \wv^\top \vsubset{A}{k} \vsubset{\Delta \alphav}{k}
+\frac{\sigma'}{2\tau}  \Big\|\vsubset{A}{k} \vsubset{\Delta \alphav}{k}\Big\|^2 + \sum_{i \in \Pk} 
g_i(\alpha_i + {\vsubset{\Delta \alphav}{k}}_i),
\end{align*}
and $\wv := \nabla f( \vv )$. 
Here we let $\vsubset{\Delta \alphav}{k}$ denote the change of local variables~$\alpha_i$ for indices $i\in\Pk$, and we set $(\vsubset{\Delta \alphav}{k})_i := 0$ for all $i \notin \Pk$. It is important to note that the subproblem~\eqref{eq:subproblem} is simple in the sense that it is always a quadratic objective (apart from the $g_i$ term). The subproblem does not depend on the function~$f$ itself, but only its linearization 
at the fixed shared vector $\vv$. This property additionally simplifies the task of the local solver, especially for cases of complex functions~$f$. 

\paragraph{Framework parameters $\gamma$ and $\sigma'$.}
There are two parameters that must be set in the framework: $\gamma$, the aggregation parameter, which controls how the updates from each machine are combined, and $\sigma'$, the subproblem parameter, which is a data-dependent term measuring the difficulty of the data partitioning $\{\Pk\}_{k=1}^K$. These terms play a crucial role in the convergence of the method, as we demonstrate in Section~\ref{sec:convergence}. In practice, we provide a simple and robust way to set these parameters: For a given aggregation parameter $\gamma \in (0,1]$, the subproblem parameter $\sigma'$ will be set as $\sigma' := \gamma K$, but can also be improved in a data-dependent way as we discuss below. \textit{In general, as we show in Section~\ref{sec:convergence}, setting $\gamma := 1$ and $\sigma' := K$ will guarantee convergence while delivering our fastest convergence rates.}

\begin{definition}[Data-dependent aggregation parameter] 
\label{def:sigma}
In Algorithm~\ref{alg:generalizedcocoa}, the aggregation parameter~$\gamma$ controls the level of adding ($\gamma:=1$) versus averaging ($\gamma:=\tfrac{1}{K}$) of the partial solutions from all machines.
For our convergence results (Section~\ref{sec:convergence}) to hold, the subproblem parameter $\sigma'$ must be chosen not smaller than
\begin{equation}
\label{eq:sigmaPrimeSafeDefinition} 
\sigma'
\geq
\sigma'_{min}
 \eqdef
 \aggpar
 \max_{\alphav\in \R^{\n}}
 \frac{
 \|A\alphav\|^2}{\sum_{k=1}^K \|\vsubset{A}{k}\vsubset{\alphav}{k}\|^2} \, . \vspace{-2mm}
\end{equation}
\label{def:gamma}
\end{definition}
The simple choice of $\sigma' := \gamma K$ is valid for \eqref{eq:sigmaPrimeSafeDefinition}, i.e., 
\[
\gamma K \geq \sigma'_{min} \, .
\] In some cases, it will be possible to give a better (data-dependent) choice for $\sigma'$, closer to the actual bound given in $\sigma'_{min}$.

\paragraph{Subproblem interpretation.} Here we provide further intuition behind the data-local subproblems~\eqref{eq:subproblem}. The local objective functions $\Ggk$ are defined to closely approximate the global objective in \eqref{eq:primal} as the ``local'' variable~$\vsubset{\Delta \alphav}{k}$ varies, which we will see in the analysis (Appendix~\ref{app:convgproofs}, Lemma~\ref{lem:RelationOfDTOSubproblems}).  
In fact, if the subproblem were solved exactly, this could be interpreted as a data-dependent, block-separable proximal step, applied to the $f$ part of the objective~\eqref{eq:primal} as follows:
\begin{align}
\label{eq:interpretation}
\sum_{k=1}^K 
 \Ggk(\vsubset{\Delta \alphav}{k}; \vv, \vsubset{\alphav}{k}) = R{+}f(\vv){+}\nabla f( \vv )^\top A\Delta \alphav 
{+}\frac{\sigma'}{2\tau} \Delta \alphav^\top 
\begin{bmatrix}
    A_{[1]}^\top A_{[1]}\vspace{-1mm} &   & 0 \\
      &\hspace{-4mm}\ddots&  \\
    0 &   &\hspace{-4mm}A_{[K]}^\top A_{[K]}
\end{bmatrix}
\Delta \alphav,
\end{align}
where $R = \sum_{i \in [\n]} g_i(-\alpha_i - \Delta \alpha_i) \, .$ 

However, note that in contrast to traditional proximal methods, \proxcocoa does \emph{not} assume that this subproblem is solved to high accuracy, as we instead allow the use of local solvers of any approximation quality $\Theta$. 

\paragraph{Reusability of existing single-machine solvers.}
The local subproblems \eqref{eq:subproblem} have the appealing property of being very similar in structure to the global problem~\eqref{eq:primal}, with the main difference being that they are defined on a smaller (local) subset of the data, and are simpler because they are not dependent on the shape of $f$.
For a user of \proxcocoa, this presents a significant advantage in that existing single machine-solvers can be directly re-used in our distributed framework (Algorithm~\ref{alg:generalizedcocoa}) by employing them on the subproblems~$\Ggk$.

Therefore, problem-specific tuned solvers which have already been developed, along with associated speed improvements (such as multi-core implementations), can be easily leveraged in the distributed setting. We quantify the dependence on local solver performance with the following assumption and remark, and relate this performance to our global convergence rates in Section~\ref{sec:convergence}. 

\begin{assumption}[$\Theta$-approximate solution]
\label{asm:theta}
We assume that  there exists $\Theta \in [0,1)$ such that 
$\forall k\in [K]$, 
the local solver at any outer iteration $t$ produces
a (possibly) randomized approximate solution $\vsubset{\Delta \alphav}{k}$,
which satisfies
\begin{align}
\vspace{-1em}
\label{eq:localSolutionQuality}
 \frac{\E \big[
 \Ggk(\vsubset{\Delta \alphav}{k};\vv, \vsubset{\alphav}{k})\! - \Ggk(\vsubset{\Delta \alphav^{\star}}{k};\vv, \vsubset{\alphav}{k}) \big]}{
 ~~\Ggk(~\0~;\vv,\vsubset{\alphav}{k}) - \Ggk(\vsubset{\Delta \alphav^{\star}}{k};\vv, \vsubset{\alphav}{k})} \leq \Theta 
 \, ,
\end{align}
\vspace{-.25em}
where
\vspace{-1em}
\begin{align}
\label{eq:asjfcowjfcaw}
\vsubset{\Delta \alphav^{\star}}{k}
\in \argmin_{\Delta \alphav \in \R^{\n}} \ 
 \Ggk(\vsubset{\Delta \alphav}{k};\vv, \vsubset{\alphav}{k}), \hspace{2mm} \forall k\in[K] \, . 
\end{align}
\end{assumption} 

\begin{remark}\label{rem:localtime}
In practice, the time spent solving the local subproblems in parallel should be chosen comparable to the time of a communication round, for best overall efficiency on a given system. 
We study this trade-off in theory (Section \ref{sec:convergence}) and experiments (Section \ref{sec:experiments}).
\end{remark}
\begin{remark}\label{rem:localqual}
Note that the accuracy parameter $\Theta$ does not have to be chosen \emph{a priori}: Our convergence results (Section~\ref{sec:convergence}) are valid if $\Theta$ is an upper bound on the actual empirical values $\Theta$ in Algorithm \ref{alg:generalizedcocoa}. This allows for some of the $K$ machines to at times deliver better or worse accuracy $\Theta$ (e.g., this would allow a slow local machine to be stopped early during a specific round in order to avoid stragglers). See \citep{smith2017federated} for more details.
\end{remark}

\begin{remark}\label{rem:localtheta}
From a theoretical perspective, the multiplicative notion of accuracy is advantageous over classical additive accuracy as existing convergence results for first- and second-order optimization methods typically appear in multiplicative form, i.e., relative to the error at the initialization point (here $\vsubset{\Delta \alphav}{k}=\0$). This accuracy notion $\Theta$ is also useful beyond the distributed setting \citep[see, e.g., ][]{praneeth2018balancing,praneeth2018newton}. We discuss local solvers and associated rates to achieve  accuracy $\Theta$ for particular applications in Section~\ref{sec:applications}.
\end{remark}

With this general framework in place, we next discuss two variants of our framework, \cocoa-Primal and \cocoa-Dual. In running either the primal or dual variant of the framework, the goal will always be to solve objective~\eqref{eq:primal} in a distributed fashion. The main difference will be whether this objective is viewed as the primal or dual of the input problem~\eqref{eq:generalobj}. 
We make this mapping technique precise and discuss its implications in the following subsections (Sections~\ref{sec:primal}--\ref{sec:primalvsdual}).

\subsection{Primal Distributed Optimization}
\label{sec:primal}
In the primal distributed version of the framework (Algorithm~\ref{alg:primal}), the  framework is run by mapping the initial problem~\eqref{eq:generalobj} directly to objective~\eqref{eq:primal} and then applying the generalized \proxcocoa framework described in Algorithm~\ref{alg:generalizedcocoa}. In other words, we view problem~\eqref{eq:primal} as the primal objective, and solve this problem directly. 

From a theoretical perspective, viewing~\eqref{eq:primal} as the primal will allow us to consider non-strongly convex regularizers, since we allow the terms $g_i$ to be non-strongly convex. This setting was not covered in earlier work of \cite{Yang:2013vl,Jaggi:2014vi,Ma:2015ti}; and \cite{Ma:2017dx}, and we discuss it in detail in Section~\ref{sec:convergence}, as additional machinery must be introduced to develop primal-dual rates for this setting.

Running the primal version of the framework has important practical implications in the distributed setting, as it typically implies that the data is distributed by feature rather than by training point. In this setting, the amount of communication at every outer iteration will be $\mathcal{O}(\#$ of training points$)$. When the number of features is high (as is common when using sparsity-inducing regularizers) this can help to reduce communication and improve overall performance, as we demonstrate in Section~\ref{sec:experiments}.

\begin{minipage}{.91\textwidth}
\begin{algorithm}[H]
\caption{\cocoa-Primal (Mapping Problem \eqref{eq:generalobj} to  \eqref{eq:primal})}
\label{alg:primal}
\begin{algorithmic}[1]
\STATE{\bf Map:} Input problem~\eqref{eq:generalobj} to objective~\eqref{eq:primal} 
\STATE{\bf Distribute:} Dataset $A$ by columns (here typically features) according to partition $\{\Pk\}_{k=1}^K$
\STATE{\bf Run:} Algorithm~\ref{alg:generalizedcocoa} with aggregation parameter $\gamma$ and subproblem parameter $\sigma'$
\end{algorithmic}
\end{algorithm}
\end{minipage}
\vspace{1em}

\subsection{Dual Distributed Optimization}
\label{sec:dual}
In the dual distributed version of the framework (Algorithm~\ref{alg:dual}), we run the framework by mapping the original problem~\eqref{eq:generalobj} to objective~\eqref{eq:dual}, and then solve the problem by running Algorithm~\ref{alg:generalizedcocoa} on the dual~\eqref{eq:primal}. In other words, we view problem~\eqref{eq:dual} as the primal, and solve this problem via the dual~\eqref{eq:primal}.

This version of the framework will allow us to consider non-smooth losses, such as the hinge loss or absolute deviation loss, since the terms $g_i^*$ can be non-smooth.
From a practical perspective, this version of the framework will typically imply that the data is distributed by training point, and for a vector $\mathcal{O}(\#$ of features$)$ to be communicated at every outer iteration. This variant may therefore be preferable when the number of training points exceeds the number of features.

\begin{minipage}{.91\textwidth}
\begin{algorithm}[H]
\caption{\cocoa-Dual (Mapping Problem \eqref{eq:generalobj} to  \eqref{eq:dual})}
\label{alg:dual}
\begin{algorithmic}[1]
\STATE{\bf Map:} Input problem~\eqref{eq:generalobj} to objective~\eqref{eq:dual} 
\STATE{\bf Distribute:} Dataset $A$ by columns (here typically training points) according to partition $\{\Pk\}_{k=1}^K$
\STATE{\bf Run:} Algorithm~\ref{alg:generalizedcocoa} with aggregation parameter $\gamma$ and subproblem parameter $\sigma'$
\end{algorithmic}
\end{algorithm}
\end{minipage}

\subsection{Primal vs. Dual}
\label{sec:primalvsdual}
In Table~\ref{tab:algs}, we revisit the three cases from Section~\ref{sec:setup}, showing how the primal and dual variants of \cocoa can be applied to various input problems $\ell(\uv) + r(\uv)$, depending on properties of the functions $\ell$ and $r$. In particular, in the setting where $\ell$ is smooth and~$r$ is strongly convex, the user may choose whether to run the framework in the primal (Algorithm~\ref{alg:primal}), or in the dual (Algorithm~\ref{alg:dual}). 

Intuitively, Algorithm~\ref{alg:primal} will be preferable as $r$ loses strong convexity, and Algorithm~\ref{alg:dual} will be preferable as $\ell$ loses smoothness. However, there are also systems-related aspects to consider. In Algorithm~\ref{alg:primal}, we typically distribute the data by feature, and in Algorithm~\ref{alg:dual}, by training point (this distribution depends on how the terms $n$ and $\d$ are defined in our mapping; see Section~\ref{sec:applications}). Depending on whether the number of features or number of training points is the dominating term, we may chose to run Algorithm~\ref{alg:primal} or Algorithm~\ref{alg:dual}, respectively, in order to reduce communication costs. We validate these ideas empirically in Section~\ref{sec:experiments} by comparing the performance of each variant (primal vs. dual) on real distributed datasets.

\begin{table}[h!]
\caption{Criteria for Running Algorithms~\ref{alg:primal} vs.~\ref{alg:dual}.}
\centering
\begin{tabular}{c  c  c }
& Smooth $\ell$ & Non-smooth and separable $\ell$ \\
\toprule
Strongly convex $r$ & Case I:~Alg.~\ref{alg:primal} or~\ref{alg:dual} & Case III:~Alg.~\ref{alg:dual}\\
\midrule
Non-strongly convex and separable $r$ & Case II: ~Alg.~\ref{alg:primal} & -- \\
\bottomrule
\end{tabular}
\label{tab:algs}
\end{table}

In the following subsection, we provide greater insight into the \proxcocoa framework and its relation to prior work. An extended discussion on related work is available in Section~\ref{sec:relatedwork}.

\subsection{Interpretations of \proxcocoa}
\label{subsec:interpretations}
There are numerous methods that have been developed to solve~\eqref{eq:primal} and~\eqref{eq:dual} in parallel and distributed environments. We describe related work in detail in Section~\ref{sec:relatedwork}, and here briefly position \proxcocoa and in relation to other widely-used parallel and distributed methods.

\paragraph{\proxcocoa in the context of classical parallelization schemes.}
We first contrast \proxcocoa with common distributed mini-batch and batch methods, such as mini-batch stochastic gradient descent or coordinate descent, gradient descent, and quasi-Newton methods.

\proxcocoa is similar to these methods in that they are all \textit{iterative}, i.e., they make progress towards the optimal solution by updating the parameter vector $\alphav$ according to some function $h$: $\R^{\n}{\to} \R^{\n}$ at each iteration $t$:
$
\alphav^{(t+1)}{=}h(\alphav^{(t)}), \, t{=}0, 1, \dots ,
$
until convergence is reached. 
From a coordinate-wise perspective, two approaches to update $\alphav$ iteratively include the Jacobi ``all-at-once'' and Gauss-Seidel ``one-at-a-time'' methods~\citep{Bertsekas:1989}: 
\begin{align*}
\text{Jacobi:} & \quad \alpha_i^{(t+1)} = h_i(\alpha_1^{(t)}, \dots, \alpha_{\n}^{(t)}), \quad i = 1, \dots, \n, \\
\text{Gauss-Seidel:} & \quad \alpha_i^{(t+1)} = h_i(\alpha_1^{(t+1)}, \dots, \alpha_{i-1}^{(t+1)}, \alpha_i^{(t)}, \dots, \alpha_{\n}^{(t)}), \quad i = 1, \dots, \n. 
\end{align*}
The Jacobi method does not require information from the other coordinates to update coordinate $i$, which makes this style of method well-suited for parallelization. However, the sequential Gauss-Seidel-style method tends to converge faster in terms of iterations, as it is able to incorporate information from the updates of other coordinates more quickly. This difference is well-known and evident in single machine solvers, where stochastic methods (benefiting from fresh updates) tend to outperform their batch counterparts.

Typical mini-batch methods, e.g., mini-batch coordinate descent, perform a Jacobi-style update on a subset of the coordinates at each iteration. This makes these methods amenable to high levels of parallelization. However, they are unable to incorporate information as quickly as their serial counterparts in terms of number of data points accessed, as synchronization is required before updating the coordinates. As the size of the mini-batch grows, this can increase the runtime and even lead to divergence \citep{richtarik2016distributed}.

\proxcocoa instead aims to combine attractive properties of both of these update paradigms. In \cocoa, Jacobi-style updates are applied in parallel to \textit{blocks} of the coordinates of $\alphav$ to distribute the method, while allowing for (though not necessarily requiring) faster Gauss-Seidel-style updates on each machine. This change in parallelization scheme is one of the key reasons for improved performance over simpler mini-batch or batch style methods.  

\paragraph{Two extremes: from distributed CD to one-shot communication.} In addition to the parallel block-Jacobi updating scheme described above, \proxcocoa incorporates an additional level of flexibility by allowing for an \textit{arbitrary number} of sequential Gauss-Seidel iterations (or any local solver for that matter) to be performed locally on each machine. This flexibility is critical in the distributed setting, as one of the key indicators of parallel efficiency is the time spent on local computation vs. communication. In particular, the flexibility to solve each subproblem to arbitrary accuracy, $\Theta$, allows \proxcocoa to scale from low-communication environments, where more iterations can be performed before communicating, to high communication environments, where fewer local iterations are necessary. 

In comparison with other distributed methods, this flexibility also affords an explanation of \cocoa as a method that can freely move between two extremes. On one extreme, if the subproblems~\eqref{eq:subproblem} are solved exactly, \proxcocoa recovers block coordinate descent, where the coordinate updates are applied as part of a block-separable proximal step~\eqref{eq:interpretation}. If only one outer round of communication is performed, this is similar in spirit to one-shot communication schemes, which attempt to completely solve for and then combine locally-computed models~\citep[see, e.g.,][]{Mann:2009tr,Zhang:2013wq,Heinze:2016tu}. While these one-shot communication schemes are ideal in terms of reducing communication, they are, in contrast to \cocoa, generally not guaranteed to converge to the optimal solution.

On the other extreme, if just a single update (i.e., with respect to one coordinate $\alpha_i$) is performed at each communication round, this recovers traditional distributed coordinate descent. In comparison to  \cocoa, vanilla distributed coordinate descent can suffer from a high communication bottleneck due to the low relative amount of local computation. Even in the case of mini-batch coordinate descent, the most amount of work that can be performed locally at each round includes a single pass through the data, whereas \cocoa has the flexibility to take multiple passes. We empirically compare to mini-batch distributed coordinate descent in Section~\ref{sec:experiments} to demonstrate the effect of this issue in practice.

\paragraph{Comparison to ADMM.}
\label{sec:admm}

Finally, we provide a direct comparison between \proxcocoa and the alternating direction method of multipliers  (ADMM), a well-established framework for distributed optimization~\citep{Boyd:2010bw}. Similar to \proxcocoa, ADMM defines a subproblem for each machine to solve in parallel, rather than parallelizing a mini-batch update. ADMM also leverages duality structure, similar to that presented in Section~\ref{sec:setup}.  For consensus ADMM~\citep{Mota:2013ja},~\eqref{eq:dual} is decomposed with a re-parameterization:
\begin{align*}
\max_{\wv_1, \dots \wv_K, \wv} & \quad \sum_{k=1}^K \sum_{i \in \Pk} g^*(-\xv_i^\top\wv_k) + f^*(\wv) \quad \text{s.t.} \, \, \wv_k = \wv, \, \, k = 1, \dots, K.
\end{align*}
This problem is then solved by constructing the augmented Lagrangian, which yields the following decomposable updates:
\begin{align*}
\wv_k^{(t+1)} & = \argmin_{\wv_k} \, \sum_{i \in \Pk}  g_i^*(-\xv_i^\top\wv_k) + \rho{\uv_k^{(t)}}^\top(\wv_k - \wv^{(t)}) +  \frac{\rho}{2}\|\wv_k - \wv^{(t)}\|^2, \\
\wv^{(t+1)} & = \argmin_{\wv} \, f^*(\wv) + \frac{\rho K}{2} \|\wv - (\bar{\wv}_k^{(t+1)}+\bar{\uv}_k^{(t)})\|^2, \\
\uv_{k}^{(t+1)} & = \uv_{k}^{(t)} + \wv_{k}^{(t+1)} - \wv^{(t+1)} ,
\end{align*}
where $\rho$ is a penalty parameter that must be tuned for best performance. It can be shown that the update to $\wv_k$ can be reformulated in terms of the conjugate functions $g_i$ as:
\begin{align*}
\label{obj:subproblemadmm}
\argmin_{\vsubset{\alphav}{k}}\sum_{i \in \Pk} g_i({\vsubset{\alphav}{k}}_i) +(\wv^{(t)} - \uv_k^{(t)})^\top A_{[k]}\vsubset{\alphav}{k} + \frac{1}{2\rho}\| A_{[k]}\vsubset{\alphav}{k}\|^2  \, .
\tagthis
\end{align*}
Thus, we see that the update to $\wv_k$ closely matches the \proxcocoa subproblem~\eqref{eq:subproblem}, where $\rho := \frac{\tau}{\sigma'}$. This is intuitive as both methods use proximal steps to derive the subproblem, and a similar result can be shown when applying ADMM to the~\eqref{eq:primal} formulation, which can be seen as an instantiation of the sharing variant of ADMM~\cite[Section 7.3]{Boyd:2010bw}.

However, there remain major differences between the methods despite this connection. First, \proxcocoa has a more direct and simplified scheme for updating the global weight vector~$\wv$, as the additional proximal step is not required. Second, in \proxcocoa, there is no need to tune any parameters such as $\rho$, as the method can be run simply using $\sigma'$$=$$K$. Finally, in the \proxcocoa method and theory, the subproblem can be solved approximately to any accuracy $\Theta$, rather than requiring a full batch update as in ADMM. We will see in our experiments that these differences have a substantial impact in practice (Section~\ref{sec:experiments}). We provide a full derivation of the comparison to ADMM for reference in Appendix \ref{app:admm}.

\section{Convergence Analysis}
 \label{sec:convergence}
 
In this section, we provide convergence rates for the proposed framework and introduce a key theoretical technique in analyzing non-strongly convex terms in the primal-dual setting. 

For simplicity of presentation, we assume in the analysis that the data partitioning is balanced, i.e., $n_k = n/K$ for all $k$. Furthermore, we assume that the columns of A satisfy $\| \xv_i \| \le 1$ for all $i \in [n]$, and $\OB$ contains the average term $\frac{1}{n}\sum_{i=1}^n g^*_i(\cdot)$, as is common in ERM-type problems. We present rates for the case where $\gamma:=1$ in Algorithm~\ref{alg:generalizedcocoa}, and where the subproblems~\eqref{eq:subproblem} are defined using the corresponding safe bound $\sigma':=K$. This case will guarantee convergence while delivering our fastest rates in the distributed setting, which in particular do not degrade as the number of machines $K$ increases and $n$ remains fixed. More general rates and all proof details can be found in the appendix.

\subsection{Proof Strategy: Relating Subproblem Approximation to Global Progress}

To guarantee convergence, it is critical to show how progress made on the local subproblems~\eqref{eq:subproblem} relates to the global objective~$\OA$. 
Our first lemma provides exactly this information. In particular, we see that  if the aggregation and subproblem parameters are selected according to Definition~\ref{def:sigma}, the sum of the subproblem objectives, $\sum_{k=1}^K\Ggk$, will form a block-separable upper bound on the global objective~$\OA$.
\begin{lemma}
\label{lem:RelationOfDTOSubproblems}
For any weight vector 
$\alphav, \Delta \alphav 
\in \R^{\n}$, $\vv = \vv(\alphav) := A\alphav$, and real values $\aggpar,\sigma'$ satisfying~\eqref{eq:sigmaPrimeSafeDefinition}, it holds that
\begin{equation}
  \OA\Big(
\alphav +\aggpar 
\sum_{k=1}^K
\vsubset{\Delta \alphav}{k}\!
\Big) 
 \leq 
 (1-\aggpar) \OA(\alphav)  + \aggpar 
 \sum_{k=1}^K 
 \Ggk(\vsubset{\Delta \alphav}{k}; \vv, \vsubset{\alphav}{k}) \, .
\end{equation}
\end{lemma}

A proof of Lemma~\ref{lem:RelationOfDTOSubproblems} is provided in Appendix~\ref{app:convgproofs}. We use this main lemma, in combination with our measure of quality of the subproblem approximations (Assumption~\ref{asm:theta}), to deliver global convergence rates.

\subsection{Rates for General Convex $g_i$, $L$-Lipschitz $g^*_i$}
Our first main theorem provides convergence guarantees for objectives with general convex~$g_i$ (or, equivalently, $L$-Lipschitz $g^*_i$), including models with non-strongly convex regularizers such as lasso and sparse logistic regression, or models with non-smooth losses, such as the hinge loss support vector machine.

\begin{theorem}
\label{thm:convergenceNonsmooth}
Consider Algorithm \ref{alg:generalizedcocoa} with $\gamma :=1$, 
and let $\Theta$ be the quality of the local solver as in Assumption~\ref{asm:theta}.
Let~$g_i$ have $L$-bounded support, 
and let $f$ be $(1/{\tau})$-smooth. 
Then after~$T$ iterations, where\vspace{-2mm}
\begin{align}\label{eq:dualityRequirements}
& T
\geq
T_0 +
 \max\{\Big\lceil \frac1{1-\Theta}\Big\rceil,
\frac{4L^2}{\tau \epsilon_{G}(1-\Theta)}\} \, ,
\\
& T_0
\geq t_0+
\Big[
\frac{2}{ 1-\Theta }
\left(\frac {8L^2} {\tau \epsilon_{ G}}
-1
\right)
\Big]_+  \, , \notag \\
& t_0  \geq
  \max(0,\Big\lceil \tfrac1{(1-\Theta)}
\log \left(
\tfrac{
 \tau n({\OA}(\vc{\alphav}{0})-{\OA}(\alphav^{\star} ))
  }{2 L^2 K}
  \right)
 \Big\rceil)\, ,\notag
\end{align}
we have that the expected duality gap satisfies 
\[
\Exp\big[\OA(\overline \alphav) - (-\OB( \wv(\overline\alphav))) \big] \leq \epsilon_{ G} \, , 
\]
where $\overline\alphav$ is the averaged iterate: $\frac{1}{T-T_0}\sum_{t=T_0+1}^{T-1} \alphav^{(t)}$.

\end{theorem}

Providing primal-dual rates and globally defined primal-dual accuracy certificates for these objectives may require a theoretical technique that we introduce below, in which we show how to satisfy the notion of $L$-bounded support for $g_i$, as stated in Definition~\ref{def:lbounded}. 

\subsubsection{Bounded support modification}\label{sec:boundedSup}
Additional work is necessary if Theorem~\ref{thm:convergenceNonsmooth} is to be applied to non-strongly convex regularizers such as the $\Lone$ norm, which do not have $L$-bounded support for each~$g_i$, and thus violate the main assumptions.
Note for example that the conjugate function of $g_i=|\cdot|$, which is the indicator function of an interval, is not defined globally over $\R$, and thus (without further modification) the duality gap $G(\alphav)$:=$\OA(\alphav)$-$($-$\OB(\wv(\alphav)))$ is not defined at all points~$\alphav$.

\paragraph{Smoothing.} To address this problem, existing approaches typically use a simple smoothing technique \citep[e.g.,][]{Nesterov:2005ic,ShalevShwartz:2014dy}: by adding a small amount of $\Ltwo$ regularization, the functions $g_i$ become strongly convex. Following this change, the methods are run on the dual instead of the original primal problem.
While this modification satisfies the necessary assumptions for convergence, this smoothing technique is often undesirable in practice, as it changes the iterates, the algorithms at hand, the convergence rate, and the tightness of the resulting duality gap compared to the original objective. Further, the amount of smoothing can be difficult to tune and has a large impact on empirical performance. We perform experiments to highlight these issues in practice in Section~\ref{sec:experiments}. 

\paragraph{Bounded support modification.}
In contrast to smoothing, our approach preserves all solutions of the original objective, leaves the iterate sequence unchanged, and allows for direct reusability of existing solvers for the original $g_i$ objectives (such as $L_1$ solvers). It also removes the need for tuning a smoothing parameter. 
To achieve this, we modify the function $g_i$ by imposing an additional weak constraint 
 that is inactive in our region of interest. Formally, we replace $g_i(\alpha_i)$ by the following modified function:\vspace{-.5mm}
\begin{equation}\label{eq:boundedSup}
\bar{g_i}(\alpha_i) :=
\begin{cases}
g_i(\alpha_i) 
&: \alpha_i \in [-B,B] \\
+\infty &: \text{otherwise.}
\end{cases}
\end{equation}
For large enough $B$, this problem yields the same solution as the original objective. Note also that this only affects convergence theory, in that it allows us to present a strong primal-dual rate (Theorem~\ref{thm:convergenceNonsmooth} for $L$=$B$). The modification of~$g_i$ does not affect the algorithms for the original problems. Whenever a monotone optimizer is used, we will never leave the level set defined by the objective at the starting point.

Using the resulting modified function will allow us to apply the results of Theorem~\ref{thm:convergenceNonsmooth} for general convex functions $g_i$.
This technique can also be thought of as ``Lipschitzing'' the dual~$g^*_i$, because of the general result that $g^*_i$ is $L$-Lipschitz if and only if $g_i$ has $L$-bounded support \cite[Corollary 13.3.3]{Rockafellar:1997ww}.
We derive the conjugate function $\bar{g}_i^*$ for completeness in Appendix~\ref{app:primaldual} (Lemma~\ref{lem:l1surrogate}).  
In Section~\ref{sec:applications}, we show how to leverage this technique for a variety of application input problems. See also \cite{Dunner:2016vga} for a follow-up discussion of this technique in the non-distributed case.

\subsection{Rates for Strongly Convex $g_i$, Smooth $g_i^*$} 

For the case of objectives with strongly convex $g_i$ (or, equivalently, smooth $g_i^*$), e.g., elastic net regression or logistic regression, we obtain the following faster \emph{linear} convergence rate. 
\begin{theorem}
\label{thm:convergenceSmooth}
Consider Algorithm~\ref{alg:generalizedcocoa} with $\gamma := 1$,
and let $\Theta$ be the quality of the local solver as in Assumption~\ref{asm:theta}.
Let~$g_i$ be  $\mu$-strongly convex $\forall i\in[\n]$, 
and let~$f$ be $(1/{\tau})$-smooth.
Then after~$T$ iterations where
\begin{equation}
T \geq
\tfrac{1}
   {(1-\Theta)}
\tfrac
{\mu\tau+1}
{ \mu \tau}
    \log \tfrac 1 {\epsilon_{\OA}} \, , 
\end{equation}
it holds that
\begin{equation*}
\E\big[\OA(\vc{\alphav}{T}) - \OA(\alphav^{\star})\big] \le \epsilon_{\OA} \, .
\end{equation*}

\noindent Furthermore, after $T$ iterations with
\[ T \geq \tfrac{1}{(1-\Theta)}\tfrac{\mu\tau+1}{ \mu \tau}\log \left( \tfrac{1}{(1-\Theta)}\tfrac {\mu\tau+
1}{ \mu \tau}\tfrac 1 {\epsilon_\gap}\right) \, , \]
\noindent we have the expected duality gap
\[ \Exp\big[\OA(\vc{\alphav}{T}) -(-\OB( \wv(\vc{\alphav}{T}))\big] \leq \epsilon_\gap \, .\]
\end{theorem}
We provide proofs of both Theorem~\ref{thm:convergenceNonsmooth} and Theorem~\ref{thm:convergenceSmooth} in Appendix~\ref{app:convgproofs}.

\subsection{Convergence Cases}
Revisiting Table~\ref{tab:cases} from Section~\ref{sec:setup}, we summarize our convergence guarantees for the three cases of input problems~\eqref{eq:generalobj} in the following table. In particular, we see that for cases II and III, we obtain a sublinear convergence rate, whereas for case I we can obtain a faster linear rate, as provided in Theorem~\ref{thm:convergenceSmooth}.

\begin{table}[h]
\caption{Applications of Convergence Rates.}
\centering
\begin{tabular}{c  c  c }
& Smooth $\ell$ & Non-smooth, separable $\ell$ \\
\toprule
Strongly convex $r$ & Case I: Theorem~\ref{thm:convergenceSmooth} 
& Case III:~Theorem~\ref{thm:convergenceNonsmooth}\\
\midrule
Non-strongly convex, separable $r$ & Case II: ~Theorem~\ref{thm:convergenceNonsmooth} & -- \\
\bottomrule
\end{tabular}
\label{tab:rates}
\end{table}

\subsection{Recovering Earlier Work as a Special Case}
\label{sec:oldcocoa}

As a special case, the proposed framework and rates directly apply to $L_2$-regularized loss-minimization problems, including those presented in the earlier work of \cite{Jaggi:2014vi} and \cite{Ma:2015ti}.

\begin{remark}
If we run Algorithm~\ref{alg:dual} (mapping~\eqref{eq:generalobj} to~\eqref{eq:dual}) and restrict $f^*(\cdot) := \tfrac{\lambda}{2}\| \cdot \|^2$ (so that $\tau = \lambda$), 
 Theorem~\ref{thm:convergenceNonsmooth} recovers as a special case the \cocoap rates for general $L$-Lipschitz $\oldell^*_i$ losses \citep[see][Corollary 9]{Ma:2015ti}. The earlier work of  \cocoa-v1 \citep{Jaggi:2014vi} did not provide rates for  $L$-Lipschitz $\oldell^*_i$ losses. 
\end{remark}

\begin{remark}
If we run Algorithm~\ref{alg:dual} (mapping~\eqref{eq:generalobj} to~\eqref{eq:dual}) and restrict $f^*(\cdot) := \tfrac{\lambda}{2}\| \cdot \|^2$ (so that $\tau = \lambda$), 
 Theorem~\ref{thm:convergenceSmooth} recovers as a special case the \cocoap rates for $(1/\oldell^*_i)$-smooth losses \citep[see][Corollary 11]{Ma:2015ti}. The earlier rates of \cocoa-v1 can be obtained by setting $\gamma$:=$\tfrac{1}{K}$ and $\sigma'$=$1$ in Algorithm~\ref{alg:generalizedcocoa} \citep[Theorem 2]{Jaggi:2014vi}.
\end{remark}

These cases follow since $g^*_i$ is $L$-Lipschitz if and only if 
$g_i$ has $L$-bounded support \cite[Corollary 13.3.3]{Rockafellar:1997ww}, and $g^*_i$ is $\mu$-strongly convex if and only if~$g_i$ is $(1/{\mu})$-smooth \cite[Theorem 4.2.2]{hiriart-urruty:2001df}.

\section{Applications}
\label{sec:applications}

In this section we detail example applications that can be cast within the general \proxcocoa framework. For each example, we describe the primal-dual setup and algorithmic details, discuss the convergence properties our framework for the application, and include practical concerns such as information on state-of-the-art local solvers. We discuss examples according to the three cases defined in Table~\ref{tab:cases} of Section~\ref{sec:setup} for finding a minimizer of the general objective $\ell(\uv) + r(\uv)$, and provide a summary of these common examples in Table~\ref{tab:objectives}.

\begin{table}[h!]
\caption{Common Losses and Regularizers.} 
\renewcommand{\thesubtable}{\roman{subtable}}
\begin{subtable}{.485\linewidth}
\footnotesize
\caption{Losses}
\centering
\begin{tabular}{lcl}
\toprule
Loss & Obj & $f$ / $g^*$ \\
\midrule
Least Squares& \eqref{eq:primal} & $f$=$\frac{1}{2} \|A\alphav - \bv\|_2^2$ \\ \vspace{.75em}
& \eqref{eq:dual} & $g^*$=$\frac{1}{2} \|A^\top\wv - \bv\|_2^2$ \\
Logistic Reg.& \eqref{eq:primal} & $f$=$\frac{1}{\d}\!\sum_{j}\!\log (1\!+\!\exp(b_j\xv_j^\top\alphav))$ \\ \vspace{.75em}
  & \eqref{eq:dual} & $g^*$=$\frac{1}{\n}\!\sum_{i}\!\log (1\!+\!\exp(b_i\xv^\top_i\wv))$ \\
SVM & \eqref{eq:dual} & $g^*$=$\frac{1}{\n}\sum_{i} \max(0,1\!-\!y_i\xv^\top_i\wv)$\\
Absolute Dev. & \eqref{eq:dual} & $g^* = \frac{1}{n}\sum_{i} | \xv_i^\top\wv - y_i |$ \\
\bottomrule
\end{tabular}
\end{subtable}
\quad
\begin{subtable}{.5\linewidth}
\footnotesize
\caption{Regularizers}
\centering
\begin{tabular}{lcl}
\toprule
Regularizer & Obj & $g$ / $f^*$ \\
\midrule \vspace{.07em}
Elastic Net & \eqref{eq:primal} & $g$=$\lambda(\eta\|\alphav\|_1 \!+\! \frac{1-\eta}{2}\| \alphav \|_2^2)$ \\ \vspace{.81em}
 & \eqref{eq:dual} & $f^*$=$\lambda(\eta\|\wv\|_1\!+\!\frac{1-\eta}{2}\| \wv \|_2^2)$ \\ 
$L_2$ & \eqref{eq:primal} & $g$=$\frac{\lambda}{2}\|\alphav\|_2^2$ \\ \vspace{.81em}
& \eqref{eq:dual} & $f^*$=$\frac{\lambda}{2} \| \wv \|_2^2 $\\
$L_1$ & \eqref{eq:primal} & $g$=$\lambda \|\alphav \|_1$ \\
Group Lasso & \eqref{eq:primal} & $g$=$\lambda\!\sum_p\!\|\alphav_{\Ip}\|_2$, $\Ip \subseteq [n] $ \\
\bottomrule
\end{tabular}
\end{subtable}
\label{tab:objectives}
\end{table}

\subsection{Case I: Smooth $\ell$, Strongly convex $r$}
\label{sec:case1}

For input problems~\eqref{eq:generalobj} with smooth~$\ell$ and strongly convex~$r$, Theorem~\ref{thm:convergenceSmooth} from Section~\ref{sec:convergence} gives a global linear (geometric) convergence rate. Smooth loss functions can be mapped either to the function $f$ in objective~\eqref{eq:primal}, or $g^*$ in~\eqref{eq:dual}. Similarly, strongly convex regularizers can be mapped either to function $g$ in objective~\eqref{eq:primal}, or $f^*$ in~\eqref{eq:dual}. To illustrate the role of $f$ as a smooth loss function and $g$ as a strongly convex regularizer in objective~\eqref{eq:primal}, contrasting with their traditional roles in prior work~\citep{Yang:2013vl,Jaggi:2014vi,Ma:2015ti,Ma:2017dx}, we consider the following examples. Note that mapping to objective~\eqref{eq:dual} instead will follow trivially assuming that the loss is separable across training points (see~Table~\ref{tab:objectives}). 

For the examples in this subsection, we use $\d$ to represent the number of training points and $n$ the number of features. Note that these definitions may change in the following subsections: this flexibility is useful so that we can present both the primal and dual variations of our framework (Algorithms~\ref{alg:primal} and~\ref{alg:dual}) via a single abstract method~(Algorithm~\ref{alg:generalizedcocoa}).

\paragraph{Smooth $\ell$: least squares loss.}
Let $\bv \in \R^\d$ be labels or response values, and consider the least squares objective, $f(\vv) := \frac12 \|\vv - \bv\|_2^2$, which is $1$-smooth. We obtain the familiar least-squares regression objective in our optimization problem \eqref{eq:primal}, using
\begin{equation}
f(A\alphav) := \textstyle\frac{1}{2}\| A\alphav - \bv\|_2^2 \, .
\end{equation}
Observing that the gradient of $f$ is $\nabla f(\vv)=\vv{-}\bv$, the primal-dual mapping is given by: $\wv(\alphav) := A \alphav{-}\bv$, which is well known as the \textit{residual vector} in least-squares regression.

\paragraph{Smooth $\ell$: logistic regression loss.}
For classification, we consider a logistic regression model with $\d$ training examples, $\yv_j\in\R^n$ for $j\in[\d]$, collected as the rows of the data matrix~$A$.
For each training example, we are given a binary label, which we collect in the vector $\bv \in \{-1,1\}^\d$. 
Formally, the objective is defined as $f(\vv) := \sum_{j=1}^\d \log{(1 + \exp{(-b_j v_j)})}$, which is again a separable function.
The classifier loss is given by
\begin{equation}
\label{eq:logisticlossprimal}
f(A\alphav) := \sum_{j=1}^\d \log{(1 + \exp{(-b_j \yv_j^\top \alphav)})} \, , 
\end{equation}
where $\alphav\in\R^n$ is the parameter vector. It is not hard to show that $f$ is $1$-smooth if the labels satisfy $b_j\in[-1,1]$. 
The primal-dual mapping is given by 
$
w_j(\alphav) 
:= \frac{-b_j}{1 + \exp{(b_j \yv_j^\top \alphav)}} \, .$  

\paragraph{Strongly convex $r$: elastic net regularizer.} 
\label{sec:elasticnet}
An application we can consider for a strongly convex regularizer, $g$ in~\eqref{eq:primal} or~$f^*$ in~\eqref{eq:dual}, is elastic net regularization, $\eta\lambda \|\uv\|_1 + (1-\eta)\frac{\lambda}{2}\| \uv \|_2^2$,
for fixed parameter $\eta \in (0,1]$. This can be obtained in~\eqref{eq:primal} by setting
\begin{equation}
g(\alphav) = \sum_{i=1}^n g_i(\alpha_i) := \sum_{i=1}^n \eta \lambda |\alpha_i| + (1-\eta) \tfrac{\lambda}{2}\alpha_i^2 .
\end{equation}
For the special case $\eta=1$, we obtain the $\Lone$-norm, and for $\eta=0$, we obtain the $\Ltwo$-norm. The conjugate of $g_i$ is given by: $
    g^*_i(x) := \textstyle\frac1{2(1-\eta)} \big(\big[|x|-\eta\big]_+\big)^2$,
where $[.]_+$ is the positive part operator, $[s]_+ = s$ for $s>0$, and zero otherwise.

\subsection{Case II: Smooth $\ell$, Non-Strongly Convex Separable $r$}
\label{sec:l1}

In case II, we consider mapping the input problem \eqref{eq:generalobj} to objective~\eqref{eq:primal}, where $\ell$ is assumed to be smooth, and $r$ non-strongly convex and separable. For smooth losses in~\eqref{eq:primal}, we can consider as examples those provided in Subsection~\ref{sec:case1}, e.g., the least squares loss or logistic loss. For an example of a non-strongly convex regularizer, we consider the important case of $\Lone$ regularization below. Again, we note that this application cannot be realized by objective~\eqref{eq:dual}, where it is assumed that the regularization term $f^*$ is strongly convex.

\paragraph{Non-strongly convex $r$: $\Lone$ regularizer.}  $\Lone$ regularization is obtained in objective~\eqref{eq:primal} by letting $g_i(\cdot) := \lambda |\cdot |$. However, an additional modification is necessary to obtain primal-dual convergence and certificates for this setting. In particular, we employ the modification introduced in Section~\ref{sec:convergence}, which will guarantee $L$-bounded support. Formally, we replace $g_i(\cdot) = |\cdot|$ by
\begin{equation*}
\bar{g}(\alpha) :=
\begin{cases}
|\alpha| &: \alpha \in [-B,B], \\
+\infty &: \text{otherwise.}
\end{cases}
\end{equation*}
For large enough $B$, this problem yields the same solution as the original $\Lone$-regularized objective. 
Note that this only affects convergence theory, in that it allows us to present a strong primal-dual rate (Theorem~\ref{thm:convergenceNonsmooth} for $L$=$B$).
With this modified $\Lone$ regularizer, the optimization problem~\eqref{eq:primal} with regularization parameter $\lambda$ becomes
\begin{equation}
\label{eq:surrogatel1regproblem}
\min_{\alphav \in \R^{n}} \ f(A \alphav)  + \lambda \sum_{i = 1}^{\n} \bar{g}(\alpha_i) \, .
\end{equation}
For large enough choice of the value $B$, this problems yields the same solution as the original objective: $f(A \alphav)  + \lambda \sum_{i=1}^{\n} |\alpha_i|$.
The modified $\bar{g}$ is simply a constrained version of the absolute value to the interval $[-B,B]$.
Therefore by setting $B$ to a large enough value that the values of $\alpha_i$ will never reach it, $\bar{g}^*$ will be continuous and at the same time make \eqref{eq:surrogatel1regproblem} equivalent to the original objective.

Formally, a simple way to obtain a large enough value of $B$, so that all solutions 
are unaffected, is the following: If we start the algorithm at $\alphav = \0$, for every point encountered during execution of a monotone optimizer, the objective values will never become worse than ${\OA}(\0)$. 
Formally, under the assumption that $f$ is non-negative, we will have that (for each $i$):
\[
\lambda |\alpha_i| \leq {f(\0)=\OA(\0)}  \implies |\alpha_i| \leq \frac{f(\0)}{\lambda}.
\]
We can therefore safely set the value of $B$ as $\frac{f(\0)}{\lambda }$. For the modified $\bar{g}_i$, the conjugate $\bar{g}_i^*$ is given by: \vspace{-1mm}
\[
    \bar{g}_i^*(x) := 
    \begin{cases}
            0 & : x \in [-1,1],  \\
            B(|x| - 1) & : \text{otherwise.}
        \end{cases}
\]
We provide a proof of this in Appendix~\ref{app:primaldual} (Lemma~\ref{lem:l1surrogate}).

\paragraph{Non-strongly convex $r$: group lasso.} The group lasso penalty can be mapped to objective~\eqref{eq:primal}, with:
\begin{equation}
g(\alphav) := \lambda \sum_{p=1}^P \|\alphav_{\Ip}\|_2 \hspace{1em} \text{with} \hspace{1em}  \bigcup\limits_{p=1}^P \Ip = \{ 1, \dots, n \} \, ,
\end{equation} 
where the disjoint sets $\Ip \subseteq \{ 1, \dots, n \}$ represent a  partitioning of the total set of variables.
This penalty can be viewed as an intermediate between a pure $\Lone$ or $\Ltwo$ penalty, performing variable selection only at the group level. The term $\alphav_{\Ip} \in \R^{|\Ip|}$ denotes part of the vector~$\alphav$ with indices $\Ip$. The conjugate is given by: $$ g^*(\wv) = I_{\{\wv | \max_{\Ip \in [n]} \| \alphav_{\Ip} \|_2  \le \lambda\}}(\wv).$$
For details, see, e.g., \cite{Dunner:2016vga} or \citet[Example 3.26]{Boyd:2004uz}.

\subsection{Case III: Non-Smooth Separable $\ell$, Strongly Convex $r$}

Finally, in case III, we consider mapping the input problem \eqref{eq:generalobj} to objective~\eqref{eq:dual}, where~$\ell$ is assumed to be non-smooth and separable, and $r$ strongly convex. 
We discuss two common cases of general non-smooth losses $\ell$, including the the hinge loss for classification and absolute deviation loss for regression. When paired with a strongly convex regularizer, the regularizer via $f$ gives rise to the primal-dual mapping, and Theorem~\ref{thm:convergenceNonsmooth} provides a sublinear convergence rate for objectives of this form. We note that these losses cannot be realized directly by objective~\eqref{eq:primal}, where it is assumed that the loss term $f$ is smooth.

\paragraph{Non-smooth $\ell$: hinge loss.}
For classification problems, we can consider a hinge loss support vector machine model, on $\n$ training points in $\R^{\d}$, given with the loss:
\begin{equation}
g^*
(-A^\top\wv) = \sum_{i=1}^ng_i^*(-\xv_i^\top\wv) := \frac{1}{\n}
\sum_{i=1}^{\n} \max\{0, 1- y_i \xv_i^\top\wv \}. 
\end{equation}
The conjugate function of the hinge loss $\phi(a) = \max\{0,1-b\}$ is given by $\phi^*(b) = \{ b \, $ if $ b \in [-1, 0]$, else \,$\infty\}$. When using the $\Ltwo$ norm for regularization in this problem: $f^*(\wv) := \lambda \| \wv\|_2^2$, a primal-dual mapping is given by: $\wv(\alphav) := \frac{1}{\lambda n} A \alphav$.

\paragraph{Non-smooth $\ell$: absolute deviation loss.} The absolute deviation loss, used, e.g., in quantile regression or least absolute deviation regression, can be realized in objective~\eqref{eq:dual} by setting:
\begin{equation}
g^*(-A^\top\wv) = \sum_{i=1}^n g_i^*(-\xv_i^\top\wv) := \frac{1}{\n}\sum_{i=1}^n\left| \xv_i^\top\wv -y_i \right|.
\end{equation}
The conjugate function of the absolute deviation loss $\phi(a) = | a - y_i |$ is given by $\phi^*(-b) = -by_i$, with $b \in [-1,1]$.

\subsection{Local Solvers}
As discussed in Section~\ref{sec:proxcocoa}, the subproblems solved on each machine in the \proxcocoa framework are appealing in that they are very similar in structure to the global problem~\eqref{eq:primal}, with the main difference being that they are defined on a smaller (local) subset of the data, and have a simpler dependence on the term $f$. Therefore, solvers which have already proven their value in the single machine or multicore setting can be easily leveraged within the framework. We discuss some specific examples of local solvers below, and point the reader to~\cite{Ma:2017dx} for an empirical exploration of these choices.

\paragraph{Local solvers for Algorithm~\ref{alg:primal}.} In the primal setting (Algorithm~\ref{alg:primal}), the local subproblem \eqref{eq:subproblem} becomes a simple quadratic problem on the local data, with regularization applied only to local variables $\vsubset{\alphav}{k}$. For the $\Lone$-regularized examples discussed, existing fast $\Lone$-solvers for the single-machine case, such as \glmnet variants~\citep{Friedman:2010wm} or \textsc{blitz}~\citep{Johnson:2015tq} can be directly applied to each local subproblem $\Ggk(\,\cdot\,; \vv, \vsubset{\alphav}{k})$ within Algorithm~\ref{alg:generalizedcocoa}. The sparsity induced on the subproblem solutions of each machine naturally translates into the sparsity of the global solution, since the local variables $\vsubset{\alphav}{k}$ will be concatenated.

In terms of the approximation quality parameter $\Theta$ for the local problems (Assumption~\ref{asm:theta}), we can apply existing recent convergence results from the single machine case. For example, for randomized coordinate descent (as part of \glmnet), \citet[Theorem 1]{Lu:2013tl} gives a $\mathcal{O}(1/t)$ approximation quality for any separable regularizer, including $\Lone$ and elastic net; see also~\cite{Tappenden:2015vha} and \cite{ShalevShwartz:2011vo}.

\paragraph{Local solvers for Algorithm~\ref{alg:dual}.} In the dual setting (Algorithm~\ref{alg:dual}) for the discussed examples, the losses are applied only to local variables $\vsubset{\alphav}{k}$, and the regularizer is approximated via a quadratic term. Current state of the art for the problems of the form in~\eqref{eq:dual} are variants of randomized coordinate ascent---Stochastic Dual Coordinate Ascent (SDCA) \citep{ShalevShwartz:2013wl}. This algorithm and its variants are increasingly used in practice \citep{Wright:2015bn}, and extensions such as accelerated and parallel versions can directly be applied~\citep{ShalevShwartz:2014dy,Fan:2008tf} in our framework. For non-smooth losses such as SVMs, the analysis of \cite{ShalevShwartz:2013wl} provides a $\mathcal{O}(1/t)$ rate, and for smooth losses, a faster linear rate. There have also been recent efforts
to derive a linear convergence rate for problems like the hinge-loss support vector machine that could be applied, e.g., by using 
error bound conditions
\citep{necoara2014distributed,
wang2014iteration},
weak strong convexity conditions
\citep{ma2015linear,
necoara2015linear}
or by considering Polyak-{\L}ojasiewicz conditions 
\citep{karimi2016linear}.

\vspace{.5em}
\section{Experiments}
\label{sec:experiments}

In this section we demonstrate the empirical performance of \proxcocoa in the distributed setting. We first compare \proxcocoa to competing methods for two common machine learning applications: lasso regression (Section~\ref{sec:primalexp}) and support vector machine (SVM) classification (Section~\ref{sec:dualexp}). We then explore the performance of \proxcocoa in the primal versus the dual directly by solving an elastic net regression model with both variants (Section~\ref{sec:primaldual}). Finally, we illustrate general properties of the \proxcocoa method empirically in Section~\ref{sec:h}.

\paragraph{Experimental setup.} We compare \proxcocoa to numerous state-of-the-art general-purpose methods for large-scale optimization, including: 
\begin{itemize}
\setlength\itemsep{0mm}
\setlength{\listparindent}{0.1in}
\setlength{\itemindent}{-1em}
\item \textsc{Mb-SGD}: Mini-batch stochastic gradient. For our experiments with lasso, we compare against \textsc{Mb-SGD} with an $\Lone$-prox.
\item \textsc{GD}: Full gradient descent. For lasso we use the proximal version, \textsc{Prox-GD}.
\item \textsc{L-BFGS}: Limited-memory quasi-Newton method. For lasso, we use OWL-QN (orthant-wise limited quasi-Newton).
\item \textsc{ADMM}: Alternating direction method of multipliers. We use conjugate gradient internally for the lasso experiments, and SDCA for SVM experiments.
\item \textsc{Mb-CD}: Mini-batch parallel coordinate descent. For SVM experiments, we implement \textsc{Mb-SDCA} (mini-batch stochastic dual coordinate ascent).
\end{itemize} 

The first three methods are optimized and implemented in Apache Spark's MLlib (v1.5.0) 
\citep{Meng:2015tu}. 
We test the performance of each method in large-scale experiments fitting lasso, elastic net regression, and SVM models to the datasets shown in Table~\ref{tab:datasets}. In comparing to other methods, we plot the distance to the optimal primal solution. This optimal value is calculated by running all methods for a large number of iterations (until progress has stalled),
and then selecting the smallest primal value amongst the results.
All code is written in \textsf{\small Apache Spark} and experiments are run on public-cloud Amazon EC2 m3.xlarge machines with one core per machine. Our code is publicly available at \href{http://gingsmith.github.io/cocoa/}{\texttt{gingsmith.github.io/cocoa/}}.

\begin{table}[h!]
\captionof{table}{Datasets for Empirical Study.}
\label{tab:datasets}
\centering
\begin{tabular}{l r  
      r  r }
      \toprule
    {\small\textbf{Dataset}} & {\small\textbf{Training Size}} &
    {\small\textbf{Feature Size}} & {\small\textbf{Sparsity}}   \\
    \midrule
	url & 2 M & 3 M & 3.5e-5 \\ 
	epsilon & 400 K & 2 K & 1.0 \\
	kddb & 19 M & 29 M & 9.8e-7 \\
	webspam & 350 K & 16 M & 2.0e-4 \\
	\bottomrule
      \end{tabular}
\end{table}

We carefully tune each competing method in our experiments for best performance. \textsc{ADMM} requires the most tuning, both in selecting the penalty parameter $\rho$ and in solving the subproblems. Solving the subproblems to completion for ADMM is prohibitively slow, and we thus use an iterative method internally and improve performance by allowing early stopping. We also use a varying penalty parameter~$\rho$ --- practices described in \citet[Sections 4.3, 8.2.3, 3.4.1]{Boyd:2010bw}. For \textsc{Mb-SGD}, we tune the step size and mini-batch size parameters. For \textsc{Mb-CD} and \textsc{Mb-SDCA}, we scale the updates at each round by $\frac{\beta}{b}$ 
for mini-batch size $b$ and $\beta \in [1,b]$, and tune both parameters $b$ and $\beta$. Further implementation details for all methods are given in Section~\ref{sec:expdetails}. For simplicity of presentation and comparison, in all of the following experiments, we restrict \proxcocoa to only use simple coordinate descent as the local solver. We note that even stronger empirical results for \proxcocoa could be obtained by plugging in state-of-the-art local solvers for each application at hand.

\subsection{\proxcocoa in the Primal: An Application to Lasso Regression}
\label{sec:primalexp}
We first demonstrate the performance of \proxcocoa in the primal (Algorithm~\ref{alg:primal}) by applying \proxcocoa to a lasso regression model~\eqref{ex:lasso} fit to the datasets in Table~\ref{tab:datasets}. We use stochastic coordinate descent as a local solver for \proxcocoa, and select the number of local iterations $H$ (a proxy for subproblem approximation quality, $\Theta$) from several options with best performance. 

\newcommand{\smalltrimfig}[1]{\begin{subfigure}[b]{.45\linewidth}\includegraphics[trim = 30 180 60 180, clip, width=\linewidth]{#1}\end{subfigure}}
\begin{figure*}[h!]
\centering
\smalltrimfig{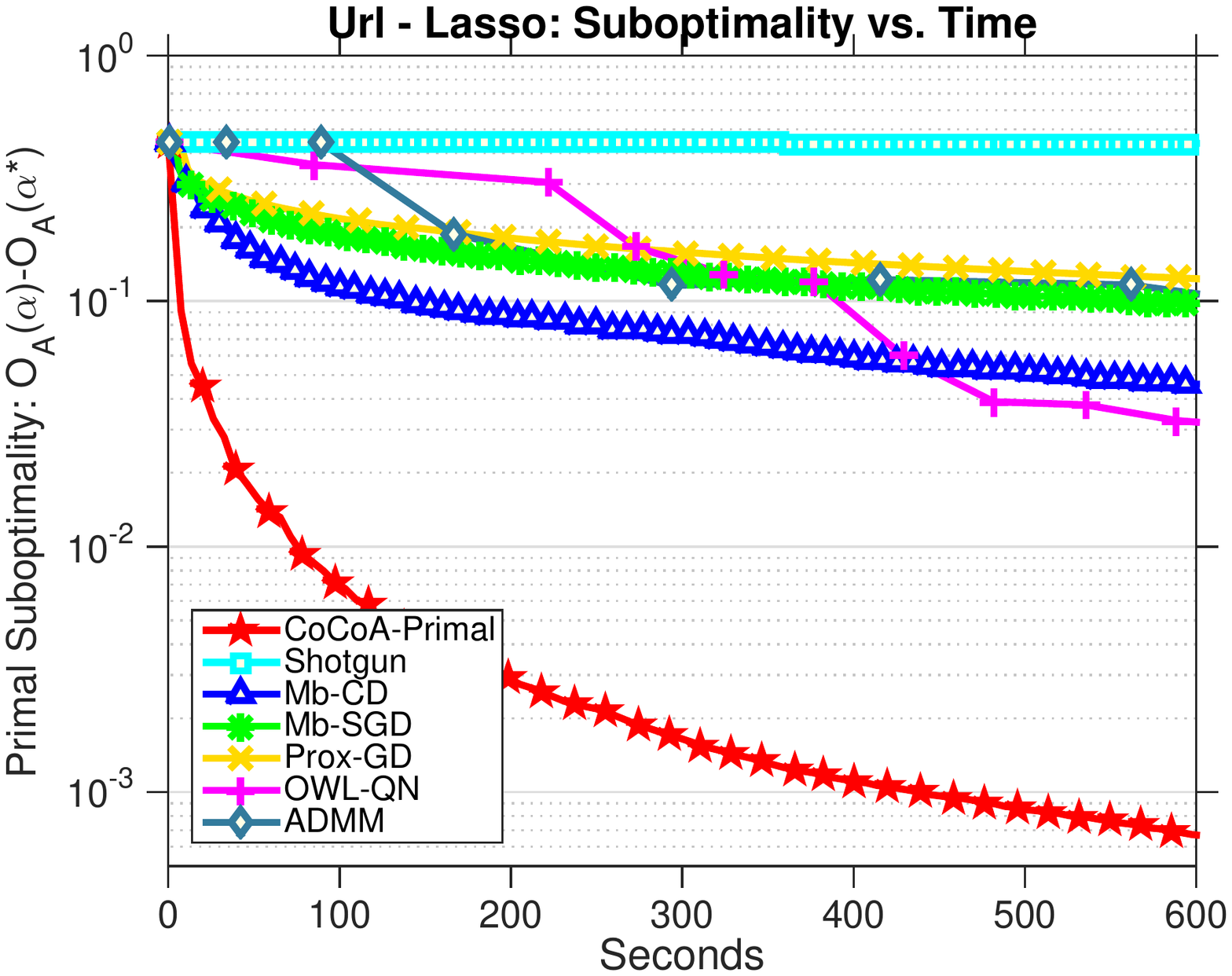}
\smalltrimfig{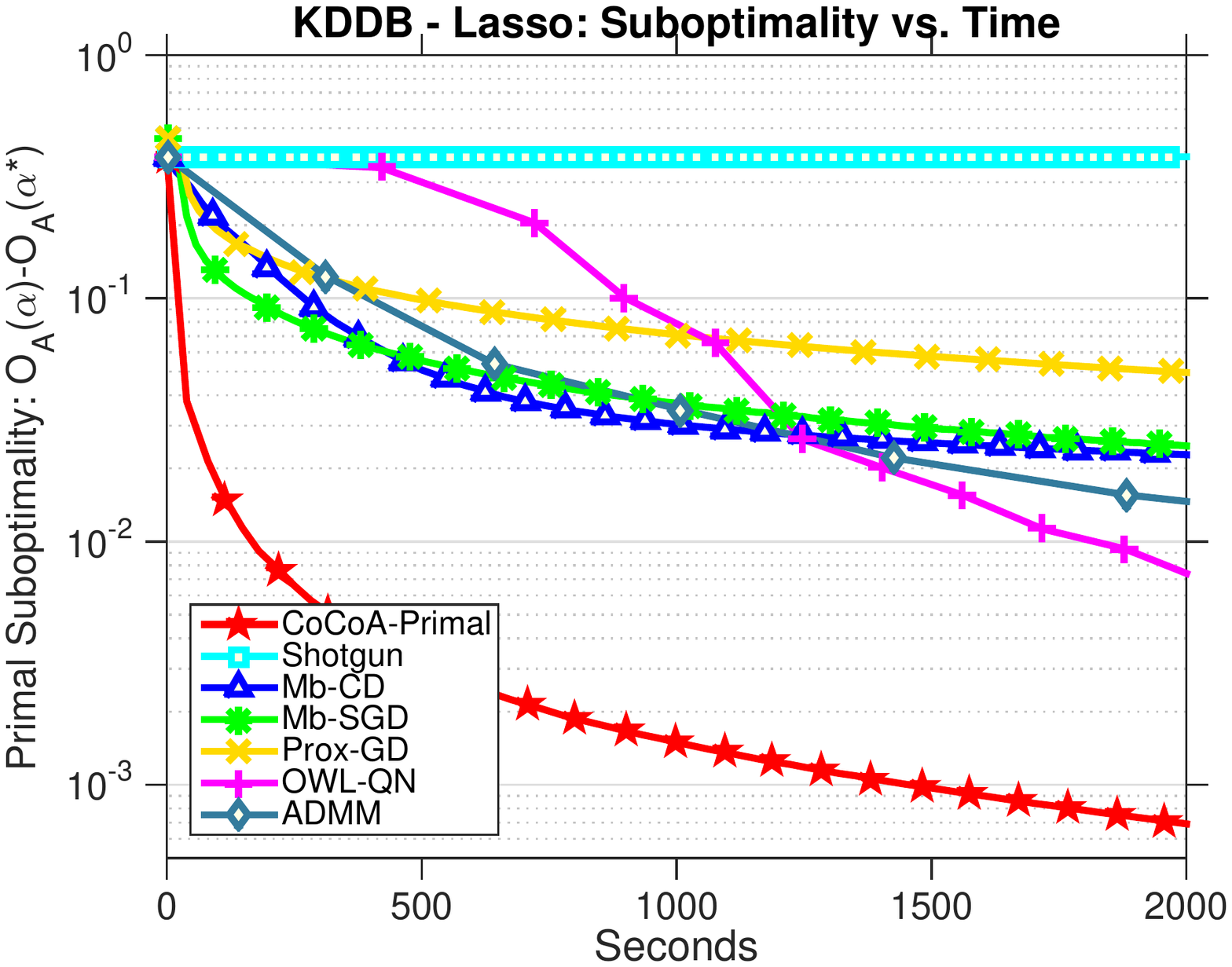}
\smalltrimfig{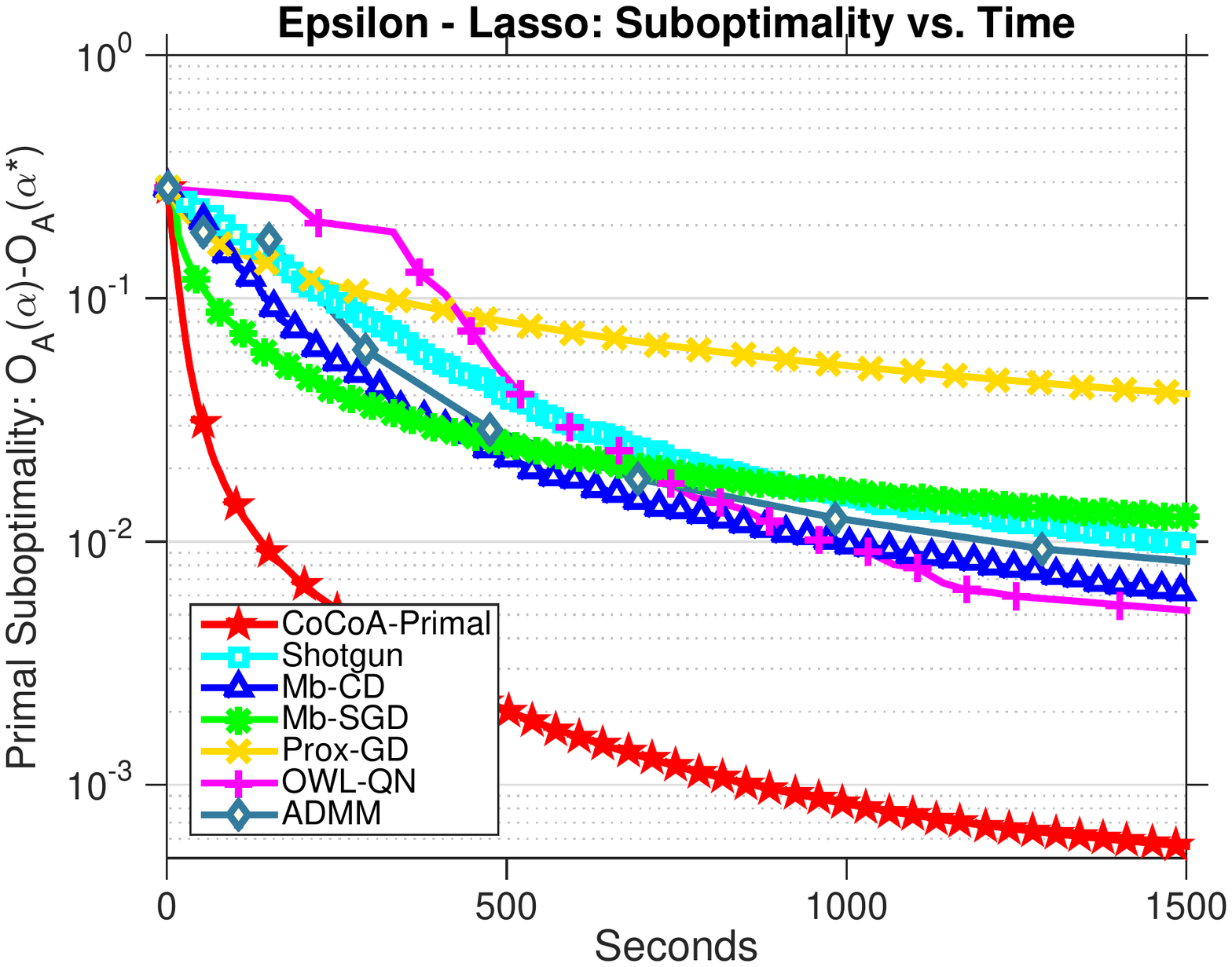}
\smalltrimfig{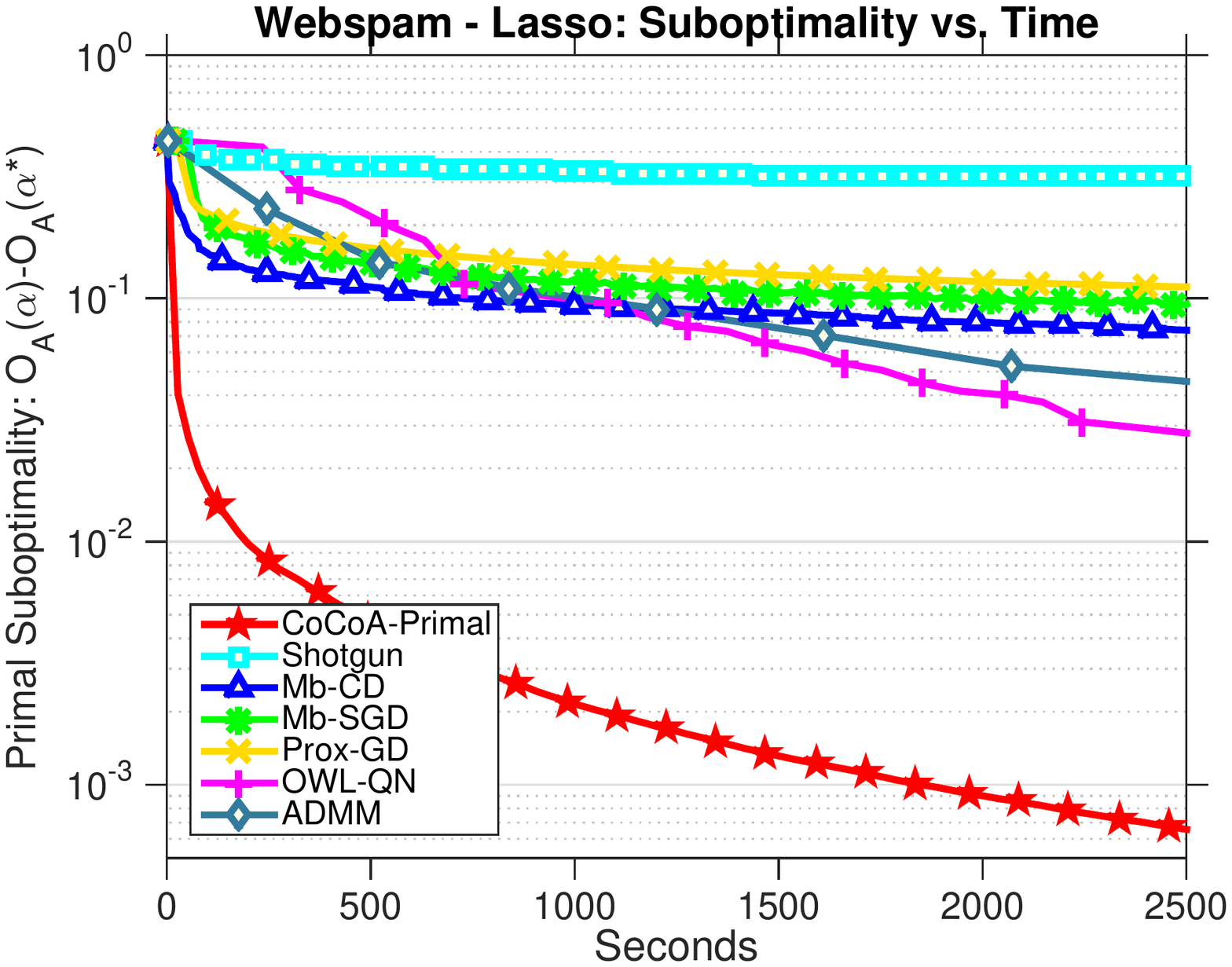}
\caption{Suboptimality in terms of ${\OA}(\alphav)$ for fitting a lasso regression model to four datasets: url ($K$=4, $\lambda$=\mbox{1\sc{e}-4}), kddb ($K$=4, $\lambda$=\mbox{1\sc{e}-6}), epsilon ($K$=8, $\lambda$=\mbox{1\sc{e}-5}), and webspam ($K$=16, $\lambda$=\mbox{1\sc{e}-5}) datasets. \proxcocoa applied to the primal formulation converges more quickly than all other compared methods in terms of the time in seconds.
}
\label{fig:comparison}
\end{figure*}

We compare \proxcocoa to the general methods listed above, including \textsc{Mb-SGD} with an $\Lone$-prox, \textsc{Prox-GD}, \textsc{OWL-QN}, \textsc{ADMM}, and \textsc{Mb-CD}. We provide a comparison with \textsc{Shotgun}~\citep{Bradley:2011wq}, a popular method for solving $\Lone$-regularized problems in the multicore environment, as an extreme case to highlight the detrimental effects of frequent communication in the distributed environment. For \textsc{Mb-CD}, \textsc{Shotgun}, and \proxcocoa in the primal, datasets are distributed by feature, whereas for \textsc{Mb-SGD}, \textsc{Prox-GD}, \textsc{OWL-QN} and \textsc{ADMM} they are distributed by training point.

In analyzing the performance of each algorithm (Figure \ref{fig:comparison}), we measure the improvement to the primal objective given in \eqref{eq:primal} $({\OA}(\alphav))$ in terms of  
wall-clock time in seconds. We see that both \textsc{Mb-SGD} and \textsc{Mb-CD} are slow to converge, and come with the additional burden of having to tune extra parameters (though \textsc{Mb-CD} makes clear improvements over \textsc{Mb-SGD}). 
As expected, naively distributing \textsc{Shotgun} (single coordinate updates per machine) does not perform well, as it is tailored to shared-memory systems and requires communicating too frequently. 
\textsc{OWL-QN} performs the best of all compared methods, but is still much slower to converge than \proxcocoa, and converges, e.g., 50$\times$ more slowly for the webspam dataset. The optimal performance of \proxcocoa is particularly evident in datasets with large numbers of features (e.g., url, kddb, webspam), which are exactly the datasets where $\Lone$ regularization would most typically be applied. 

Results are shown for regularization parameters $\lambda$ such that the resulting weight vector~$\alphav$ is sparse. However, our results are robust to varying values of~$\lambda$ as well as to various problem settings, as we illustrate in Figure~\ref{fig:lambda}.

\newcommand{\tinytrimfig}[1]{\begin{subfigure}[b]{.45\linewidth}\includegraphics[trim = 25 180 55 180, clip, width=\linewidth]{#1}\end{subfigure}}
\begin{figure}[h!]
\centering
\captionsetup{type=figure}
\tinytrimfig{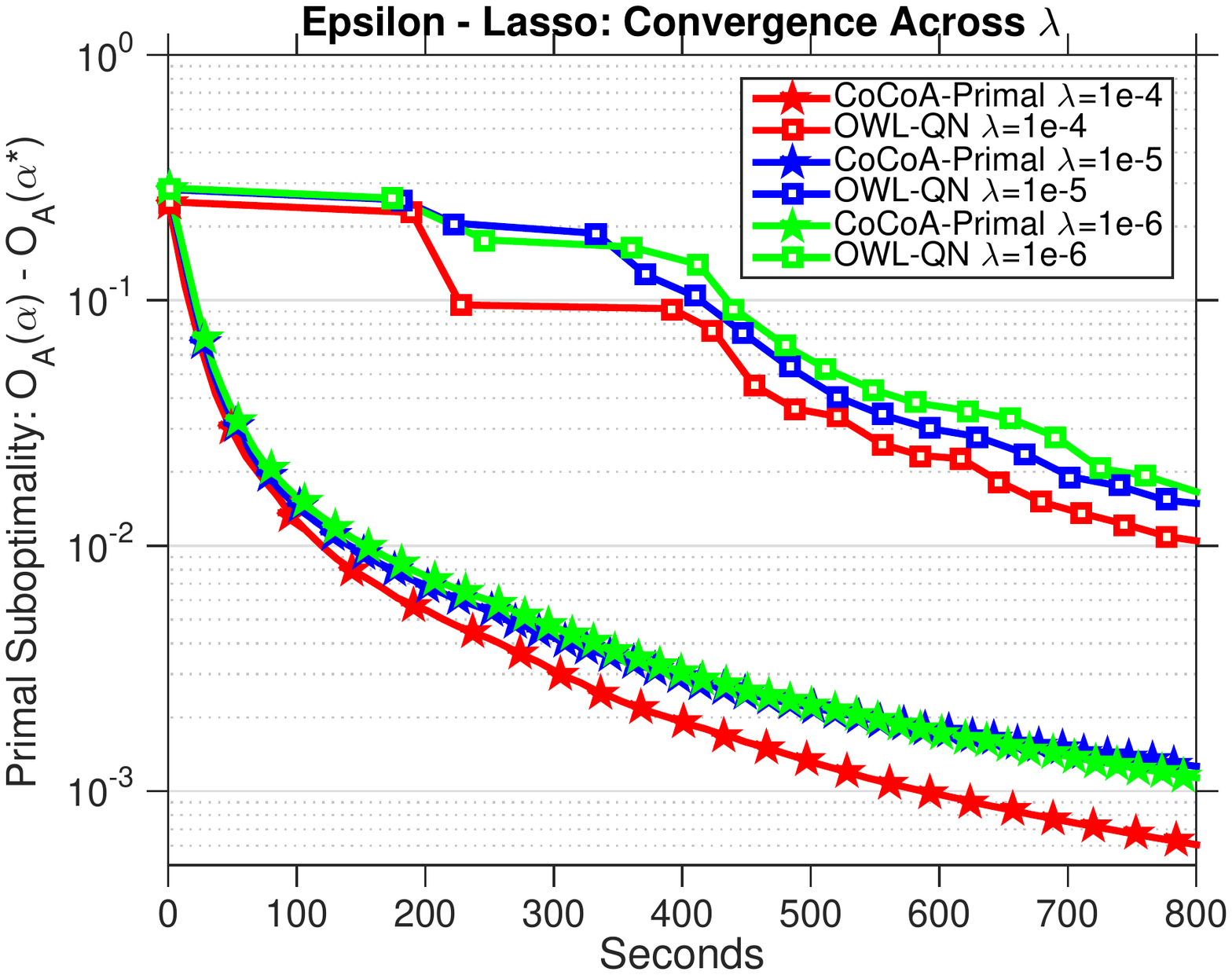}
\tinytrimfig{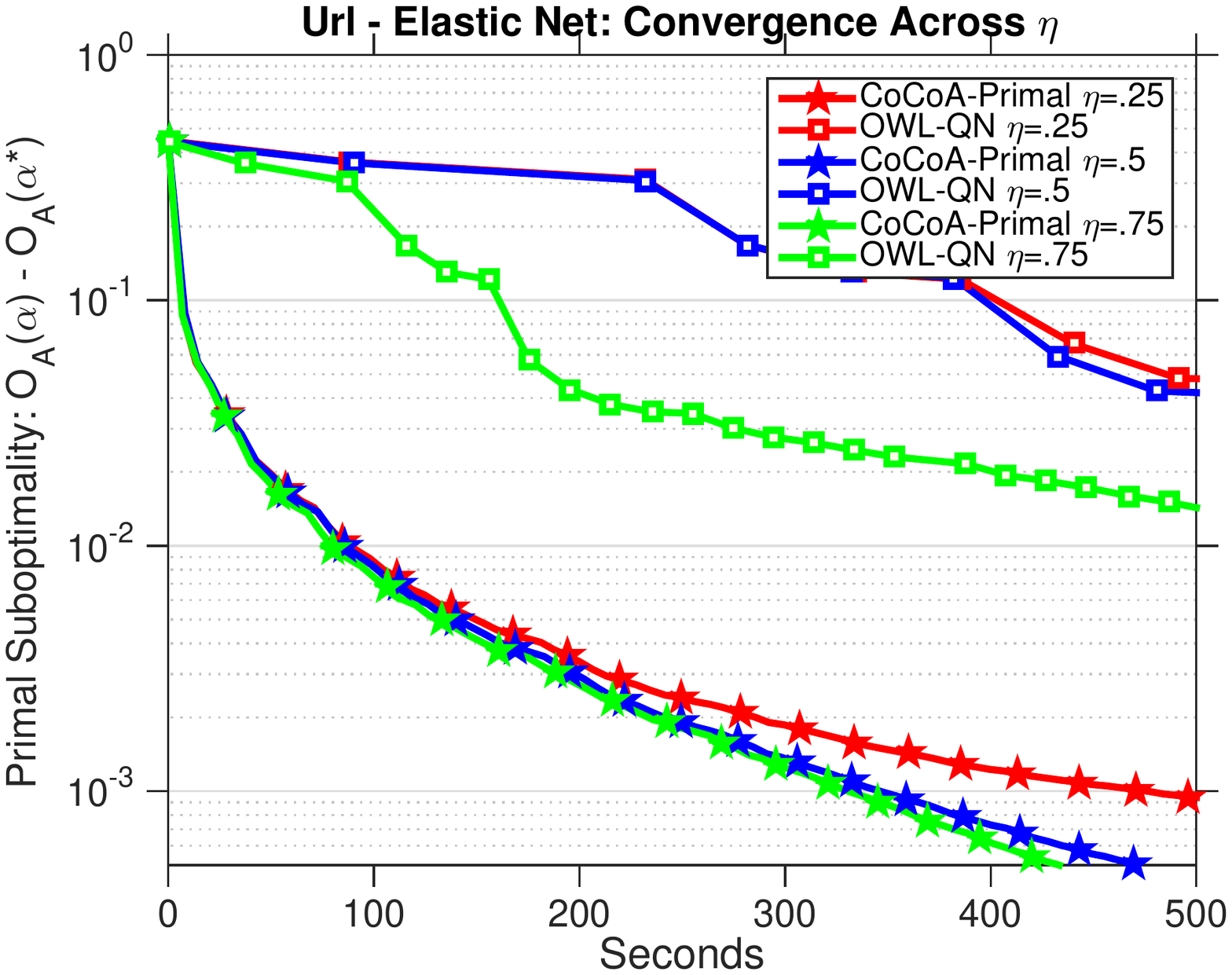}
\caption{{Suboptimality in terms of ${\OA}(\alphav)$ for solving lasso for the epsilon dataset (left, $K$=8) and elastic net for the url dataset, (right, $K$=4, $\lambda$=\mbox{1\sc{e}-4}). Speedups are robust over different regularizers $\lambda$ (left), and across problem settings, including varying $\eta$ parameters of elastic net regularization (right).}}
\label{fig:lambda}
\end{figure}

\newcommand{\newtinytrimfig}[1]{\begin{subfigure}[b]{.45\linewidth}\includegraphics[trim = 35 180 60 180, clip, width=\linewidth]{#1}\end{subfigure}}
\begin{table}[ht]
\centering
\captionsetup{type=figure}
\newtinytrimfig{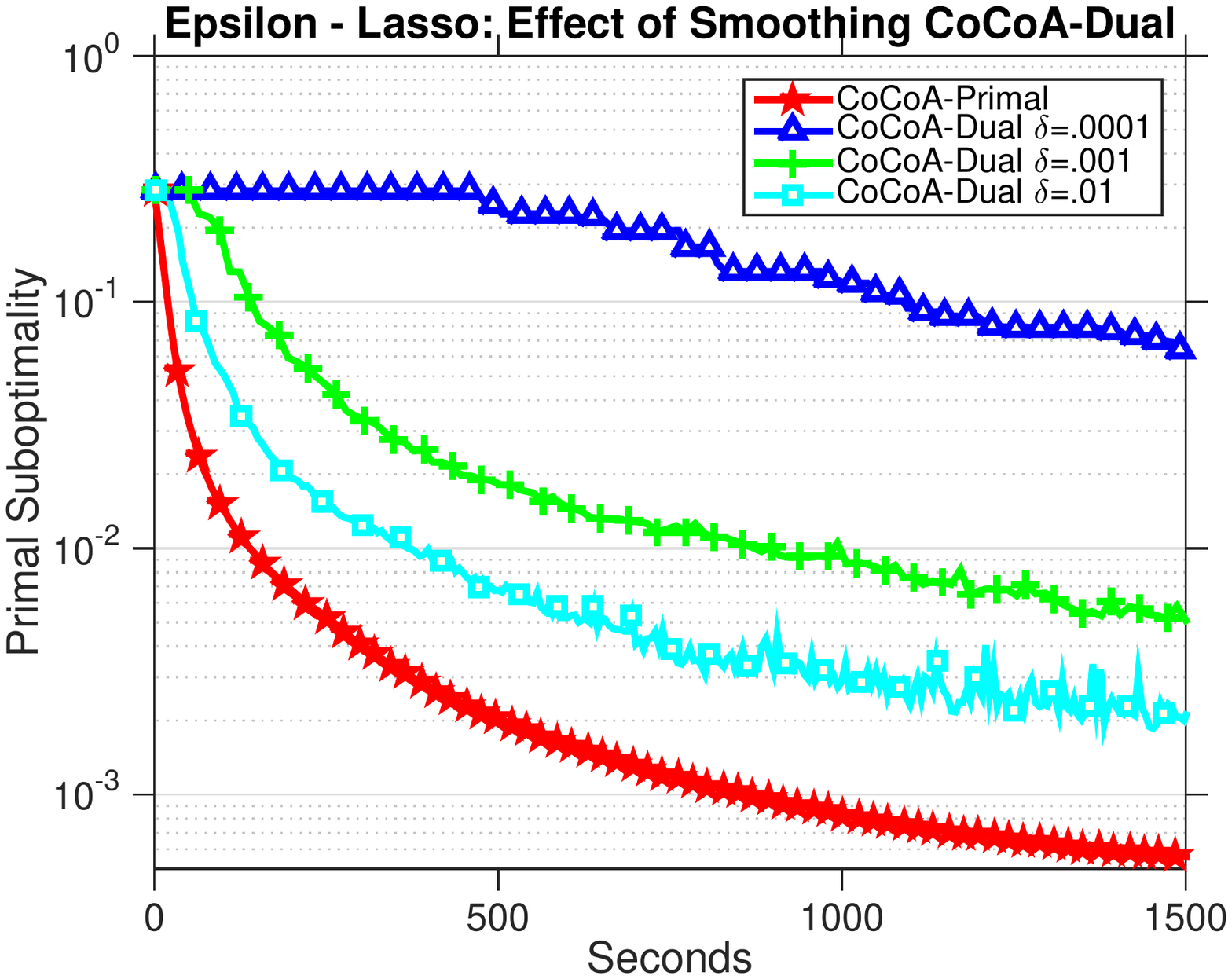}
\qquad
      \begin{tabular}[b]{ l  c }
            \multicolumn{2}{c}{Table 6: Sparsity of Final Iterates.} \vspace{.5em} \\
      \toprule
      {\small\textbf{Method}} & {\small\textbf{Sparsity}} \\
      \midrule
    {\small{\proxcocoa-Primal}} & {0.6030} \\ 
    {\small \proxcocoa-Dual: \textbf{$\delta=0.0001$}} & {0.6035} \\ 
    {\small \proxcocoa-Dual: {$\delta=0.001$}} & {0.6240} \\ 
    {\small \proxcocoa-Dual: \textbf{$\delta=0.01$}} & {0.6465}  \\ 
	\bottomrule \vspace{1.5em}
	\label{tab:sparsity}
      \end{tabular}
   \captionlistentry[table]{A table beside a figure}
    \captionsetup{labelformat=andtable}
\caption*{Figure 3 \& Table 6: For pure $L_1$ regularization, smoothing is not an effective option for \proxcocoa in the dual. It either modifies the solution (Figure 3) or slows convergence (Table 6). This motivates running \proxcocoa instead on the primal for these problems.} 
\label{fig:cocoa}
\end{table}

\paragraph{A case against smoothing.}

We additionally motivate the use of \proxcocoa in the primal by showing how it improves upon \proxcocoa in the dual \citep{Yang:2013vl,Jaggi:2014vi,Ma:2015ti,Ma:2017dx} for non-strongly convex regularizers. 
First, \proxcocoa in the dual cannot be  included in the set of experiments in Figure~\ref{fig:comparison} because it cannot be directly applied to the lasso objective  (recall that Algorithm~\ref{alg:dual} only allows for strongly convex regularizers). 

To get around this requirement, previous work has suggested implementing the  smoothing technique used in, e.g., \cite{ShalevShwartz:2014dy,Zhang:2015vj} --- adding a small amount of strong convexity $\delta \|\alphav\|_2^2$ to the objective for lasso regression. In Figure~3 we demonstrate the issues with this approach, comparing \proxcocoa in the primal on a pure $\Lone$-regularized regression problem to \proxcocoa in the dual for decreasing levels of $\delta$. The smaller we set $\delta$, the less smooth the problem becomes. As $\delta$ decreases, the final sparsity of running \proxcocoa in the dual starts to match that of running pure $L_1$ (Table~6), but the performance also degrades (Figure~3). We note that by using \proxcocoa in the primal with the modification presented in Section~\ref{sec:convergence}, we can deliver strong rates without having to make any compromises in terms of the training speed or accuracy.

\subsection{\proxcocoa in the Dual: An Application to SVM Classification}
\label{sec:dualexp}

Next we present results on \proxcocoa in the dual against competing methods, for a hinge loss support vector machine model~\eqref{ex:svm} on the datasets in Table~\ref{tab:datasets}. We use stochastic dual coordinate ascent (SDCA) as a local solver for \proxcocoa in this setting, again selecting the number of local iterations~$H$ from several options with best performance. We compare \proxcocoa to the general methods listed above, including \textsc{Mb-SGD}, \textsc{GD}, \textsc{L-BFGS}, \textsc{ADMM}, and \textsc{Mb-SDCA}. All datasets are distributed by training point for these methods.

\begin{figure*}[h!]
\centering
\smalltrimfig{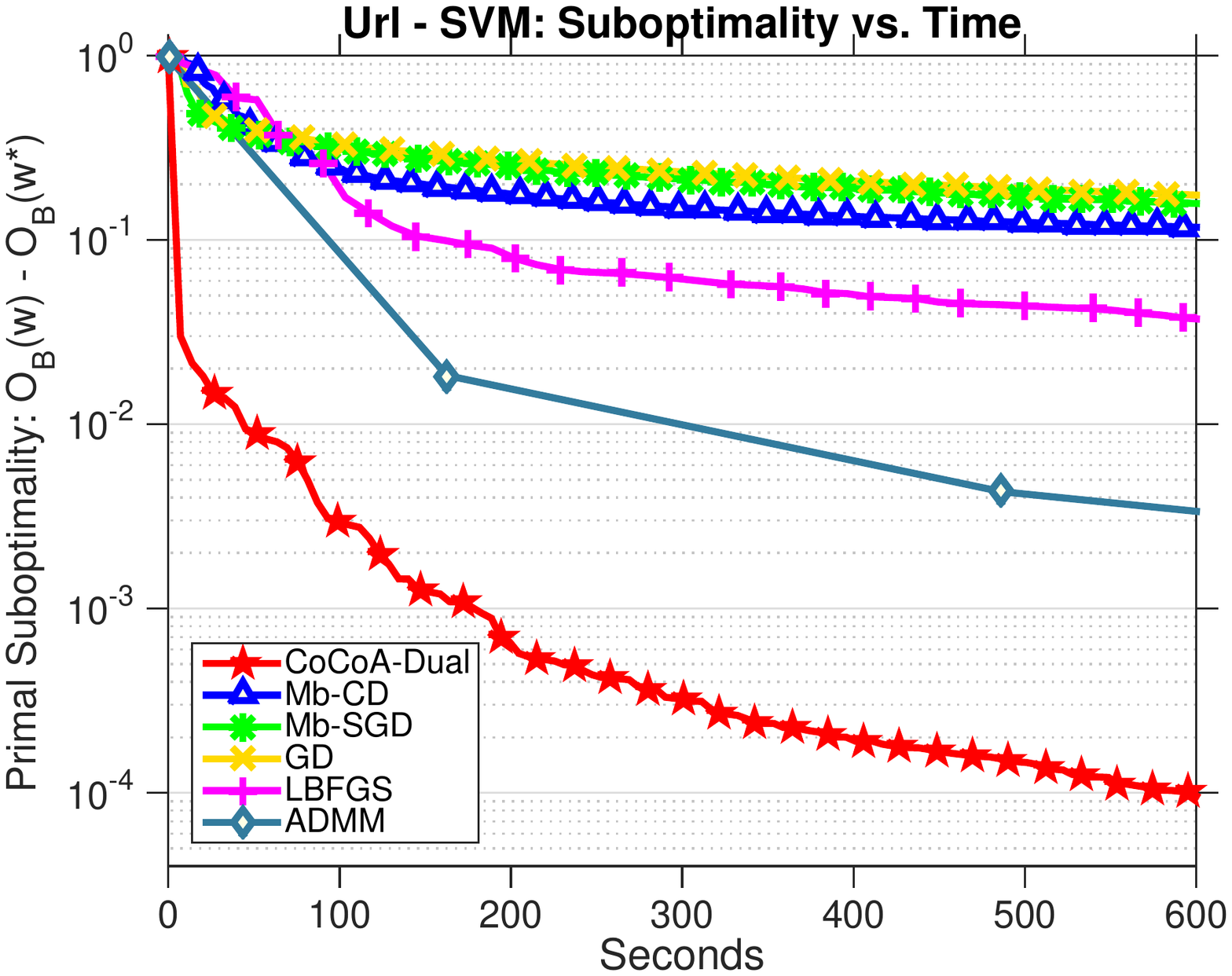}
\smalltrimfig{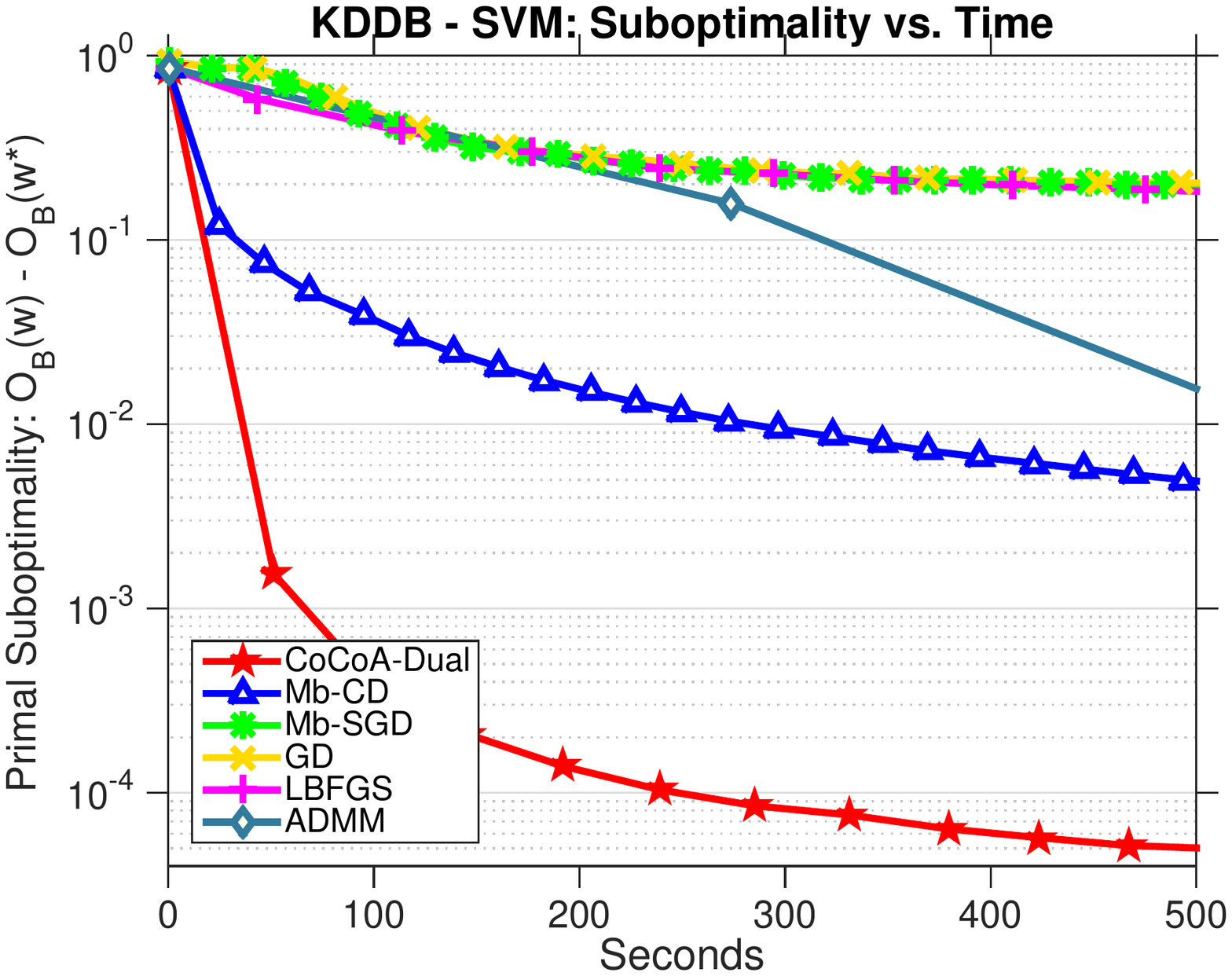}
\smalltrimfig{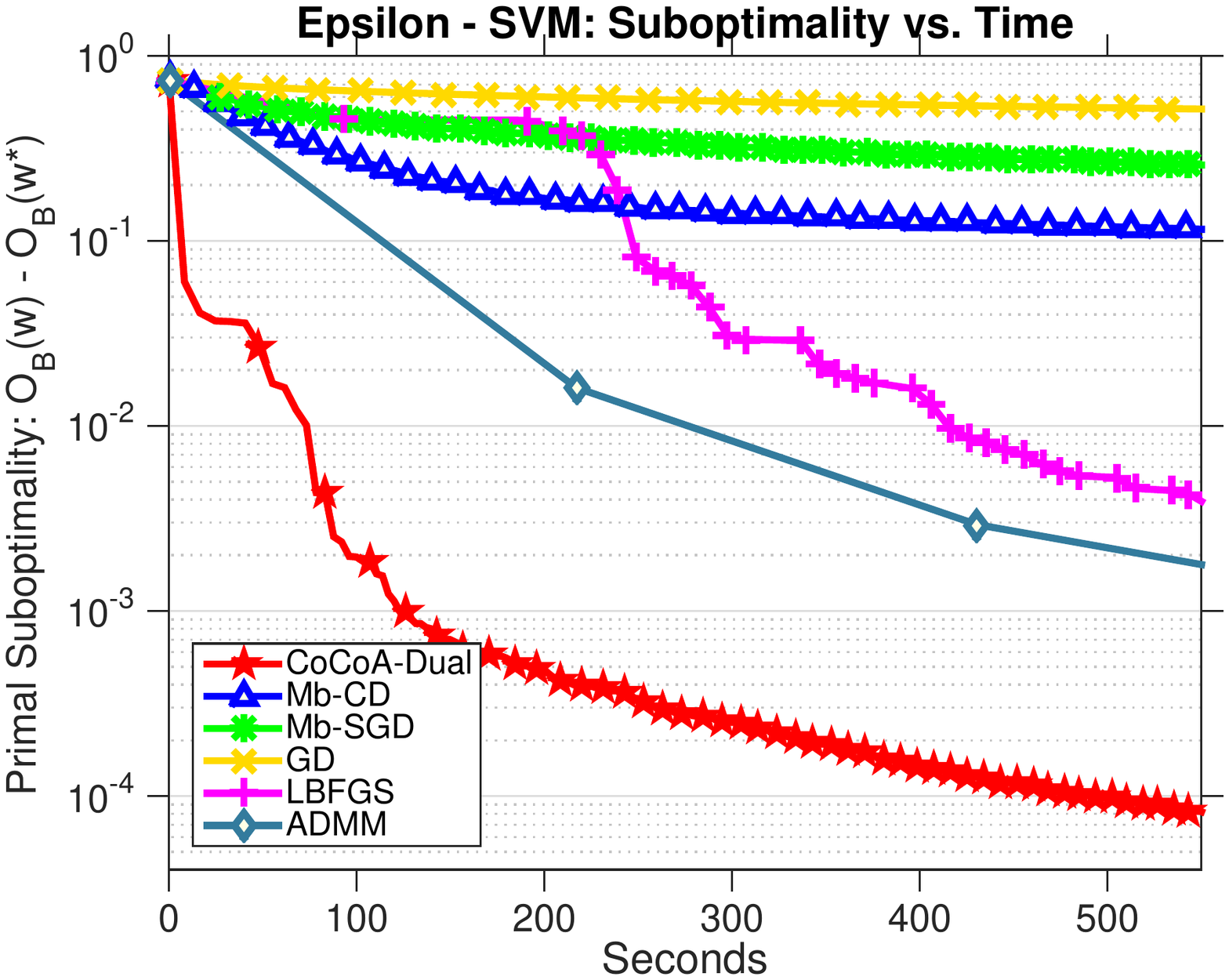}
\smalltrimfig{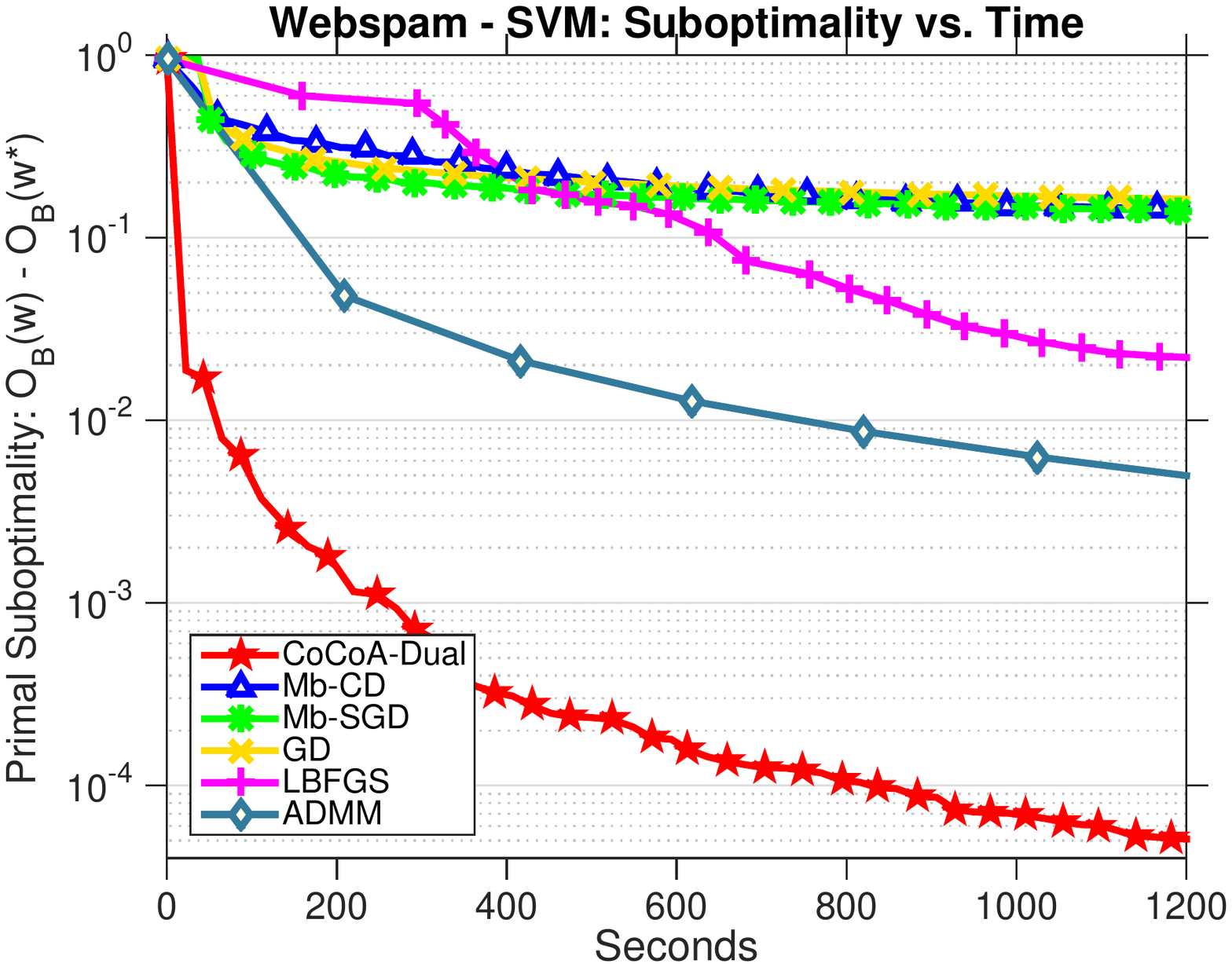}
\caption{Suboptimality in terms of ${\OB}(\wv)$ for solving a hinge-loss support vector machine model: url ($K$=4, $\lambda$=\mbox{1\sc{e}-4}), kddb ($K$=4, $\lambda$=\mbox{1\sc{e}-6}), epsilon ($K$=8, $\lambda$=\mbox{1\sc{e}-5}), and webspam ($K$=16, $\lambda$=\mbox{1\sc{e}-5}) datasets. \proxcocoa applied to the dual formulation converges more quickly than all other compared methods in terms of the time in seconds.}
\label{fig:svm}
\end{figure*}

In comparing methods in this setting (Figure \ref{fig:svm}), we measure the improvement to the primal objective ${\OB}(\wv)$ in terms of wall-clock time in seconds. We see again that \textsc{Mb-SGD} and \textsc{Mb-CD} are slow to converge, and come with the additional burden of having to tune extra parameters. ADMM performs the best of the methods other than \proxcocoa, followed by L-BFGS. However, both are still much slower to converge than \proxcocoa in the dual. ADMM in particular is affected by the fact that many internal iterations of SDCA are necessary in order to guarantee convergence. In contrast, \proxcocoa can incorporate arbitrary amounts of work locally and still converge. We note that although \proxcocoa, ADMM, and \textsc{Mb-SDCA} run in the dual, Figure~\ref{fig:svm} tracks progress towards the primal objective, $\OB(\wv)$.

\subsection{Primal vs. Dual: An Application to Elastic Net Regression}
\label{sec:primaldual}

To compare primal vs. dual optimization for \proxcocoa, we explore both variants by fitting an elastic net regression model~\eqref{ex:elasticnet} to two datasets. We use coordinate descent (with closed-form updates) as the local solver in both variants.
From the results in Figure~\ref{fig:primaldual}, we see that \proxcocoa in the dual 
 tends to perform better on datasets with a large number of training points (relative to the number of features), and that the performance deteriorates as the strong convexity in the problem disappears. 
In contrast, \proxcocoa in the primal performs well on datasets with a large number of features relative to training points, and is robust to changes in strong convexity. These changes in performance are to be expected, as we have discussed that \proxcocoa in the primal is more suited for non-strongly convex regularizers (Section~\ref{sec:primalexp}), and that the feature size dominates communication for \proxcocoa in the dual, as compared to the training point size for \proxcocoa in the primal (Section~\ref{sec:primalvsdual}).

\begin{figure}[h!]
\centering
\captionsetup{type=figure}
\smalltrimfig{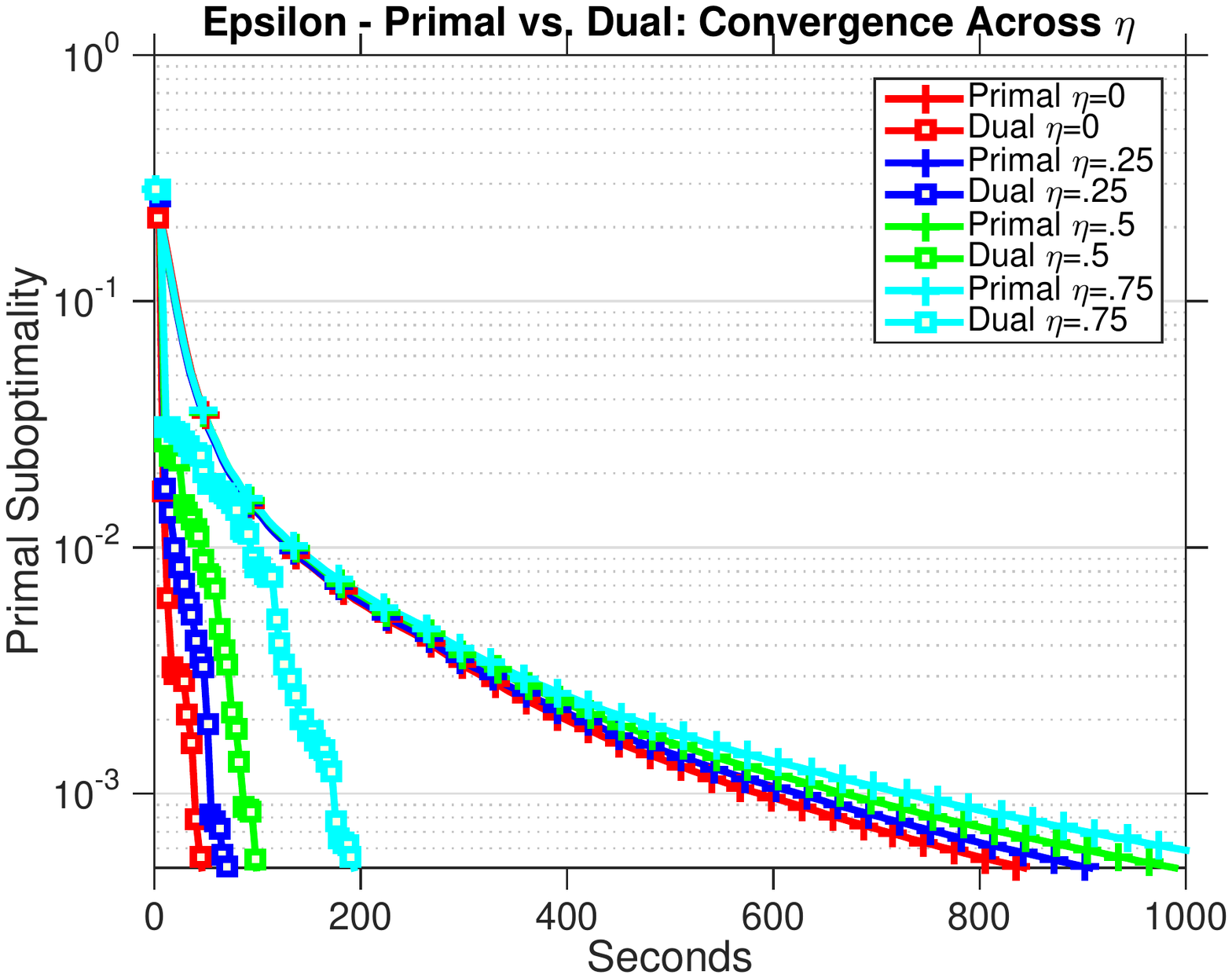} 
\smalltrimfig{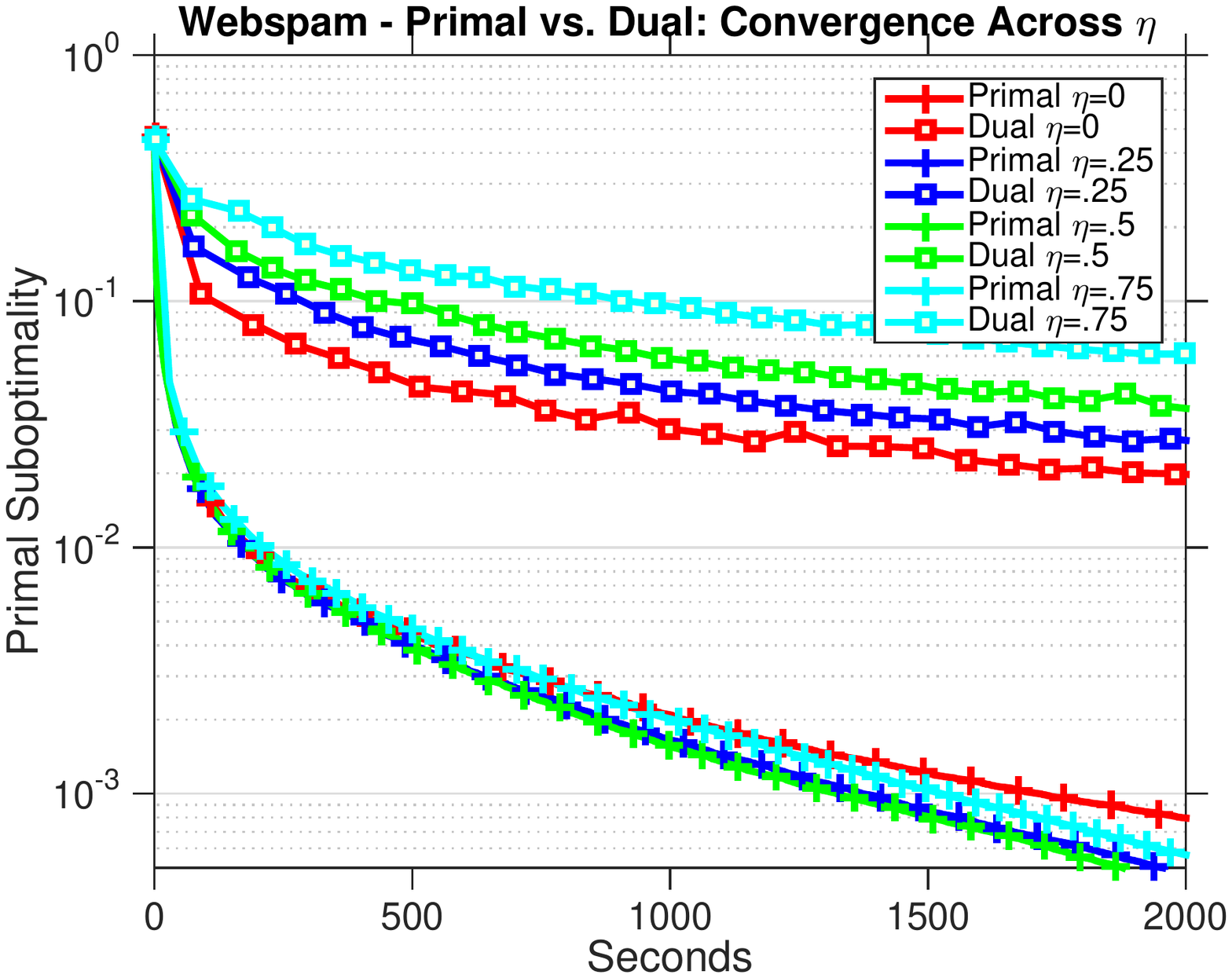} 
\caption{The convergence of \proxcocoa in the primal versus dual for various values of $\eta$ in an elastic net regression model.
\proxcocoa in dual performs better on the Epsilon dataset, where the training point size is the dominating term, and \proxcocoa in the primal performs better on the Webspam dataset, where the feature size is the dominating term. For both datasets, \proxcocoa in the dual is susceptible to changes in strong convexity---converging more quickly as the problem becomes more strongly convex ($\eta \to 0$), whereas \proxcocoa in the primal remains robust to changes in strong convexity.}
\label{fig:primaldual}
\end{figure}

\vspace{-1em}
\subsection{General Properties: Effect of Communication}
\label{sec:h}

Finally, we note that in contrast to the compared methods from Sections~\ref{sec:primalexp} and~\ref{sec:dualexp}, \proxcocoa comes with the benefit of having only a single parameter to tune: the subproblem approximation quality,~$\Theta$, which we control in our experiments via the number of local subproblem iterations, $H$, for the example of local coordinate descent. We further explore the effect of this parameter in Figure~\ref{fig:heffect}, and provide a general guideline for choosing it in practice (see Remark~\ref{rem:localtime}). 
In particular, we see that while increasing $H$ always results in better performance in terms of the number of communication rounds, smaller or larger values of~$H$ may result in better performance in terms of wall-clock time, depending on the cost of communication and computation. The flexibility to fine-tune $H$ is one of the reasons for \proxcocoa's significant performance gains.

\begin{figure}[h!]
\centering
\captionsetup{type=figure}
\tinytrimfig{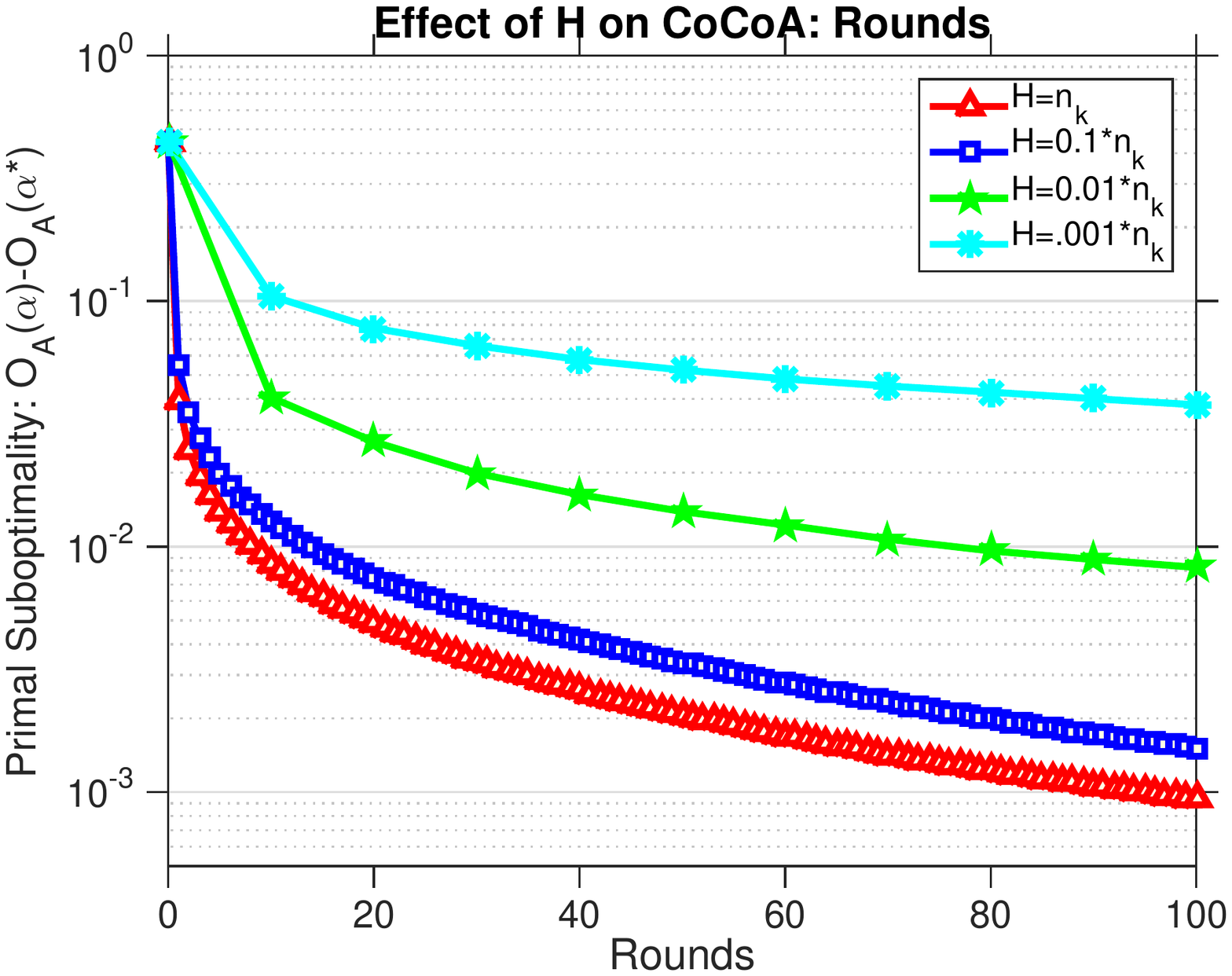}
\tinytrimfig{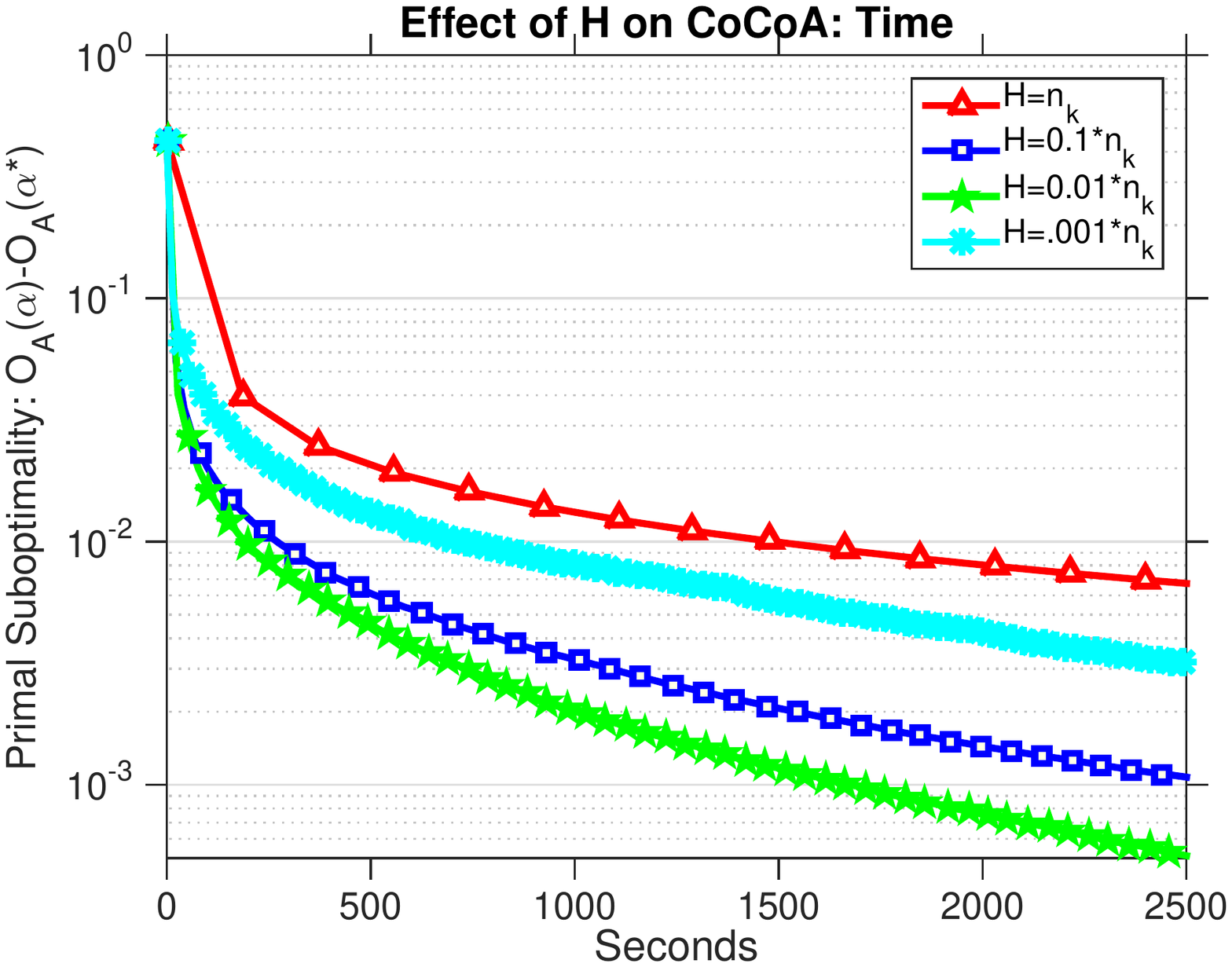}
\caption{{Suboptimality in terms of ${\OA}(\alphav)$ 
for solving lasso for the webspam dataset ($K$=16, $\lambda$=\mbox{1\sc{e}-5}). Here we illustrate how the work spent in the local subproblem (given by~$H$) affects the total performance of \proxcocoa in terms of number of rounds as well as wall time.}} 
\label{fig:heffect}
\end{figure}

\subsection{Experiment Details}
\label{sec:expdetails}
In this subsection we provide thorough details on the experimental setup and methods used in our comparison. All experiments are run on Amazon EC2 clusters of m3.xlarge machines, with one core per machine. The code for each method is written in \textsf{\small Apache Spark}, v1.5.0. Our code is open source and publicly available at~\href{http://gingsmith.github.io/cocoa/}{\texttt{gingsmith.github.io/cocoa/}}.

\paragraph{ADMM.} ADMM \citep{Boyd:2010bw} is a popular method that lends itself naturally to the distributed environment. For lasso regression, implementing ADMM for the problems of interest requires solving a large linear system $C\xv=\dv$ on each machine, where $C \in \R^{\n \times \n}$ with $\n$ scaling beyond $10^7$ for the datasets in Table~\ref{tab:datasets}, and with~$C$ being possibly dense. It is prohibitively slow to solve this directly on each machine, and we therefore employ conjugate gradient with early stopping (see, e.g., \citealp[Section 4.3]{Boyd:2010bw}). For SVM classification, we use stochastic dual coordinate ascent as an internal optimizer, which is shown in \cite{Zhang:2012bp} to have superior performance. We further improve performance with a varying rather than constant penalty parameter, as suggested in \citealp[Section 3.4.1]{Boyd:2010bw}.

\paragraph{Mini-batch SGD and proximal GD.} Mini-batch SGD is a standard and widely used method for parallel and distributed optimization. We use the optimized code provided in Spark's machine learning library, MLlib, v1.5.0 \citep{Meng:2015tu}. We tune both the size of the mini-batch and SGD step size using grid search. For lasso, we use the proximal version of the method. Full gradient descent can be seen as a specific setting of mini-batch SGD, where the mini-batch size is equal to the total number of training points. We thus also use the implementation in MLlib for full GD, and tune the step size parameter using grid search. 

\paragraph{Mini-batch CD and SDCA.} Mini-batch CD (for lasso) and SDCA (for SVM) aim to improve mini-batch SGD by employing coordinate descent, which has theoretical and practical justifications \citep{ShalevShwartz:2011vo,Takac:2013ut,Fercoq:2015kd}. We implement mini-batch CD and SDCA in Spark and scale the updates made at each round by $\frac{\beta}{b}$ for mini-batch size~$b$ and $\beta \in [1,b]$, tuning both parameters $b$ and $\beta$ via grid search. For the case of lasso regression, we implement Shotgun \citep{Bradley:2011wq}, which is a popular method for parallel optimization. Shotgun can be seen an extreme case of mini-batch CD where the mini-batch is set to $K$, i.e., there is a single update made by each machine per round. We see in the experiments that communicating this frequently becomes prohibitively slow in the distributed environment. 

\paragraph{OWL-QN.} OWN-QN \citep{Yu:2010vw} is a quasi-Newton method optimized in Spark's spark.ml package \citep{Meng:2015tu}. Outer iterations of OWL-QN make significant progress towards convergence, but the iterations themselves can be slow as they require processing the entire dataset. \proxcocoa, the mini-batch methods, and ADMM with early stopping all improve on this by allowing the flexibility to process only a subset of the dataset at each iteration. \proxcocoa and ADMM have even greater flexibility by allowing internal methods to process the dataset more than once. \proxcocoa makes this approximation quality explicit, both in theoretical convergence rates and via guidelines for setting the parameter. 

\paragraph{\proxcocoa.} We implement \proxcocoa with coordinate descent as the local solver. We note that since the framework and theory allow any internal solver to be used, \proxcocoa could benefit even beyond the results shown, e.g., by using existing fast $\Lone$-solvers for the single-machine case, such as \glmnet variants~\citep{Friedman:2010wm} or \textsc{blitz}~\citep{Johnson:2015tq} or SVM solvers like \textsc{liblinear}~\citep{Fan:2008tf}. The only parameter influencing the overall performance of \proxcocoa is the level of approximation quality, which we parameterize in the experiments through $H$, the number of local iterations of the iterative method run locally. Our theory relates local approximation quality to global convergence (Section~\ref{sec:convergence}), and we provide a guideline for how to choose this value in practice that links the parameter to the systems environment at hand (Remark~\ref{rem:localtime}). 

\section{Related Work}

There exist myriad optimization methods for the distributed setting; the following section is not meant to be wholly comprehensive, but to provide an overview of the most prevalent and related approaches. We additionally note that many new methods have been proposed since the time of submission of this manuscript in October 2016, including several extensions of the presented \proxcocoa framework---e.g., for 
federated learning~\citep{smith2017federated}, 
computing over heterogeneous systems~\citep{dunner2017efficient}, 
second-order algorithm extensions~\citep{gargiani2017hessian,lee2017distributed,dunner2018hessian,lee2018distributed}, 
and accelerated methods~\citep{ma2017accelerated,zheng2017general}. 
We defer the readers to these follow-up works for the most current literature review.

\label{sec:relatedwork}

\paragraph{Single-machine coordinate solvers.}
For strongly convex regularizers, 
 state-of-the-art for empirical loss minimization is randomized coordinate ascent on the dual (SDCA)~\citep{ShalevShwartz:2013wl} and accelerated variants \citep[e.g.,][]{ShalevShwartz:2014dy}. In contrast to primal stochastic gradient descent (SGD) methods, the SDCA family is often preferred as it is free of learning rate parameters and has faster (geometric) convergence guarantees.  
Interestingly, a similar trend in coordinate solvers exists in recent lasso literature, but with the roles of primal and dual reversed. For those problems, primal-based coordinate descent methods are state-of-the-art, as in \glmnet~\citep{Friedman:2010wm} and extensions~\citep{Yuan:2012wi}; see, e.g., the overview in \cite{Yuan:2010ub}. However, primal-dual rates for unmodified coordinate methods have to our knowledge only been obtained for strongly convex regularizers to date \citep{ShalevShwartz:2014dy,Zhang:2015vj}.

Coordinate descent on $\Lone$-regularized problems (i.e., \eqref{eq:primal} with $g(\cdot)=\lambda\|\cdot\|_1$) can be interpreted as the iterative minimization of a quadratic approximation of the smooth part of the objective, 
followed by a shrinkage step. In the single-coordinate update case, this is at the core of \glmnet~\citep{Friedman:2010wm,Yuan:2010ub}, and widely used in, e.g., solvers based on the primal formulation of $\Lone$-regularized objectives \citep{ShalevShwartz:2011vo,Yuan:2012wi,Bian:2013wx,Fercoq:2015kd,Tappenden:2015vha}. When changing more than one coordinate at a time, again employing a quadratic upper bound on the smooth part, this results in a two-loop method as in \glmnet for the special case of logistic regression.
In the distributed setting, when the set of active coordinates coincides with the ones on the local machine, these single-machine approaches closely resemble the distributed framework proposed here. 

\paragraph{Parallel methods.}

For the general regularized loss minimization problems of interest, methods based on stochastic subgradient descent (SGD) are well-established. Several variants of SGD have been proposed for parallel computing, many of which build on the idea of asynchronous communication \citep{Niu:2011wo, Duchi:2013te}.
Despite their simplicity and competitive performance on shared-memory systems, the downside of this approach in the distributed environment is that the amount of required communication is equal to the amount of data read locally, since one data point is accessed per machine per round (e.g., mini-batch SGD with a batch size of one per worker). These variants are in practice not competitive with the more communication-efficient methods considered in this work, which allow more local updates per communication round.

For the specific case of $\Lone$-regularized objectives, parallel coordinate descent (with and without using mini-batches) was proposed in~\cite{Bradley:2011wq} (Shotgun) and generalized in~\cite{Bian:2013wx}
, and is among the best performing solvers in the parallel setting.
Our framework reduces to Shotgun as a special case when the internal solver is a single-coordinate update on the subproblem \eqref{eq:subproblem}, $\gamma=1$, and for a suitable~$\sigma'$. However, Shotgun is not covered by our convergence theory, since it uses a potentially unsafe 
upper bound of~$\beta$ instead of $\sigma'$, which isn't guaranteed to satisfy our condition for convergence~\eqref{eq:sigmaPrimeSafeDefinition}. We compare empirically with Shotgun in Section~\ref{sec:experiments} to highlight the detrimental effects of running this high-communication method in the distributed environment.

\paragraph{One-shot communication schemes.}
At the other extreme, there are distributed methods that use only a single round of communication, such as \cite{Mann:2009tr,Zinkevich:2010tj,Zhang:2013wq,McWilliams:2014tl}; and \cite{Heinze:2016tu}.
These methods require additional assumptions on the partitioning of the data, which are usually not satisfied in practice if the data are distributed ``as is'', i.e., if we do not have the opportunity to distribute the data in a specific way beforehand. 
Furthermore, some cannot guarantee convergence rates beyond what could be achieved if we ignored data residing on all but a single computer, as shown in \cite{Shamir:2014vf}.
Additional relevant lower bounds on the minimum number of communication rounds necessary for a given approximation quality are presented in \cite{Balcan:2012tc} and \cite{Arjevani:2015vka}.

\paragraph{Mini-batch methods.} Mini-batch methods (which use updates from several training points or features per round) are more flexible and lie within the two extremes of parallel and one-shot communication schemes. However,
mini-batch versions of both SGD and coordinate descent (CD) (e.g., \citealp{Dekel:2012wm, Takac:2013ut, shalev2013accelerated, Shamir:2014tp, qu2015quartz, richtarik2016distributed, defossez2017adabatch}) suffer from their convergence rate degrading towards the rate of batch gradient descent as the size of the mini-batch is increased. 
This follows because mini-batch updates are made based on the outdated previous parameter vector $\wv$, in contrast to methods that allow immediate local updates like \proxcocoa.

Another disadvantage of mini-batch methods is that the aggregation parameter is more difficult to tune, as it can lie anywhere in the order of mini-batch size. The optimal choice is often either unknown or too challenging to compute in practice.
In the \proxcocoa framework there is no need to tune parameters, as the aggregation parameter and subproblem parameters can be set directly using the safe bound discussed in Section~\ref{sec:proxcocoa} (Definition~\ref{def:sigma}).

\paragraph{Batch solvers.}
ADMM \citep{Boyd:2010bw}, gradient descent, and quasi-Newton methods such as L-BFGS and are also often used in distributed settings because of their relatively low communication requirements. However, they require at least a full (distributed) batch gradient computation at each round, and therefore do not allow the gradual trade-off between communication and computation provided by \proxcocoa. 
In Section~\ref{sec:experiments}, we include experimental comparisons with ADMM, gradient descent, and L-BFGS variants, including orthant-wise limited memory quasi-Newton (OWL-QN) for the $\Lone$ setting \citep{Andrew:2007cu}.

Finally, we note that while the convergence rates provided for \proxcocoa mirror the convergence class of classical batch gradient methods in terms of the number of outer rounds, existing batch gradient methods come with a weaker theory, as they do not allow general inexactness $\Theta$ for the local subproblem~\eqref{eq:subproblem}. 
In contrast, our convergence rates incorporate this approximation directly, and, moreover, hold for arbitrary local solvers of much cheaper cost than batch methods (where in each round, every machine has to process exactly a full pass through the local data). This makes \proxcocoa more flexible in the distributed setting, as it can adapt to varied communication costs on real systems. We have seen in  Section~\ref{sec:experiments} that this flexibility results in significant performance gains over the competing methods.

\paragraph{Distributed solvers.}

By making use of the primal-dual structure in the line of work of \cite{Yu:2012fp,Pechyony:2011wi,Yang:2013vl,Yang:2013ui} and \cite{Lee:2015vra}, the \cocoa-v1 and \cocoap frameworks (which are special cases of the presented framework, \proxcocoa) are the first to allow the use of any local solver---of weak local approximation quality---in each round in the distributed setting. The practical variant of the DisDCA \citep{Yang:2013vl}, called DisDCA-p, allows for additive updates in a similar manner to \cocoa, but is restricted to coordinate decent (CD) being the local solver, and was initially proposed without convergence guarantees. DisDCA-p, \cocoa-v1, and \cocoap are all limited to strongly convex regularizers, and therefore are not as general as the \proxcocoa framework discussed in this work. 

In the $\Lone$-regularized setting, an approach related to our framework includes  distributed variants of \glmnet as in \cite{Mahajan:2014tg}. Inspired by \glmnet and \cite{Yuan:2012wi}, the works of \cite{Bian:2013wx} and \cite{Mahajan:2014tg} 
introduced the idea of a block-diagonal Hessian upper approximation in the distributed $\Lone$ context.
The later work of \cite{Trofimov:2014vb,Trofimov:2016tu} specialized this approach to sparse logistic regression.

If hypothetically each of our quadratic subproblems $\Ggk(\vsubset{\Delta \alphav}{k})$ as defined in \eqref{eq:subproblem} were to be minimized exactly, the resulting steps could be interpreted as block-wise Newton-type steps on each coordinate block~$k$, where the Newton-subproblem is modified to also contain the $\Lone$-regularizer \citep{Mahajan:2014tg,Yuan:2012wi,Qu:2015ve}. 
While \cite{Mahajan:2014tg} allows a fixed accuracy for these subproblems, but not arbitrary approximation quality $\Theta$ as in our framework, the works of
\cite{Trofimov:2016tu,Yuan:2012wi}; and \cite{Yen:2015vy} assume that the quadratic subproblems are solved exactly. 
Therefore, these methods are not able to freely trade off communication and computation. Also, they do not allow the re-use of arbitrary local solvers. 
On the theoretical side, the convergence rate results provided by \cite{Mahajan:2014tg,Trofimov:2016tu}; and \cite{Yuan:2012wi} are not explicit convergence rates but only asymptotic, as the quadratic upper bounds are not explicitly controlled for safety as with our $\sigma'$.

\section{Discussion}

To enable large-scale machine learning and signal processing, we have developed, analyzed, and evaluated a general-purpose framework for communication-efficient primal-dual optimization in the distributed environment. Our framework, \cocoa, takes a unique approach by using duality to derive subproblems for each machine to solve in parallel. These subproblems closely match the global problem of interest, which allows for state-of-the-art single-machine solvers to easily be re-used in the distributed setting. Further, by allowing the local solvers to find solutions of arbitrary approximation quality to the subproblems on each machine, our framework permits a highly flexible communication scheme. In particular, as the local solvers make updates directly to their local parameters, the need to communicate reduces and can be adapted to the system at hand, which helps to manage the communication bottleneck in the distributed setting.

We analyzed the impact of the local solver approximation quality and derived global primal-dual convergence rates for our framework that are agnostic to the specifics of the local solvers. We have taken particular care in extending our framework to the case of non-strongly convex regularizers, where we introduced a bounded-support modification technique to provide robust convergence guarantees. Finally, we demonstrated the efficiency of our framework in an extensive experimental comparison with state-of-the-art distributed solvers. Our framework achieves up to a 50$\times$ speedup over other widely-used methods on real-world distributed datasets.

\acks{We thank Michael P. Friedlander, Matilde Gargiani, Sai Praneeth Karimireddy, Jakub Kone{\v c}n{\'y}, Ching-pei Lee, and Peter Richt{\'a}rik
for their help and for fruitful discussions. We are additionally grateful to the reviewers for their valuable comments.
We wish to acknowledge support from the U.S. National Science Foundation, under award number NSF:CCF:1618717, NSF:CMMI:1663256 and NSF:CCF:1740796; the Swiss National Science Foundation, under grant number 175796;
and the Mathematical Data Science program of the
Office of Naval Research, under grant number N00014-15-1-2670.
}

\appendix

\section{Convex Conjugates}
\label{app:conjugates}

The convex conjugate of a function $f: \R^\d\rightarrow \R$ is defined as 
\begin{equation}
f^*(\vv) := \max_{\uv\in\R^\d} \vv^\top \uv - f(\uv) \, .
\end{equation}
Below we list several useful properties of conjugates \cite[see, e.g.,][Section 3.3.2]{Boyd:2004uz}:
\vspace{-1mm}
\begin{itemize} 
\addtolength{\itemindent}{.25em}
\item Double conjugate: \hspace{1em}
$(f^*)^* = f$ if $f$ is closed and convex.
\item Value Scaling:  (for $\alpha>0$) \hspace{2em}
$
f(\vv) = \alpha g(\vv) 
\qquad\Rightarrow\qquad
f^*(\wv) = \alpha g^*(\wv/\alpha)    \, .
$
\item Argument Scaling:  (for $\alpha\ne0$) \hspace{1em}
$
f(\vv) = g(\alpha \vv) 
\qquad\Rightarrow\qquad
f^*(\wv) = g^*(\wv/\alpha) \, .
$
\item Conjugate of a separable sum: \hspace{1em}
$
f(\vv)=\sum_i \phi_i(v_i)
\qquad\Rightarrow\qquad
f^*( \wv ) = \sum_i \phi_i^* ( w_i ) \, .
$
\end{itemize}\vspace{1mm}

\begin{lemma}[{Duality between Lipschitzness and L-Bounded Support, \cite[Corollary 13.3.3]{Rockafellar:1997ww}}]
\label{lem:dualLipschitz}
Given a proper convex function $f$, it holds that
$f$ is $L$-Lipschitz
if and only if 
$f^*$ has $L$-bounded support.
\end{lemma}

\begin{lemma}[{Duality between Smoothness and Strong Convexity, \cite[Theorem 4.2.2]{hiriart-urruty:2001df}}]
\label{lem:dualSmooth}
Given a closed convex function~$f$, it holds that
$f$ is $\mu$ strongly convex w.r.t. the norm $\|\cdot\|$
if and only if
$f^*$ is $(1/{\mu})$-smooth w.r.t. the dual norm~$\|\cdot\|_*$.
\end{lemma}

\section{Proofs of Primal-Dual Relationships}
\label{app:primaldual}

In the following subsections we provide derivations of the primal-dual relationship of the general objectives~\eqref{eq:primal} and \eqref{eq:dual}, and then show how to derive the conjugate of the modified $\Lone$-norm, as an example of the bounded-support modification introduced in Section~\ref{sec:convergence}.

\subsection{Primal-Dual Relationship}
The relation of the original formulation \eqref{eq:primal} to its dual formulation \eqref{eq:dual} is standard in convex analysis.
Using the linear map $A$ as in our case, the relationship is an instance of Fenchel-Rockafellar Duality, see e.g. \citet[Theorem 4.4.2]{Borwein:2005ge} or \citet[Proposition 15.18]{Bauschke:2011ik}. For completeness, we illustrate this correspondence with a self-contained derivation of the duality.

Starting with the original formulation \eqref{eq:primal}, we introduce an auxiliary vector $\vv\in \R^\d$ representing $\vv =A\alphav$. Then optimization problem~\eqref{eq:primal} becomes:
\begin{equation}
\label{eq:constrainedprimal}
\min_{\alphav, \vv} \quad  f(\vv) + g( \alphav) \quad \text{such that} \ \vv =A\alphav \, .
\end{equation}
Introducing Lagrange multipliers $\wv \in \R^\d$,  the Lagrangian is given by:
\[
L(\alphav, \vv; \wv) := f(\vv) +   g(\alphav) + \wv^\top\left(A\alphav-\vv\right) \, .
\]
The dual problem of \eqref{eq:primal} follows by taking the infimum with respect to both $\alphav$ and $\vv$:
\begin{align}
\inf_{\alphav, \vv} L(\wv, \alphav, \vv) & =   \inf_{\vv} \left\{ f(\vv) - \wv^\top \vv \right\} + \inf_{\alphav} \left\{ g(\alphav) +  \wv^\top A\alphav\right\} \notag \\
& =  - \sup_{\vv} \left\{  \wv^\top \vv - f(\vv) \right\}- \sup_{\alphav} \left\{(-\wv^\top A)\alphav -  g(\alphav) \right\} \notag\\
& = - f^*(\wv) - g^*(-A^\top \wv) \label{eq:Lagrangian}\, .
\end{align}
We change signs and turn the  maximization of the dual problem \eqref{eq:Lagrangian} into a minimization, thereby arriving at the dual formulation $\eqref{eq:dual}$ as claimed:
\[
    \min_{\wv \in \R^\d} \quad \Big[ \ 
    \OB(\wv) := f^*(\wv) + g^*(-A^\top \wv) \ \Big] \, .
\]

\subsection{Continuous Conjugate Modification for Indicator Functions}

\begin{lemma}[Conjugate of the modified $\Lone$-norm]
\label{lem:l1surrogate}
The convex conjugate of the bounded support modification of the $\Lone$-norm, 
as defined in~\eqref{eq:boundedSup}, is:
\[
    \bar{g}_i^*(x) := 
    \begin{cases}
            0 & : x \in [-1,1],  \\
            B(|x| - 1) & : \text{otherwise,}
        \end{cases}
\]
and is $B$-Lipschitz.
\end{lemma}

\begin{proof}
We start by applying the definition of convex conjugate:
\[
\bar{g_i}(\alpha) = \sup_{x \in \R} \left[ \alpha x - \bar{g}^*_i(x) \right] \, .
\]
We begin by looking at the case in which $\alpha \geq B$; in this case it's easy to see that when $x \to +\infty$, we
have:
\[
\alpha x - B(|x|-1) = (\alpha - B)x - B \to +\infty \, ,
\]
as $\alpha - B \geq 0$. The case $\alpha \leq -B$ holds analogously.
We'll now look at the case $\alpha \in [0,B]$; in this case it is clear we must have $x^{\star} \ge 0$.
 It also must hold that $x^{\star} \leq 1$, since
\[
\alpha x - B(x-1) < \alpha x \, ,
\]
for every $x > 1$. Therefore the maximization becomes
\[
\bar{g_i}(\alpha) = \sup_{x \in [0,1]} \alpha x \, ,
\]
which has maximum $\alpha$ at $x = 1$.  The remaining $\alpha \in [-B,0]$ case follows in similar fashion.

Lipschitz continuity of $\bar{g}^*_i$ follows directly, or alternatively also from the general result that $g^*_i$ is $L$-Lipschitz if and only if $g_i$ has $L$-bounded support \cite[Corollary 13.3.3]{Rockafellar:1997ww} or \cite[Lemma 5]{Dunner:2016vga}.
\end{proof}

\section{Comparison to ADMM}
\label{app:admm}

\subsection{ADMM Applied to the $\OB(\cdot)$ Formulation}
Here we compare consensus ADMM~\citep{Mota:2013ja} applied to the problem~\eqref{eq:dual} to the \cocoa framework, as discussed in Section~\ref{sec:admm}. For consensus ADMM, the objective~$\OB(\cdot)$ can be decomposed using the following re-parameterization:
\begin{align*}
\min_{\wv_1, \dots \wv_K, \wv} & \quad \sum_{k=1}^K \sum_{i \in \Pk} g_i^*(-\xv_i^\top\wv_k) + f^*(\wv) \\
s.t. & \quad \wv_k = \wv, \, \, k = 1, \dots, K.
\end{align*}
To solve this problem, we construct the augmented Lagrangian:
\begin{align*}
L_{\rho}(\wv_1, \dots, \wv_k, \wv, \uv_1, \dots, \uv_k) &:= 
\sum_{k=1}^K \sum_{i \in \Pk} g_i^*(-\xv_i^\top\wv_k) \\ &+ f^*(\wv) + \sum_{k=1}^K\uv_k^\top(\wv_k-\wv) + \frac{\rho}{2} \sum_{k=1}^K \|\wv_k-\wv\|^2 \, ,
\end{align*}
which yields the following decomposable updates:
\begin{align*}
\wv_k^{(t+1)} & = \argmin_{\wv_k} \, \sum_{i \in \Pk}  g_i^*(-\xv_i^\top\wv_k) + {\uv_k^{(t)}}^\top(\wv_k - \wv^{(t)}) +  \frac{\rho}{2}\|\wv_k - \wv^{(t)}\|^2, \\
\wv^{(t+1)} & = \argmin_{\wv} \, f^*(\wv) +  \sum_{k=1}^K{\uv_k^{(t)}}^\top(\wv_k^{(t+1)}-\wv) + \frac{\rho}{2} \sum_{k=1}^K \|\wv_k^{(t+1)}-\wv\|^2, \\
\uv_{k}^{(t+1)} & = \uv_{k}^{(t)} + \rho(\wv_{k}^{(t+1)} - \wv^{(t+1)}).
\end{align*}
These updates can be further simplified by using the scaled form of $\uv_k$ and combining terms for $\wv$ using the averages $\bar{\wv}_k$ and $\bar{\uv}_k$:
\begin{align*}
\wv_k^{(t+1)} & = \argmin_{\wv_k} \, \sum_{i \in \Pk}  g_i^*(-\xv_i^\top\wv_k) + \rho{\uv_k^{(t)}}^\top(\wv_k - \wv^{(t)}) +  \frac{\rho}{2}\|\wv_k - \wv^{(t)}\|^2, \\
\wv^{(t+1)} & = \argmin_{\wv} \, f^*(\wv) + \frac{\rho K}{2} \|\wv - (\bar{\wv}_k^{(t+1)}+\bar{\uv}_k^{(t)})\|^2, \\
\uv_{k}^{(t+1)} & = \uv_{k}^{(t)} + \wv_{k}^{(t+1)} - \wv^{(t+1)} .
\end{align*}
To compare this to the \cocoa subproblems~\eqref{eq:subproblem}, we will derive the dual form of the update to $\wv_k$. Suppressing the iteration counter for simplicity, the minimization is of the form:
\begin{align*}
& \min_{\wv_k} \sum_{i \in \Pk} g_i^*(-\xv_i^\top\wv_k) + \rho {\uv_k}^\top(\wv_k - \wv) +  \frac{\rho}{2}\|\wv_k - \wv\|^2 \\
 = & \min_{\wv_k} \sum_{i \in \Pk}\max_{\alpha_i} -\xv_i^\top\wv_k\alpha_i - g_i(\alpha_i) + \rho {\uv_k}^\top(\wv_k - \wv) +  \frac{\rho}{2}\|\wv_k - \wv\|^2 \\
 = &  \max_{\vsubset{\alphav}{k}} \min_{\wv_k} -\wv_k^\top A_{[k]}\vsubset{\alphav}{k} - \sum_{i \in \Pk} g_i({\vsubset{\alphav}{k}}_i) + \rho {\uv_k}^\top(\wv_k - \wv) +  \frac{\rho}{2}\|\wv_k - \wv\|^2 \, . 
\end{align*}
Solving the minimization yields: $\wv_k = \frac{1}{\rho} A_{[k]}\vsubset{\alphav}{k} - \uv_k + \wv$. Plugging this back in, we have:
\begin{align*}
 & \max_{\vsubset{\alphav}{k}}\sum_{i \in \Pk} -g_i({\vsubset{\alphav}{k}}_i) -(\frac{1}{\rho} A_{[k]}\vsubset{\alphav}{k} - \uv_k + \wv)^\top A_{[k]}\vsubset{\alphav}{k} + \rho {\uv_k}^\top(\frac{1}{\rho} A_{[k]}\vsubset{\alphav}{k} - \uv_k) +  \frac{\rho}{2}\|\frac{1}{\rho} A_{[k]}\vsubset{\alphav}{k} - \uv_k\|^2 \\
  =  & \max_{\vsubset{\alphav}{k}}\sum_{i \in \Pk} -g_i({\vsubset{\alphav}{k}}_i) + \uv_k^\top A_{[k]}\vsubset{\alphav}{k} - \wv^\top A_{[k]}\vsubset{\alphav}{k} -  \frac{1}{2\rho}\| A_{[k]}\vsubset{\alphav}{k}\|^2  \\
   = & \min_{\vsubset{\alphav}{k}}\sum_{i \in \Pk} g_i({\vsubset{\alphav}{k}}_i) +(\wv - \uv_k)^\top A_{[k]}\vsubset{\alphav}{k} + \frac{1}{2\rho}\| A_{[k]}\vsubset{\alphav}{k}\|^2  \, .
\end{align*}
We therefore see that the update to $\wv_k$ has a similar form to the \cocoa subproblem~\eqref{eq:subproblem}, where $\rho := \frac{\tau}{\sigma'}$.

\subsection{ADMM Applied to the $\OA(\cdot)$ Formulation}
We can also compare \cocoa to ADMM as applied to the~\eqref{eq:primal} problem. For consensus ADMM, the objective~$\OA(\cdot)$ can be decomposed using the following re-parametrization, which introduces local copies ${\alphav}_k$ of the global variable ${\alphav}$, and a set of consensus constraints to achieve the equality between them:
\begin{align*}
\min_{{\alphav}, {\alphav}_{1}, \dots {\alphav}_{K}} &  \quad  f(A{\alphav})+ \sum_{k=1}^K \sum_{i \in \Pk} g_i(\alphav_{k,i})  \\
s.t. & \quad \vsubset{A}{k}{\alphav}=\vsubset{A}{k}{\alphav}_{k}, \, \, k = 1, \dots, K.
\end{align*}
To solve this problem, we construct the augmented Lagrangian with a penalty parameter $\rho$:
\begin{align*}
L_{\rho}(\tilde{\alphav},\alphav_{1}, \dots, {\alphav}_{K}, \wv_1, \dots, \wv_k) &:= f(A\tilde{\alphav}) + \\
&  \sum_{k=1}^K \left[\sum_{i \in \Pk} g_i({\alphav}_{k,i}) + \wv_k^\top\vsubset{A}{k}({\alphav}_{k} - {\alphav}) + \frac{\rho}{2} \| \vsubset{A}{k} ({\alphav}_{k} - {\alphav}) \|_2^2\right]
\end{align*}
which yields to the following decomposable updates:
\begin{alignat*}{2}
&{\alphav}_{k}^{(t)}  &&= \argmin_{{\alphav}_{k}} \, \sum_{i \in \Pk}  g_i({{\alphav}_{k,i}}) +{\wv_k^{(t-1)}}^\top \vsubset{A}{k}({\alphav}_{k}-{\alphav}^{(t-1)}) + \frac{\rho}{2} \| \vsubset{A}{k} ({\alphav}_{k} - {\alphav}^{(t-1)}) \|_2^2,\\
&{\alphav}^{(t)} && = \argmin_{{\alphav}} \, f(A{\alphav}) +  \sum_{k=1}^K \left[\wv_k^{(t-1)\top}\vsubset{A}{k}({\alphav}_{k}^{(t)} - {\alphav}) + \frac{\rho}{2} \| \vsubset{A}{k} ({\alphav}_{k}^{(t)} - {\alphav}) \|_2^2\right], \\
&\wv_{k}^{(t)} &&  = \wv_{k}^{(t-1)} + \rho\vsubset{A}{k}({\alphav}_{k}^{(t)}-{\alphav}^{(t)}).
\end{alignat*}
The first minimization is solved locally in a distributed manner by the $K$ partitions. By setting $\rho := \frac{\sigma'}{\tau}$ and applying the change of variables $\Delta \vsubset{\alphav}{k}=({\alphav}_{k}-{{\alphav}}^{(t-1)})$, we can obtain the \proxcocoa local subproblems~\eqref{eq:subproblem}:
\begin{align*}
&\Delta {\vsubset{\alphav}{k}} ^{(t)}= \argmin_{\Delta \vsubset{\alphav}{k}} \sum_{i \in \Pk} g({\alpha}_i^{(t-1)}+\Delta{\vsubset{\alphav}{k}}_i) + {\wv^{(t-1)}}^\top \vsubset{A}{k}\Delta \vsubset{\alphav}{k} + \frac{\sigma'}{2 \tau} \| \vsubset{A}{k} \Delta \vsubset{\alphav}{k}\|^2
\end{align*}
Note that the update to ${\alphav}$ in its current form is not separable and must be solved via some distributed optimization procedure. The comparison to the $\OB(\cdot)$ formulation is more natural in this sense, as it captures a setting and formulation in which distributed ADMM would more commonly be applied. However, the above formulation is closely related to the \textit{sharing} variant of ADMM~\cite[Section 7.3]{Boyd:2010bw}, where data for canonical regularized loss minimization problems is assumed to be distributed via features~\cite[see, e.g.,][Section 8.3]{Boyd:2010bw}.

\section{Convergence Proofs}
\label{app:convgproofs}

In this section we provide proofs of our main convergence results. The arguments follow the reasoning in \cite{Ma:2015ti,Ma:2017dx}, but where we have generalized them to be applicable directly to~\eqref{eq:primal}. We provide full details of Lemma~\ref{lem:RelationOfDTOSubproblems} as a proof of concept, but omit details in later proofs that can be derived using the arguments in \cite{Ma:2015ti} or earlier work of \cite{ShalevShwartz:2013wl}, and instead outline the proof strategy and highlight sections where the theory deviates.

\subsection{Approximation of $\OA(\cdot)$ by the Local Subproblems $\Ggk(\cdot)$}

Our first lemma in the overall proof of convergence helps to relate progress on the local subproblems to the global objective $\OA(\cdot)$.
\begin{replemma}{lem:RelationOfDTOSubproblems}
For any dual variables
$\alphav, \Delta \alphav 
\in \R^{\n}$, $\vv = \vv(\alphav) := A\alphav$, and real values $\aggpar,\sigma'$ satisfying~\eqref{eq:sigmaPrimeSafeDefinition}, it holds that
\begin{equation}
  \OA\Big(
\alphav +\aggpar 
\sum_{k=1}^K
\vsubset{\Delta \alphav}{k}\!
\Big) 
 \leq 
 (1-\aggpar) \OA(\alphav)  + \aggpar 
 \sum_{k=1}^K 
 \Ggk(\vsubset{\Delta \alphav}{k}; \vv, \vsubset{\alphav}{k}) \, .
\end{equation}
\end{replemma}
\begin{proof}
\allowdisplaybreaks
In this proof we follow the line of reasoning in \citet[Lemma 4]{Ma:2015ti} with a more general $(1/\tau)$ smoothness assumption on $f(\cdot)$. An outer iteration of \proxcocoa performs the following update:
\begin{align}
\label{eq:dualtosubproblem1}
\OA(\alphav+\gamma \sum_{k=1}^K\vsubset{\Delta \alphav}{k})
&= \underbrace{ f(\vv(\alphav + \gamma \sum_{k=1}^K \vsubset{\Delta \alphav}{k}))}_A +
\underbrace{\sum_{i=1}^{n}
g_i(\alpha_i +\gamma (\sum_{k=1}^K \vsubset{\Delta \alphav}{k})_i)}_{B} \, .
\end{align}

We bound $A$ and $B$ separately. First we bound A using $(1/\tau)$-smoothness of $f$:
\begin{align*}
A &= f\Big(\vv(\alphav + \gamma \sum_{k=1}^K \vsubset{\Delta \alphav}{k})\Big) =
f\Big(\vv(\alphav) + \gamma \sum_{k=1}^K \vv(\vsubset{\Delta \alphav}{k})\Big) \\
& \overset{\text{smoothness of $f$ as in \eqref{eq:smooth}}}{\leq} f(\vv(\alphav)) +\sum_{k=1}^K \gamma \nabla f(\vv(\alphav))^\top  \vv(\vsubset{\Delta \alphav}{k})  + \frac{\gamma^2}{2\tau} \|\sum_{k=1}^K \vv(\vsubset{{\Delta}\alphav}{k})\|^2 \\
& \overset{\text{definition of $\wv$ as in \eqref{eq:dualPdualrelation}}}{=} f(\vv(\alphav)) +\sum_{k=1}^K \gamma \vv(\vsubset{\Delta \alphav}{k})^\top\wv(\alphav)  + \frac{\gamma^2}{2\tau} \|\sum_{k=1}^K \vv(\vsubset{{\Delta}\alphav}{k})\|^2 \\
& \overset{\text{safe choice of $\sigma'$ as in \eqref{eq:sigmaPrimeSafeDefinition}}}{\leq} f(\vv(\alphav)) +\sum_{k=1}^K \gamma \vv(\vsubset{\Delta \alphav}{k})^\top\wv(\alphav)  + \frac{1}{2\tau}\gamma \sigma' \sum_{k=1}^K \| \vv(\vsubset{{\Delta}\alphav}{k})\|^2 \, .
\end{align*}  
 
Next we use Jensen's inequality to bound B:
\begin{align*}
B = \sum_{k=1}^K \left( \sum_{i\in \Pk} g_i(\alpha_i+\gamma   (\vsubset{\Delta \alphav}{k})_i) \right)
& = \sum_{k=1}^K \left( \sum_{i\in \Pk} g_i((1-\gamma)   \alpha_i+\gamma (\alpha_i + (\vsubset{\Delta \alphav}{k})_i)) \right) \\
&\leq  \sum_{k=1}^K \left( \sum_{i\in \Pk} (1-\gamma) g_i(\alpha_i) +\gamma g_i(\alpha_i + ({\vsubset{\Delta \alphav}{k}})_i) \right) \, .
\end{align*}

Plugging $A$ and $B$ back into~\eqref{eq:dualtosubproblem1} yields:
\begin{align*}
\OA\Big(\alphav & +\gamma \sum_{k=1}^K \vsubset{\Delta \alphav}{k}\Big) 
\le \   f(\vv(\alphav)) \pm \gamma f(\vv(\alphav))
+\sum_{k=1}^K \gamma  \vv(\vsubset{\Delta \alphav}{k})^\top\wv(\alphav)   + \frac{1}{2\tau}\gamma \sigma' \sum_{k=1}^K \| \vv(\vsubset{{\Delta}\alphav}{k})\|^2 \\
& +  \sum_{k=1}^K\sum_{i\in \Pk} (1-\gamma) g_i(\alpha_i) +\gamma g_i(\alpha_i + ({\vsubset{\Delta \alphav}{k}})_i) \\
= \ & \underbrace{(1-\gamma) f(\vv(\alphav)) +  \sum_{k=1}^K \left(\sum_{i\in \Pk} (1-\gamma) g_i(\alpha_i)  \right)}_{(1-\gamma) \OA(\alphav)} \\
&  +  \gamma \sum_{k=1}^K \left(\frac{1}{K} f(\vv(\alphav)) + \vv(\vsubset{\Delta \alphav}{k})^\top\wv(\alphav) + \frac{\sigma'}{2\tau}  \| \vv( \vsubset{{\Delta}\alphav}{k})\|^2 + \sum_{i\in \Pk} g_i(\alpha_i + ({\vsubset{\Delta \alphav}{k}})_i)  \right) \\
\overset{\eqref{eq:subproblem}}{=}& \ (1-\gamma) \OA(\alphav) +\gamma \sum_{k=1}^K \Ggk(  \vsubset{\Delta \alphav}{k}; \vv, \vsubset{\alphav}{k}) \, , 
\end{align*}
where the last equality is by the definition of the subproblem objective $\Ggk(.)$ as in \eqref{eq:subproblem}.
\end{proof}

\subsection{Proof of Main Convergence Result (Theorem \ref{thm:convergenceNonsmooth})}

Before proving the main convergence results, we introduce several useful quantities, and establish the following lemma, which characterizes the effect of iterations of Algorithm~\ref{alg:generalizedcocoa} on the duality gap for any chosen local solver of approximation quality $\Theta$.

\begin{lemma}
\label{lem:basic}
Let $g_i$ be strongly convex~\footnote{Note that the case of weakly convex $g_i(.)$ is explicitly allowed here as well, as the Lemma holds for the case $\mu = 0$.} with convexity parameter $\mu \geq 0$ with respect to the norm $\|\cdot\|$, $\forall i\in[n]$.
Then at each iteration of Algorithm~\ref{alg:generalizedcocoa} under Assumption~\ref{asm:theta}, and any $s\in [0,1]$, it holds that
\begin{align}
\label{eq:lemma:dualdecrease_vs_dualitygap}
&\E[ \OA(\vc{\alphav}{t}) - \OA(\vc{\alphav}{t+1})] \geq \gamma (1-\Theta) \Big(s G(\vc{\alphav}{t}) - \frac{\sigma's^2}{2\tau} \vc{R}{t} \Big) \, ,
\end{align}
where
\begin{align}
\label{eq:defOfR}
\vc{R}{t}&:= - \tfrac{ \tau \mu (1-s)}{\sigma' s } \|\vc{\uv}{t}-\vc{\alphav}{t}\|^2 + \textstyle{\sum}_{k=1}^K \| \vsubset{A}{k}\vsubset{  (\vc{\uv}{t}  - \vc{\alphav}{t} )}{k}\|^2 \, ,
\end{align}
for $\vc{\uv}{t} \in\R^{\n}$ with
\begin{equation}
\label{eq:defintionOfUi}
\vc{u_i}{t} 
\in \partial g^*_i(-\xv_i^\top\wv(\vc{\alphav}{t})) \, .
\end{equation}
\end{lemma}
\begin{proof}
This proof is motivated by \citet[Lemma 19]{ShalevShwartz:2013wl} and follows \citet[Lemma 5]{Ma:2015ti}, with a difference being the extension to our generalized subproblems $\Ggk(\cdot;\vv, \vsubset{\alphav}{k})$ along with the  mappings $\wv(\alphav) := \nabla f(\vv(\alphav))$ with $\vv(\alphav) := A \alphav$. 

For simplicity, we write $\alphav$ instead of $\vc{\alphav}{t}$, $\vv$ instead of $\vv(\vc{\alphav}{t})$, $\wv$ instead of $\wv(\vc{\alphav}{t})$ and $\uv$ instead of $\vc{\uv}{t}$. We can estimate the expected change of the objective $\OA(\alphav)$ as follows. Starting from the definition of the update $\vc{\alphav}{t+1} := \vc{\alphav}{t} + \gamma \, \sum_k \vsubset{\Delta \alphav}{k}$ from Algorithm~\ref{alg:generalizedcocoa}, we apply Lemma \ref{lem:RelationOfDTOSubproblems}, which relates the local approximation $\Ggk(\alphav;\vv, \vsubset{\alphav}{k})$ to the global objective $\OA(\alphav)$, and then bound this using the notion of quality of the local solver ($\Theta$), as in Assumption \ref{asm:theta}. This gives us: 
\begin{align*}
\E \big[\OA(\vc{\alphav}{t}) - \OA(\vc{\alphav}{t+1})\big] &= \E \Big[\OA(\alphav) - \OA\Big(\alphav + \gamma \sum_{k=1}^K \vsubset{\Delta \alphav}{k}\Big)\Big] \\
&\ge \gamma (1-\Theta) \left( \underbrace{ \OA(\alphav) - \sum_{k=1}^K \Ggk(\vsubset{\Delta \alphav^{\star}}{k}; \vv, \vsubset{\alphav}{k}) }_{C} \right) \, .
\tagthis
\label{eq:basic1}
\end{align*} 

We next upper bound the $C$ term, denoting $\Delta \alphav^{\star} = \sum_{k=1}^K \vsubset{\Delta \alphav^{\star}}{k}$. We first plug in the definition of the objective $\OA$ in \eqref{eq:primal} and the local subproblems \eqref{eq:subproblem}, and then substitute $s(u_i-\alpha_i)$ for $\Delta \alpha^{\star}_i$ and apply the $\mu$-strong convexity of the $g_i$ terms. This gives us:
\begin{align*}
C& = \sum_{i =1}^{\n} \left(g_i(\alpha_i) - g_i(\alpha_i + \Delta \alpha^{\star}_i) \right) - (A\Delta \alphav^{\star})^\top\wv(\alphav)  - \sum_{k=1}^K  \frac{\sigma'}{2\tau} \Big\|\vsubset{A}{k}\vsubset{\Delta \alphav^{\star}}{k}\Big\|^2 \\
& = \sum_{i =1}^{\n} \left(g_i(\alpha_i) - g_i(su_i + (1-s)\alpha_i) \right) - A(s(\uv-\alphav))^\top\wv(\alphav)  - \sum_{k=1}^K  \frac{\sigma'}{2\tau} \Big\|\vsubset{A}{k}\vsubset{s (\uv  - \alphav )}{k}\Big\|^2 \\
&\ge \sum_{i =1}^{\n} \left( s g_i(\alpha_i ) -s g_i(u_i ) + \frac{\mu}{2} (1-s)s (u_i -\alpha_i)^2 \right) \\
&\qquad - A(s (\uv  - \alphav ))^\top\wv(\alphav) - \sum_{k=1}^K \frac{\sigma'}{2\tau}   \Big\|\vsubset{A}{k}(\vsubset{s (\uv  - \alphav )}{k})\Big\|^2 \, .\tagthis
\end{align*}

From the definition of the optimization problems \eqref{eq:primal} and \eqref{eq:dual}, and definition of convex conjugates, we can write the duality gap as:
\begin{align*}
\gap(\alphav) := \OA(\alphav)-(-\OB(\wv(\alphav))
&\overset{\eqref{eq:primal},\eqref{eq:dual}}{=} \sum_{i=1}^{\n} \left( g^*_i(- \xv_i^\top\wv(\alphav)) + g_i(\alpha_i) \right) + f^*(\wv(\alphav)) + f(A\alphav)) \\
& = \sum_{i=1}^{\n} \left( g^*_i( -\xv_i^\top\wv(\alphav)) + g_i(\alpha_i) \right) +  (A\alphav)^\top\wv(\alphav)  \\
& = \sum_{i=1}^{\n} \left( g^*_i( -\xv_i^\top\wv(\alphav)) +  g_i(\alpha_i) + \alpha_i \xv_i^\top\wv(\alphav) \right) \, .
\tagthis
\label{eq:basic3}
\end{align*}

The convex conjugate maximal property from \eqref{eq:defintionOfUi} implies that
\begin{equation}
\label{eq:basic2}
g_i(u_i) = u_i (-\xv_i^\top\wv(\alphav)) -g^*_i(-\xv_i^\top \wv(\alphav)) \, .
\end{equation}
Using \eqref{eq:basic2} and \eqref{eq:basic3}, we therefore have:
\begin{align*}
C &\overset{ \eqref{eq:basic2}}{\geq} \sum_{i =1}^{\n} \left(s  g_i(\alpha_i ) - s u_i (-\xv_i^\top\wv(\alphav)) + sg^*_i(-\xv_i^\top\wv(\alphav)) +  
\frac{\mu}{2} (1-s)s (u_i -\alpha_i)^2 \right) \\
& \qquad  - A(s (\uv  - \alphav ))^\top\wv(\alphav) - \sum_{k=1}^K \frac{\sigma'}{2\tau}   \Big\|\vsubset{A}{k}(\vsubset{s (\uv  - \alphav )}{k})\Big\|^2 \\
&= \sum_{i =1}^{\n}  \big[  sg_i(\alpha_i ) + sg^*_i(-\xv_i^\top \wv(\alphav)) + s \xv_i^\top \wv(\alphav) \alpha_i \big]  - \sum_{i =1}^{\n} \big[  s \xv_i^\top \wv(\alphav) ( \alpha_i-u_i ) - \frac{\mu}{2}(1-s)s (u_i -\alpha_i)^2 \big] \\
&\qquad  - A(s (\uv  - \alphav ))^\top\wv(\alphav) - \sum_{k=1}^K \frac{\sigma'}{2\tau}   \Big\|\vsubset{A}{k}(\vsubset{s (\uv  - \alphav )}{k})\Big\|^2 \\
&\overset{\eqref{eq:basic3}}{=} s \gap(\alphav) + \frac{\mu}{2} (1-s)s  \|\uv-\alphav\|^2  - \frac{\sigma's^2}{2\tau} \sum_{k=1}^K   \| \vsubset{A}{k} \vsubset{  (\uv  - \alphav )}{k}\|^2 \, .
\tagthis 
\label{eq:basic4}
\end{align*}

The claimed improvement bound \eqref{eq:lemma:dualdecrease_vs_dualitygap} then follows by plugging \eqref{eq:basic4} into \eqref{eq:basic1}.
\end{proof}

The following Lemma provides a uniform bound on~$\vc{R}{t}$.

\begin{lemma}
\label{lemma:BoundOnR}
If $g^*_i$ are $L$-Lipschitz continuous for all $i\in [\n]$, then
\begin{equation}
\label{eq:asfjoewjofa}
\forall t:  \vc{R}{t} \leq 4L^2 \underbrace{\sum _{k=1}^K \sigma_k  n_k}_{=: \sigma}\, , \end{equation}
where
\begin{equation}
\label{eq:definitionOfSigmaK}
\sigma_k \eqdef \max_{\vsubset{\alphav}{k} \in \R^n} \frac{\|\vsubset{A}{k}\vsubset{\alphav}{k}\|^2}{\|\vsubset{\alphav}{k}\|^2} \, .
\end{equation}
\end{lemma}
\begin{proof}
\citep[Lemma 6]{Ma:2015ti}. 
For general convex functions, the strong convexity parameter is 
$\mu=0$, and hence the definition \eqref{eq:defOfR} of the complexity constant $\vc{R}{t}$ becomes
\begin{align*} 
\vc{R}{t}
=
  \sum _{k=1}^K   
  \| \vsubset{A}{k} \vsubset{  (\vc{\uv} {t} - \vc{\alphav}{t} )}{k}\|^2
\overset{\eqref{eq:definitionOfSigmaK}}{\leq}   
\sum _{k=1}^K 
\sigma_k  
  \|   \vsubset{  (\vc{\uv} {t} - \vc{\alphav}{t} )}{k}\|^2
\leq
\sum _{k=1}^K 
\sigma_k  |\Pk| 4L^2 \, .
\end{align*}
Here the last inequality follows from \cite[Lemma 21]{ShalevShwartz:2013wl}, which shows that for $g^*_i : \R \to \R$ being $L$-Lipschitz, it holds that for any real value $a$ with $|a|> L$ one has that
$g_i(a) = +\infty$.
\end{proof}

\begin{remark}
\label{rmk:asfwaefwae}
\citep[Remark 7]{Ma:2015ti} If the data points $\xv_i$ are normalized such that $\|\xv_i\|\leq 1$, $\forall i\in [n]$, then $\sigma_k \leq |\Pk| = n_k$. Furthermore, if we assume that the data partition is balanced, i.e., that $n_k = \n/K$ for all $k$, then $\sigma \le \n^2/K$. This can be used to bound the constants $\vc{R}{t}$, above, as $ \vc{R}{t} \leq  \frac{4L^2 \n^2}{K}.$
\end{remark}

\begin{theorem}
\label{thm:convergenceNonsmoothCoCoA}

Consider Algorithm \ref{alg:generalizedcocoa}, using a local solver of quality $\Theta$ (See Assumption \ref{asm:theta}).
Let $g^*_i(\cdot)$ be $L$-Lipschitz continuous,
and $\epsilon_G>0$ be the desired duality gap (and hence an upper-bound on suboptimality $\epsilon_{\OA}$).
Then after $T$ iterations, where
\begin{align}\label{eq:dualityRequirementsC}
T
&\geq
T_0 +
\max\{\Big\lceil \frac1{\gamma (1-\Theta)}\Big\rceil,\Big\lceil\frac
{4L^2  \sigma   \sigma'}
{ \tau \epsilon_G
\gamma (1-\Theta)}\Big\rceil \} \, ,
\\
T_0
\geq t_0+
\Big[
\frac{2}{ \gamma (1-\Theta) }
&\left(
\frac
{8L^2  \sigma   \sigma'}
{\tau \epsilon_G}
-1
\right)
\Big]_+\, ,\notag \, \, \, \, t_0  \geq
  \max(0,\Big\lceil \tfrac1{\gamma (1-\Theta)}
\log\left(
\tfrac{\tau(
 \OA(\vc{\alphav}{0})-\OA(\alphav^{\star} ))
  }{2 L^2 \sigma \sigma'}
  \right)
 \Big\rceil)\, ,\notag
\end{align}
we have that the expected duality gap satisfies
\[
\Exp[\OA(\overline \alphav)-(-\OB( \wv(\overline\alphav))) ] \leq \epsilon_G
\]
at the averaged iterate
\begin{equation}\label{eq:averageOfAlphaDefinition}
\overline \alphav: = \tfrac1{T-T_0}\textstyle{\sum}_{t=T_0}^{T-1} \vc{\alphav}{t} \, .
\end{equation}
\end{theorem}

\begin{proof} 
We begin by estimating the expected change of feasibility for $\OA$. We can bound this above by using Lemma \ref{lem:basic} and the fact that the $\OB(\cdot)$ is always a lower bound for $-\OA(\cdot)$, and then applying \eqref{eq:asfjoewjofa} to find:
\begin{align*} 
 \Exp[\OA(\vc{\alphav}{t+1})-\OA(\alphav^{\star})] 
 &= \Exp[-\OA(\alphav^{\star})+\OA(\vc{\alphav}{t+1})-\OA(\vc{\alphav}{t})+\OA(\vc{\alphav}{t})] \\
& 
\leq
\left( 
 1-\aggpar
(1-\Theta)
   s
\right) 
   (\OA(\vc{\alphav}{t})-\OA(\alphav^{\star}))
+
\aggpar
(1-\Theta) 
 \tfrac{\sigma' s^2}{2\tau}
4L^2  \sigma \, .
\tagthis 
\label{eq:asoifejwofa}
\end{align*}
Using
\eqref{eq:asoifejwofa}
recursively we have 
 \begin{align*} 
 \Exp[\OA(\vc{\alphav}{t})-\OA(\alphav^{\star})]
&\leq
\left( 
 1-\aggpar
(1-\Theta)
   s
\right)^t 
   (\OA(\vc{\alphav}{0})-\OA(\alphav^{\star} ))
+
 s
\frac{4L^2  \sigma   \sigma'}{2\tau} \, . 
\tagthis
\label{eq:asfwefcaw}  
 \end{align*}
Choosing 
$s=1$ and $t= t_0:= \max\{0,\lceil  
\frac1{\aggpar (1-\Theta)}
\log(
 2 (\OA(\vc{\alphav}{0})-\OA(\alphav^{\star} ))
  / (4 L^2 \sigma \sigma')
  )
 \rceil\}$
leads to 
\begin{align}\label{eq:induction_step1}
  \Exp[\OA(\vc{\alphav}{t})-\OA(\alphav^{\star})]
 &\leq  
\left( 
 1-\aggpar
(1-\Theta)  
\right)^{t_0}
  (\OA(\vc{\alphav}{0})-\OA(\alphav^{\star} ))
+ 
\frac{4L^2  \sigma   \sigma'}{2\tau}
\le \frac{4L^2  \sigma   \sigma'}{\tau} \, .
\end{align} 
Next, we show inductively that 
\begin{align}
\label{eq:expectationOfDualFeasibility}
\forall t\geq t_0 :  \Exp[\OA(\vc{\alphav}{t})-\OA(\alphav^{\star} )]
&\leq 
\frac{4L^2  \sigma   \sigma'}{\tau( 1+ \frac12  \aggpar (1-\Theta)  (t-t_0))} \, .
\end{align}
Clearly, \eqref{eq:induction_step1} implies that \eqref{eq:expectationOfDualFeasibility} holds for $t=t_0$.
Assuming that it holds for any $t\geq t_0$, we show that it must also hold for $t+1$. 
Indeed, using 
\begin{equation}
\label{eq:asdfjoawjdfas}
s=
\frac{1}
 {1+ \frac12 \aggpar (1-\Theta) (t-t_0)} \in [0,1] \, ,
\end{equation} 
  we obtain
\begin{align*}
\Exp[
\OA(\vc{\alphav}{t+1})-\OA(\alphav^{\star} )]
\le \frac{4L^2  \sigma   \sigma'}{\tau}
\underbrace{\left( 
\frac{
1+ \frac12 \aggpar (1-\Theta) (t-t_0)
-\frac12 \aggpar
(1-\Theta)
}
 {(1+ \frac12 \aggpar (1-\Theta) (t-t_0))^2}
\right)}_{D} \, ,
\end{align*}
by applying the bounds \eqref{eq:asoifejwofa} and \eqref{eq:expectationOfDualFeasibility}, plugging in the definition of $s$ \eqref{eq:asdfjoawjdfas}, and simplifying. We upper bound the term $D$ using the fact that geometric mean
 is less or equal to arithmetic mean:
\begin{align*}
D&=
\frac1
{1+ \frac12 \aggpar (1-\Theta) (t+1-t_0)}
\underbrace{ 
\frac{
(1+ \frac12 \aggpar (1-\Theta) (t+1-t_0))
(1+ \frac12 \aggpar (1-\Theta) (t-1-t_0))
}
 {(1+ \frac12 \aggpar (1-\Theta) (t-t_0))^2}}_{\leq 1}
 \\
&\leq  
\frac1
{1+ \frac12 \aggpar (1-\Theta) (t+1-t_0)}.
\end{align*}
 If $\overline \alphav$ is defined as \eqref{eq:averageOfAlphaDefinition}, we apply the results of Lemma~\ref{lem:basic} and Lemma~\ref{lemma:BoundOnR} to obtain
\begin{align*}
\Exp[\gap(\overline\alphav)] &=  
 \Exp\left[\gap\left(\sum_{t=T_0}^{T-1} \tfrac1{T-T_0} \vc{\alphav}{t}\right)\right]
 \leq
  \tfrac1{T-T_0} \Exp\left[\sum_{t=T_0}^{T-1} \gap\left( \vc{\alphav}{t}\right)\right]
\\
 &\leq
\frac1{\aggpar
(1-\Theta)
 s}
   \frac1{T-T_0} 
   \Exp\left[
\OA(\vc{\alphav}{T_0})
-
\OA(\alphav^{\star})
  \right] 
+\tfrac{4L^2 \sigma \sigma' s}{2\tau} \, .  
\tagthis \label{eq:askjfdsanlfas}
  \end{align*}
If $T\geq \lceil
\frac1{\aggpar (1-\Theta)}\rceil+T_0$ such that $T_0\geq t_0$
we have
\begin{align*}
\Exp[\gap(\overline\alphav)] 
&\overset{\eqref{eq:askjfdsanlfas}
,\eqref{eq:expectationOfDualFeasibility}
}{\leq}
\frac1{\aggpar
(1-\Theta)
 s}
   \frac1{T-T_0} 
\left(
\frac{4L^2  \sigma   \sigma'}{\tau( 1+ \frac12  \aggpar (1-\Theta)  (T_0-t_0))}
\right)
+\frac{4L^2 \sigma \sigma' s}{2\tau}
\\
&=
\frac{
4L^2  \sigma   \sigma'}{\tau}
\left(
\frac1{\aggpar
(1-\Theta)
 s}
   \frac1{T-T_0} 
\frac{1}{ 1+ \frac12  \aggpar (1-\Theta)  (T_0-t_0)}
+\frac{  s}{2 }
\right) \, . 
\tagthis
\label{eq:fawefwafewa}
\end{align*}
Choosing 
\begin{equation}
\label{eq:afskoijewofaw}
s=\frac{1}{(T-T_0) \aggpar (1-\Theta)} \in [0,1]
\end{equation}
gives us
\begin{align*}
\Exp[\gap(\overline\alphav)] 
&
\overset{\eqref{eq:fawefwafewa},
\eqref{eq:afskoijewofaw}}{\leq}
\frac{4L^2  \sigma   \sigma'}{\tau}
\left(
\frac{1}{ 1+ \frac12  \aggpar (1-\Theta)  (T_0-t_0)}
+\frac{1}{(T-T_0) \aggpar (1-\Theta)} \frac{  1}{2 }
\right) \, . \tagthis
\label{eq:afsjweofjwafea}
\end{align*}
To have right hand side of
\eqref{eq:afsjweofjwafea}
smaller then 
$\epsilon_\gap$
it is sufficient to choose
$T_0$ and $T$ such that
\begin{eqnarray}
\label{eq:sfadwafeewafa}
\frac{4L^2  \sigma   \sigma'}{\tau}
\left(
\frac{1}{ 1+ \frac12  \aggpar (1-\Theta)  (T_0-t_0)}
\right)
&\leq & \frac12 \epsilon_\gap \, ,
\\
\label{eq:sfadwafeewafa2}
\frac{4L^2  \sigma   \sigma'}{\tau}
\left(
\frac{1}{(T-T_0) \aggpar (1-\Theta)} \frac{  1}{2 }
\right)
&\leq & \frac12 \epsilon_\gap \, .
\end{eqnarray}
Hence if
$T_0
\geq 
t_0+
\frac{2}{ \aggpar (1-\Theta) }
\left(
\frac
{8L^2  \sigma   \sigma'}
{\tau\epsilon_\gap}
-1
\right)$
 and
$
T\geq 
T_0
+
\frac
{4L^2  \sigma   \sigma'}
{\tau\epsilon_\gap
\aggpar (1-\Theta)}$
then 
\eqref{eq:sfadwafeewafa}
and
\eqref{eq:sfadwafeewafa2}
are satisfied.
\end{proof}

The following main theorem simplifies the results of Theorem~\ref{thm:convergenceNonsmoothCoCoA} and is a generalization of \citet[Corollary 9]{Ma:2015ti} for general $f^*(\cdot)$ functions:

\begin{reptheorem}
{thm:convergenceNonsmooth}
Consider Algorithm \ref{alg:generalizedcocoa} with $\gamma :=1$, using a local solver of quality $\Theta$ (see Assumption \ref{asm:theta}). Let $g^*_i(\cdot)$ be $L$-Lipschitz continuous, and assume that the columns of $A$ satisfy $\|\xv_i\|\leq 1$, $\forall i\in [\n]$ and $g_i^*$ is of the form $\frac{1}{n}g_i^*$, as is common in ERM-type problems.
Let $\epsilon_{G}>0$ be the desired duality gap (and hence an upper-bound on primal sub-optimality).
Then after $T$ iterations, where
\begin{align}
T
&\geq
T_0 +
\max\{\Big\lceil \frac1{1-\Theta}\Big\rceil,
\frac{4L^2}{\tau\epsilon_{G}(1-\Theta)}\} \,,
\\
T_0
&\geq t_0+
\Big[
\frac{2}{ 1-\Theta }
\left(\frac {8L^2} {\tau\epsilon_{ G}}
-1
\right)
\Big]_+ \, ,\notag
\\
t_0 &\geq
  \max(0,\Big\lceil \tfrac1{(1-\Theta)}
\log\left(
\tfrac{
 \tau n({\OA}(\vc{\alphav}{0})-{\OA}(\alphav^{\star} ))
  }{2 L^2 K}
 \right)
 \Big\rceil)\,,\notag
\end{align}
we have that the expected duality gap satisfies
\[
\Exp[\OA(\overline \alphav)-(-\OB( \wv(\overline\alphav))) ] \leq \epsilon_{ G} \, ,
\]
where $\overline\alphav$ is the averaged iterate returned by Algorithm \ref{alg:generalizedcocoa}.
\end{reptheorem}

\begin{proof} Plug in parameters $\gamma := 1$, $\sigma' := \gamma K = K, \tilde{L} := \frac{1}{n}L$ to the results of Theorem \ref{thm:convergenceNonsmoothCoCoA}, and note that for balanced datasets with $g_i^* := \frac{1}{n} g_i^*$ we have $\sigma \le \frac{\n}{K}$ (see Remark \ref{rmk:asfwaefwae}). We can further simplify the rate by noting that $\tau = 1$ for the 1-smooth losses (least squares and logistic) given as examples in this work.
\end{proof}

\subsection{Proof of Convergence Result for Strongly Convex $g_i$} 
Our second main theorem follows reasoning in \cite{ShalevShwartz:2013wl} and is a generalization of \citet[Corollary 11]{Ma:2015ti}. We first introduce a lemma to simplify the proof.

\begin{lemma}
\label{lemma:asfewfawfcda}
Assume that $g_i(0) \in [0,1]$ 
for all $i\in[\n]$, then for the zero vector $\vc{\alphav}{0}
 := {\bf 0}\in \R^n$, we have
\begin{equation}
\label{eq:afjfjaoefvcwa}
\OA(\vc{\alphav}{0})-\OA(\alphav^{\star})
= 
\OA({\bf 0})-\OA(\alphav^{\star})
 \leq n \, .
 \end{equation}
\end{lemma}
\begin{proof}
Since $-\OA(\cdot)$ is always a lower bound on $\OB(\cdot)$, and by definition of the objectives $\OA$ and $\OB$ given in~\eqref{eq:primal} and~\eqref{eq:dual} respectively, for $\alphav := {\bf 0}\in \R^n$,
\begin{align*}
0 &\leq \OA(\alphav)-\OA(\alphav^{\star})
\leq  \OA(\mathbf{0})-(-\OB (\wv(\mathbf{0})))
 \overset{\eqref{eq:primal},\eqref{eq:dual}
}{=} \\ &\overset{\eqref{eq:primal},\eqref{eq:dual}
}{=}  \ f (\mathbf{0}) + \ f^*(\wv(\mathbf{0}) )
\ +\ g(\mathbf{0}) 
+\ g^*(-A^\top\wv(\mathbf{0})) \,. 
\end{align*} 
Since $ \ f^*(\wv(\mathbf{0}) ) =  \ f^*( \nabla f( \mathbf{0} ) ) = \mathbf{0}^\top\nabla f( \mathbf{0} ) -   f( \mathbf{0} ) = -   f( \mathbf{0} )$, and given our initial assumption on $\ g(\mathbf{0}) $, the duality gap reduces to:
\begin{align*}
0 & \leq \ g(\mathbf{0}) + \ g^*(-A^\top\wv(\mathbf{0})) \leq n. \qedhere 
\end{align*}
\end{proof}

\begin{theorem}
\label{thm:convergenceSmoothCase}
Assume that $g_i$ are $\mu$-strongly convex $\forall i\in[\n]$.
We define $\sigma_{\max} = 
\max_{k\in[K]} \sigma_k$. Then after $T$ iterations of Algorithm \ref{alg:generalizedcocoa}, with
\[
 T
    \geq 
\tfrac{1}
   {\aggpar
(1-\Theta)}
\tfrac
{\mu\tau+
\sigma_{\max} \sigma'}
{\mu\tau }
    \log \tfrac \n {\epsilon_{\OA}} \, , 
\]
it holds that
\[\Exp[\OA(\vc{\alphav}{T})-\OA(\alphav^{\star})]
   \leq \epsilon_{\OA} \, .\]
Furthermore, after $T$ iterations with
\[
 T 
    \geq 
\tfrac{1}
   {\aggpar
(1-\Theta)}
\tfrac
{\mu\tau+
\sigma_{\max} \sigma'}
{ \mu \tau}
    \log 
\left(
\tfrac{1}
   {\aggpar
(1-\Theta)}
\tfrac
{\mu\tau+
\sigma_{\max} \sigma'}
{ \mu \tau}
    \tfrac \n {\epsilon_\gap}
    \right)\, ,
\]
we have the expected duality gap
\[
\Exp[
 \OA(\vc{\alphav}{T})-(-\OB( \wv(\vc{\alphav}{T})))
]\leq \epsilon_\gap \, .
\]
\end{theorem}

\begin{proof}
Given that $g_i(.)$ is $\mu$-strongly convex, we can apply \eqref{eq:defOfR} and the definition of $\sigma_k$ to find:
\begin{align*}
\vc{R}{t}&
\leq
-
\tfrac{ \tau \mu  (1-s)}{\sigma' s }
   \|\vc{\uv}{t}-\vc{\alphav}{t}\|^2 
+ 
 {\sum}_{k=1}^K   
 \sigma_k
  \|  \vsubset{   \vc{\uv}{t}  - \vc{\alphav}{t}  }{k}\|^2
\\
&\leq
\left(
-
\tfrac{ \tau \mu (1-s)}{\sigma' s }
+\sigma_{\max}
\right)
   \|\vc{\uv}{t}-\vc{\alphav}{t}\|^2 \, ,\tagthis
   \label{eq:afjfocjwfcea} 
\end{align*}
where $\sigma_{\max} = \max_{k\in[K]} \sigma_k$. If we plug the following value of $s$
 \begin{equation}
 s=
  \frac{ \tau \mu }
      {\tau \mu +
\sigma_{\max} \sigma'}\in [0,1]
\label{eq:fajoejfojew}
\end{equation} 
into
\eqref{eq:afjfocjwfcea}
we obtain that
$\forall t: \vc{R}{t}\leq 0$.
Putting the  same $s$
into
\eqref{eq:lemma:dualdecrease_vs_dualitygap}
will give us
\begin{align*}
\Exp[
\OA(\vc{\alphav}{t})
-
\OA(\vc{\alphav}{t+1})
 ]
&\overset{\eqref{eq:lemma:dualdecrease_vs_dualitygap}
,\eqref{eq:fajoejfojew}}{\geq}
\aggpar
(1-\Theta)
 \frac{ \tau \mu }
      {\tau \mu +
\sigma_{\max} \sigma'} \gap(\vc{\alphav}{t}) \\
&\geq
\aggpar
(1-\Theta)
 \frac{\tau \mu  }
      {\tau \mu +
\sigma_{\max} \sigma'} (\OA(\vc{\alphav}{t})-\OA(\alphav^{\star})) \, .
\tagthis
\label{eq:fasfawfwaf}
\end{align*}
Using the fact that
$\Exp[\OA(\vc{\alphav}{t})-\OA(\vc{\alphav}{t+1})]
=\Exp[\OA(\alphav^{\star})-\OA(\vc{\alphav}{t+1})]
+\OA(\vc{\alphav}{t})-\OA(\alphav^{\star})
$
we have 
\begin{align*}
\Exp[\OA(\alphav^{\star})-\OA(\vc{\alphav}{t+1})]
+\OA(\vc{\alphav}{t})-\OA(\alphav^{\star})
\overset{
\eqref{eq:fasfawfwaf}}
{
\geq
}
\aggpar
(1-\Theta)
 \frac{ \tau \mu  }
      {\tau \mu+
\sigma_{\max} \sigma'}(\OA(\vc{\alphav}{t})- \OA(\alphav^{\star})) \, ,
\end{align*}
which is equivalent to
\begin{align*}
\Exp[\OA(\vc{\alphav}{t+1})-\OA(\alphav^{\star})]
\leq 
\left(
1-\aggpar
(1-\Theta)
 \frac{\tau \mu  }
      {\tau \mu +
\sigma_{\max} \sigma'}\right)
(\OA(\vc{\alphav}{t})-\OA(\alphav^{\star})) \, .
\tagthis \label{eq:affpja}
\end{align*}
Therefore if we denote $\vc{\epsilon_{\OA}}{t} = \OA(\vc{\alphav}{t})-\OA(\alphav^{\star})$
we have recursively that
\begin{align*}
 \Exp[\vc{\epsilon_{\OA}}{t}] 
 & \overset{\eqref{eq:affpja}}{\leq}   \left(
 1-\aggpar
(1-\Theta)
 \frac{ \tau \mu }
      {\tau \mu +
\sigma_{\max} \sigma'}
   \right)^t \vc{\epsilon_{\OA}}{0}
\overset{\eqref{eq:afjfjaoefvcwa}}{\leq}
\left(
 1-\aggpar
(1-\Theta)
 \frac{ \tau \mu }
      {\tau \mu +
\sigma_{\max} \sigma'}
   \right)^t \n \\
& \leq \exp\left(-t \aggpar
(1-\Theta)
 \frac{ \tau \mu }
      {\tau \mu +
\sigma_{\max} \sigma'}
     \right)\n \, .
\end{align*}
The right hand side will be smaller than some $\epsilon_{\OA}$ if 
\[
 t   
    \geq 
\frac{1}
   {\aggpar
(1-\Theta)}
\frac
{\tau \mu +
\sigma_{\max} \sigma'}
{ \tau \mu  }
    \log \frac \n {\epsilon_{\OA}} \, .
\]
Moreover, to bound the duality gap, we have
\begin{align*}
\aggpar
(1-\Theta)
 \frac{ \tau \mu }
      {\tau \mu +
\sigma_{\max} \sigma'} \gap(\vc{\alphav}{t})
&
\overset{
\eqref{eq:fasfawfwaf}
}{\leq}
\Exp[
\OA(\vc{\alphav}{t})
-
\OA(\vc{\alphav}{t+1})
 ]
\leq 
\Exp[
\OA(\vc{\alphav}{t})-\OA(\alphav^{\star})
 ] \, .  
\end{align*}
Thus,  $\gap(\vc{\alphav}{t})\leq 
\frac1{
\aggpar
(1-\Theta)}
 \frac      {\tau \mu +
\sigma_{\max} \sigma'} 
{ \tau\mu  }    \vc{\epsilon_{\OA}}{t}$.  
Hence if $\epsilon_{\OA} \leq 
\aggpar
(1-\Theta)
 \frac{\tau \mu }
      {\tau\mu +
\sigma_{\max} \sigma'} 
 \epsilon_\gap $
then $\gap(\vc{\alphav}{t})\leq \epsilon_\gap$.
Therefore
after 
\[
 t   
    \geq 
\frac{1}
   {\aggpar
(1-\Theta)}
\frac
{\tau \mu+
\sigma_{\max} \sigma'}
{ \tau \mu  }
    \log 
\left(
\frac{1}
   {\aggpar
(1-\Theta)}
\frac
{\tau \mu+
\sigma_{\max} \sigma'}
{ \tau \mu  }
    \frac \n {\epsilon_\gap}
    \right) 
\]
iterations we have obtained a duality gap less than $\epsilon_\gap$.
\end{proof}

\begin{reptheorem}
{thm:convergenceSmooth}
Consider Algorithm~\ref{alg:generalizedcocoa} with $\gamma := 1$, using a local solver of quality $\Theta$ (see Assumption \ref{asm:theta}). Let $g_i(\cdot)$ be  $\mu$-strongly convex, $\forall i\in[\n]$, and assume that the columns of $A$ satisfy $\|\xv_i\|\leq 1$ $\forall i\in [\n]$ and and $g_i^*$ is of the form $\frac{1}{n}g_i^*$, as is common in ERM-type problems. Then we have that $T$ iterations are sufficient for suboptimality 
$\epsilon_{\OA}$, with
\[
T \geq
\tfrac{1}
   {
(1-\Theta)}
\tfrac
{\tau \mu+1}
{\tau \mu }
    \log \tfrac 1 {\epsilon_{\OA}} \, . 
\]
Furthermore, after $T$ iterations with
\[
 T 
    \geq 
\tfrac{1}
   {
(1-\Theta)}
\tfrac
{\tau \mu+
1}
{ \tau \mu }
    \log 
\left(
\tfrac{1}
   {
(1-\Theta)}
\tfrac
{\tau \mu+
1}
{ \tau \mu }
    \tfrac 1 {\epsilon_\gap}
    \right)\, ,
\]
we have the expected duality gap
\[
\Exp[
 \OA(\vc{\alphav}{T})-(-\OB( \wv(\vc{\alphav}{T})))
]\leq \epsilon_\gap \, .
\]
\end{reptheorem}

\begin{proof}
Plug in parameters $\gamma := 1$, $\sigma' := \gamma K = K, \tilde{\mu} = n \mu$ to the results of Theorem \ref{thm:convergenceSmoothCase} and note that for balanced datasets with $g_i^* := \frac{1}{n} g_i^*$ we have $\sigma_{\max} \le \frac{\n}{K}$ (see Remark \ref{rmk:asfwaefwae}). We can further simplify the rate by noting that $\tau = 1$ for the 1-smooth losses (least squares and logistic) given as examples in this work.
\end{proof}

\newpage
\bibliography{bibliography}

\begin{thebibliography}{76}
\providecommand{\natexlab}[1]{#1}
\providecommand{\url}[1]{\texttt{#1}}
\expandafter\ifx\csname urlstyle\endcsname\relax
  \providecommand{\doi}[1]{doi: #1}\else
  \providecommand{\doi}{doi: \begingroup \urlstyle{rm}\Url}\fi

\bibitem[Andrew and Gao(2007)]{Andrew:2007cu}
G.~Andrew and J.~Gao.
\newblock {Scalable training of L1-regularized log-linear models}.
\newblock In \emph{International Conference on Machine Learning}, 2007.

\bibitem[Arjevani and Shamir(2015)]{Arjevani:2015vka}
Y.~Arjevani and O.~Shamir.
\newblock Communication complexity of distributed convex learning and
  optimization.
\newblock In \emph{Neural Information Processing Systems}, 2015.

\bibitem[Balcan et~al.(2012)Balcan, Blum, Fine, and Mansour]{Balcan:2012tc}
M.-F. Balcan, A.~Blum, S.~Fine, and Y.~Mansour.
\newblock Distributed learning, communication complexity and privacy.
\newblock In \emph{Conference on Learning Theory}, 2012.

\bibitem[Bauschke and Combettes(2011)]{Bauschke:2011ik}
H.~H. Bauschke and P.~L. Combettes.
\newblock \emph{{Convex Analysis and Monotone Operator Theory in Hilbert
  Spaces}}.
\newblock Springer Science \& Business Media, New York, NY, 2011.

\bibitem[Bersekas and Tsitsiklis(1989)]{Bertsekas:1989}
D.~P. Bersekas and J.~N. Tsitsiklis.
\newblock \emph{Parallel and Distributed Computation: Numerical Methods}.
\newblock Prentice Hall, Englewood Cliffs, NJ, 1989.

\bibitem[Bian et~al.(2013)Bian, Li, Liu, and Yang]{Bian:2013wx}
Y.~Bian, X.~Li, Y.~Liu, and M.-H. Yang.
\newblock Parallel coordinate descent {N}ewton method for efficient
  $\ell$$_{1}$-regularized minimization.
\newblock \emph{arXiv.org}, 2013.

\bibitem[Borwein and Zhu(2005)]{Borwein:2005ge}
J.~M. Borwein and Q.~Zhu.
\newblock \emph{{Techniques of Variational Analysis}}.
\newblock Springer Science \& Buisness Media, New York, NY, 2005.

\bibitem[Boyd and Vandenberghe(2004)]{Boyd:2004uz}
S.~Boyd and L.~Vandenberghe.
\newblock \emph{{Convex Optimization}}.
\newblock Cambridge University Press, Cambridge, UK, 2004.

\bibitem[Boyd et~al.(2010)Boyd, Parikh, Chu, Peleato, and
  Eckstein]{Boyd:2010bw}
S.~Boyd, N.~Parikh, E.~Chu, B.~Peleato, and J.~Eckstein.
\newblock Distributed optimization and statistical learning via the alternating
  direction method of multipliers.
\newblock \emph{Foundations and Trends in Machine Learning}, 3\penalty0
  (1):\penalty0 1--122, 2010.

\bibitem[Bradley et~al.(2011)Bradley, Kyrola, Bickson, and
  Guestrin]{Bradley:2011wq}
J.~K. Bradley, A.~Kyrola, D.~Bickson, and C.~Guestrin.
\newblock {Parallel coordinate descent for l1-regularized loss minimization}.
\newblock In \emph{International Conference on Machine Learning}, 2011.

\bibitem[D{\'e}fossez and Bach(2017)]{defossez2017adabatch}
A.~D{\'e}fossez and F.~Bach.
\newblock Adabatch: Efficient gradient aggregation rules for sequential and
  parallel stochastic gradient methods.
\newblock \emph{arXiv.org}, 2017.

\bibitem[Dekel et~al.(2012)Dekel, Gilad-Bachrach, Shamir, and
  Xiao]{Dekel:2012wm}
O.~Dekel, R.~Gilad-Bachrach, O.~Shamir, and L.~Xiao.
\newblock {Optimal Distributed Online Prediction Using Mini-Batches}.
\newblock \emph{Journal of Machine Learning Research}, 13:\penalty0 165--202,
  2012.

\bibitem[Duchi et~al.(2013)Duchi, Jordan, and McMahan]{Duchi:2013te}
J.~Duchi, M.~I. Jordan, and B.~McMahan.
\newblock Estimation, optimization, and parallelism when data is sparse.
\newblock \emph{Neural Information Processing Systems}, 2013.

\bibitem[D{\"u}nner et~al.(2016)D{\"u}nner, Forte, Tak{\'a}{\v c}, and
  Jaggi]{Dunner:2016vga}
C.~D{\"u}nner, S.~Forte, M.~Tak{\'a}{\v c}, and M.~Jaggi.
\newblock Primal-dual rates and certificates.
\newblock In \emph{International Conference on Machine Learning}, 2016.

\bibitem[D{\"u}nner et~al.(2017)D{\"u}nner, Parnell, and
  Jaggi]{dunner2017efficient}
C.~D{\"u}nner, T.~Parnell, and M.~Jaggi.
\newblock Efficient use of limited-memory accelerators for linear learning on
  heterogeneous systems.
\newblock In \emph{Neural Information Processing Systems}, 2017.

\bibitem[D{\"u}nner et~al.(2018)D{\"u}nner, Lucchi, Gargiani, Bian, Hofmann,
  and Jaggi]{dunner2018hessian}
C.~D{\"u}nner, A.~Lucchi, M.~Gargiani, A.~Bian, T.~Hofmann, and M.~Jaggi.
\newblock {A Distributed Second-Order Algorithm You Can Trust}.
\newblock In \emph{International Conference on Machine Learning}, 2018.

\bibitem[Fan et~al.(2008)Fan, Chang, Hsieh, Wang, and Lin]{Fan:2008tf}
R.-E. Fan, K.-W. Chang, C.-J. Hsieh, X.-R. Wang, and C.-J. Lin.
\newblock {LIBLINEAR}: A library for large linear classification.
\newblock \emph{Journal of Machine Learning Research}, 9:\penalty0 1871--1874,
  2008.

\bibitem[Fercoq and Richt{\'a}rik(2015)]{Fercoq:2015kd}
O.~Fercoq and P.~Richt{\'a}rik.
\newblock Accelerated, parallel, and proximal coordinate descent.
\newblock \emph{SIAM Journal on Optimization}, 25\penalty0 (4):\penalty0
  1997--2023, 2015.

\bibitem[Forte(2015)]{Forte:2015wv}
S.~Forte.
\newblock {Distributed Optimization for Non-Strongly Convex Regularizers}.
\newblock Master's thesis, ETH Z{\"u}rich, 2015.

\bibitem[Friedman et~al.(2010)Friedman, Hastie, and
  Tibshirani]{Friedman:2010wm}
J.~Friedman, T.~Hastie, and R.~Tibshirani.
\newblock {Regularization paths for generalized linear models via coordinate
  descent}.
\newblock \emph{Journal of Statistical Software}, 33\penalty0 (1):\penalty0
  1--22, 2010.

\bibitem[Gargiani(2017)]{gargiani2017hessian}
M.~Gargiani.
\newblock {Hessian-CoCoA}: a general parallel and distributed framework for
  non-strongly convex regularizers.
\newblock Master's thesis, ETH Zurich, 2017.

\bibitem[Heinze et~al.(2016)Heinze, McWilliams, and Meinshausen]{Heinze:2016tu}
C.~Heinze, B.~McWilliams, and N.~Meinshausen.
\newblock {DUAL-LOCO}: Distributing statistical estimation using random
  projections.
\newblock In \emph{International Conference on Artificial Intelligence and
  Statistics}, 2016.

\bibitem[Hiriart-Urruty and Lemar{\'e}chal(2001)]{hiriart-urruty:2001df}
J.-B. Hiriart-Urruty and C.~Lemar{\'e}chal.
\newblock \emph{Fundamentals of convex analysis}.
\newblock Springer--Verlag, Berlin, 2001.

\bibitem[Jaggi et~al.(2014)Jaggi, Smith, Tak{\'a}{\v c}, Terhorst, Krishnan,
  Hofmann, and Jordan]{Jaggi:2014vi}
M.~Jaggi, V.~Smith, M.~Tak{\'a}{\v c}, J.~Terhorst, S.~Krishnan, T.~Hofmann,
  and M.~I. Jordan.
\newblock {Communication-efficient distributed dual coordinate ascent}.
\newblock In \emph{Neural Information Processing Systems}, 2014.

\bibitem[Johnson and Guestrin(2015)]{Johnson:2015tq}
T.~Johnson and C.~Guestrin.
\newblock Blitz: A principled meta-algorithm for scaling sparse optimization.
\newblock In \emph{International Conference on Machine Learning}, 2015.

\bibitem[Karimi et~al.(2016)Karimi, Nutini, and Schmidt]{karimi2016linear}
H.~Karimi, J.~Nutini, and M.~Schmidt.
\newblock Linear convergence of gradient and proximal-gradient methods under
  the {P}olyak-{\l}ojasiewicz condition.
\newblock In \emph{European Conference on Machine Learning}, 2016.

\bibitem[Karimireddy et~al.(2018{\natexlab{a}})Karimireddy, Stich, and
  Jaggi]{praneeth2018balancing}
S.~P. Karimireddy, S.~U. Stich, and M.~Jaggi.
\newblock {Adaptive balancing of gradient and update computation times using
  global geometry and approximate subproblems}.
\newblock In \emph{International Conference on Artificial Intelligence and
  Statistics}, 2018{\natexlab{a}}.

\bibitem[Karimireddy et~al.(2018{\natexlab{b}})Karimireddy, Stich, and
  Jaggi]{praneeth2018newton}
S.~P. Karimireddy, S.~U. Stich, and M.~Jaggi.
\newblock {Global linear convergence of Newton's method without
  strong-convexity or Lipschitz gradients}.
\newblock \emph{arXiv.org}, 2018{\natexlab{b}}.

\bibitem[Lee and Chang(2017)]{lee2017distributed}
C.-p. Lee and K.-W. Chang.
\newblock Distributed block-diagonal approximation methods for regularized
  empirical risk minimization.
\newblock \emph{arXiv.org}, 2017.

\bibitem[Lee and Roth(2015)]{Lee:2015vra}
C.-P. Lee and D.~Roth.
\newblock Distributed box-constrained quadratic optimization for dual linear
  {SVM}.
\newblock In \emph{International Conference on Machine Learning}, 2015.

\bibitem[Lee et~al.(2018)Lee, Lim, and Wright]{lee2018distributed}
C.-p. Lee, C.~H. Lim, and S.~J. Wright.
\newblock A distributed quasi-newton algorithm for empirical risk minimization
  with nonsmooth regularization.
\newblock In \emph{ACM International Conference on Knowledge Discovery and Data
  Mining}, 2018.

\bibitem[Lu and Xiao(2013)]{Lu:2013tl}
Z.~Lu and L.~Xiao.
\newblock On the complexity analysis of randomized block-coordinate descent
  methods.
\newblock \emph{arXiv.org}, 2013.

\bibitem[Ma et~al.(2015{\natexlab{a}})Ma, Smith, Jaggi, Jordan, Richt{\'a}rik,
  and Tak{\'a}{\v c}]{Ma:2015ti}
C.~Ma, V.~Smith, M.~Jaggi, M.~I. Jordan, P.~Richt{\'a}rik, and M.~Tak{\'a}{\v
  c}.
\newblock Adding vs. averaging in distributed primal-dual optimization.
\newblock In \emph{International Conference on Machine Learning},
  2015{\natexlab{a}}.

\bibitem[Ma et~al.(2015{\natexlab{b}})Ma, Tappenden, and
  Tak{\'a}{\v{c}}]{ma2015linear}
C.~Ma, R.~Tappenden, and M.~Tak{\'a}{\v{c}}.
\newblock Linear convergence of the randomized feasible descent method under
  the weak strong convexity assumption.
\newblock \emph{arXiv.org}, 2015{\natexlab{b}}.

\bibitem[Ma et~al.(2017{\natexlab{a}})Ma, Jaggi, Curtis, Srebro, and
  Tak{\'a}{\v{c}}]{ma2017accelerated}
C.~Ma, M.~Jaggi, F.~E. Curtis, N.~Srebro, and M.~Tak{\'a}{\v{c}}.
\newblock An accelerated communication-efficient primal-dual optimization
  framework for structured machine learning.
\newblock \emph{arXiv.org}, 2017{\natexlab{a}}.

\bibitem[Ma et~al.(2017{\natexlab{b}})Ma, Konecny, Jaggi, Smith, Jordan,
  Richt{\'a}rik, and Tak{\'a}{\v c}]{Ma:2017dx}
C.~Ma, J.~Konecny, M.~Jaggi, V.~Smith, M.~I. Jordan, P.~Richt{\'a}rik, and
  M.~Tak{\'a}{\v c}.
\newblock {Distributed optimization with arbitrary local solvers}.
\newblock \emph{Optimization Methods and Software}, Feb. 2017{\natexlab{b}}.

\bibitem[Mahajan et~al.(2017)Mahajan, Keerthi, and
  Sundararajan]{Mahajan:2014tg}
D.~Mahajan, S.~S. Keerthi, and S.~Sundararajan.
\newblock A distributed block coordinate descent method for training l 1
  regularized linear classifiers.
\newblock \emph{Journal of Machine Learning Research}, 18\penalty0
  (91):\penalty0 1--35, 2017.

\bibitem[Mann et~al.(2009)Mann, McDonald, Mohri, Silberman, and
  Walker]{Mann:2009tr}
G.~Mann, R.~McDonald, M.~Mohri, N.~Silberman, and D.~D. Walker.
\newblock Efficient large-scale distributed training of conditional maximum
  entropy models.
\newblock \emph{Neural Information Processing Systems}, 2009.

\bibitem[McWilliams et~al.(2014)McWilliams, Heinze, Meinshausen, Krummenacher,
  and Vanchinathan]{McWilliams:2014tl}
B.~McWilliams, C.~Heinze, N.~Meinshausen, G.~Krummenacher, and H.~P.
  Vanchinathan.
\newblock {LOCO}: Distributing ridge regression with random projections.
\newblock \emph{arXiv.org}, 2014.

\bibitem[Meng et~al.(2016)Meng, Bradley, Yavuz, Sparks, Venkataraman, Liu,
  Freeman, Tsai, Amde, Owen, Xin, Xin, Franklin, Zadeh, Zaharia, and
  Talwalkar]{Meng:2015tu}
X.~Meng, J.~Bradley, B.~Yavuz, E.~Sparks, S.~Venkataraman, D.~Liu, J.~Freeman,
  D.~Tsai, M.~Amde, S.~Owen, D.~Xin, R.~Xin, M.~J. Franklin, R.~Zadeh,
  M.~Zaharia, and A.~Talwalkar.
\newblock {ML}lib: Machine learning in apache spark.
\newblock \emph{Journal of Machine Learning Research}, 17\penalty0
  (34):\penalty0 1--7, 2016.

\bibitem[Mota et~al.(2013)Mota, Xavier, Aguiar, and Puschel]{Mota:2013ja}
J.~F.~C. Mota, J.~M.~F. Xavier, P.~M.~Q. Aguiar, and M.~Puschel.
\newblock {D-ADMM: A Communication-Efficient Distributed Algorithm for
  Separable Optimization}.
\newblock \emph{IEEE Transactions on Signal Processing}, 61\penalty0
  (10):\penalty0 2718--2723, 2013.

\bibitem[Necoara(2015)]{necoara2015linear}
I.~Necoara.
\newblock Linear convergence of first order methods under weak nondegeneracy
  assumptions for convex programming.
\newblock \emph{arXiv.org}, 2015.

\bibitem[Necoara and Nedelcu(2014)]{necoara2014distributed}
I.~Necoara and V.~Nedelcu.
\newblock Distributed dual gradient methods and error bound conditions.
\newblock \emph{arXiv.org}, 2014.

\bibitem[Nesterov(2005)]{Nesterov:2005ic}
Y.~Nesterov.
\newblock {Smooth minimization of non-smooth functions}.
\newblock \emph{Mathematical Programming}, 103\penalty0 (1):\penalty0 127--152,
  2005.

\bibitem[Niu et~al.(2011)Niu, Recht, R{\'e}, and Wright]{Niu:2011wo}
F.~Niu, B.~Recht, C.~R{\'e}, and S.~J. Wright.
\newblock Hogwild!: A lock-free approach to parallelizing stochastic gradient
  descent.
\newblock In \emph{Neural Information Processing Systems}, 2011.

\bibitem[Pechyony et~al.(2011)Pechyony, Shen, and Jones]{Pechyony:2011wi}
D.~Pechyony, L.~Shen, and R.~Jones.
\newblock Solving large scale linear {SVM} with distributed block minimization.
\newblock In \emph{International Conference on Information and Knowledge
  Management}, 2011.

\bibitem[Qu et~al.(2015)Qu, Richt{\'a}rik, and Zhang]{qu2015quartz}
Z.~Qu, P.~Richt{\'a}rik, and T.~Zhang.
\newblock Quartz: Randomized dual coordinate ascent with arbitrary sampling.
\newblock In \emph{Neural Information Processing Systems}, 2015.

\bibitem[Qu et~al.(2016)Qu, Richt{\'a}rik, Tak{\'a}{\v c}, and
  Fercoq]{Qu:2015ve}
Z.~Qu, P.~Richt{\'a}rik, M.~Tak{\'a}{\v c}, and O.~Fercoq.
\newblock {SDNA}: Stochastic dual {N}ewton ascent for empirical risk
  minimization.
\newblock In \emph{International Conference on Machine Learning}, 2016.

\bibitem[Richt{\'a}rik and Tak{\'a}\v{c}(2016)]{richtarik2016distributed}
P.~Richt{\'a}rik and M.~Tak{\'a}\v{c}.
\newblock Distributed coordinate descent method for learning with big data.
\newblock \emph{Journal of Machine Learning Research}, 17:\penalty0 1--25,
  2016.

\bibitem[Rockafellar(1997)]{Rockafellar:1997ww}
R.~T. Rockafellar.
\newblock \emph{{Convex Analysis}}.
\newblock Princeton University Press, Princeton, NJ, 1997.

\bibitem[Shalev-Shwartz and Tewari(2011)]{ShalevShwartz:2011vo}
S.~Shalev-Shwartz and A.~Tewari.
\newblock {Stochastic methods for l$_{1}$-regularized loss minimization}.
\newblock \emph{Journal of Machine Learning Research}, 12:\penalty0 1865--1892,
  2011.

\bibitem[Shalev-Shwartz and Zhang(2013{\natexlab{a}})]{ShalevShwartz:2013wl}
S.~Shalev-Shwartz and T.~Zhang.
\newblock Stochastic dual coordinate ascent methods for regularized loss
  minimization.
\newblock \emph{Journal of Machine Learning Research}, 14:\penalty0 567--599,
  2013{\natexlab{a}}.

\bibitem[Shalev-Shwartz and Zhang(2013{\natexlab{b}})]{shalev2013accelerated}
S.~Shalev-Shwartz and T.~Zhang.
\newblock Accelerated mini-batch stochastic dual coordinate ascent.
\newblock In \emph{Neural Information Processing Systems}, 2013{\natexlab{b}}.

\bibitem[Shalev-Shwartz and Zhang(2014)]{ShalevShwartz:2014dy}
S.~Shalev-Shwartz and T.~Zhang.
\newblock {Accelerated proximal stochastic dual coordinate ascent for
  regularized loss minimization}.
\newblock \emph{Mathematical Programming}, Series A:\penalty0 1--41, 2014.

\bibitem[Shamir and Srebro(2014)]{Shamir:2014tp}
O.~Shamir and N.~Srebro.
\newblock {Distributed Stochastic Optimization and Learning}.
\newblock In \emph{Allerton Conference}, 2014.

\bibitem[Shamir et~al.(2014)Shamir, Srebro, and Zhang]{Shamir:2014vf}
O.~Shamir, N.~Srebro, and T.~Zhang.
\newblock Communication-efficient distributed optimization using an approximate
  newton-type method.
\newblock In \emph{International Conference on Machine Learning}, 2014.

\bibitem[Smith et~al.(2015)Smith, Forte, Jordan, and Jaggi]{Smith:2015ua}
V.~Smith, S.~Forte, M.~I. Jordan, and M.~Jaggi.
\newblock {L1-Regularized Distributed Optimization: A Communication-Efficient
  Primal-Dual Framework}.
\newblock \emph{arXiv.org}, 2015.

\bibitem[Smith et~al.(2017)Smith, Chiang, Sanjabi, and
  Talwalkar]{smith2017federated}
V.~Smith, C.-K. Chiang, M.~Sanjabi, and A.~S. Talwalkar.
\newblock Federated multi-task learning.
\newblock In \emph{Neural Information Processing Systems}, 2017.

\bibitem[Tak{\'a}{\v c} et~al.(2013)Tak{\'a}{\v c}, Bijral, Richt{\'a}rik, and
  Srebro]{Takac:2013ut}
M.~Tak{\'a}{\v c}, A.~Bijral, P.~Richt{\'a}rik, and N.~Srebro.
\newblock Mini-batch primal and dual methods for {SVMs}.
\newblock In \emph{International Conference on Machine Learning}, 2013.

\bibitem[Tappenden et~al.(2015)Tappenden, Tak{\'a}{\v{c}}, and
  Richt{\'a}rik]{Tappenden:2015vha}
R.~Tappenden, M.~Tak{\'a}{\v{c}}, and P.~Richt{\'a}rik.
\newblock On the complexity of parallel coordinate descent.
\newblock \emph{arXiv.org}, 2015.

\bibitem[Trofimov and Genkin(2014)]{Trofimov:2014vb}
I.~Trofimov and A.~Genkin.
\newblock Distributed coordinate descent for l1-regularized logistic
  regression.
\newblock \emph{arXiv.org}, 2014.

\bibitem[Trofimov and Genkin(2016)]{Trofimov:2016tu}
I.~Trofimov and A.~Genkin.
\newblock {Distributed Coordinate Descent for Generalized Linear Models with
  Regularization}.
\newblock \emph{arXiv.org}, 2016.

\bibitem[Wang and Lin(2014)]{wang2014iteration}
P.-W. Wang and C.-J. Lin.
\newblock Iteration complexity of feasible descent methods for convex
  optimization.
\newblock \emph{Journal of Machine Learning Research}, 15\penalty0
  (1):\penalty0 1523--1548, 2014.

\bibitem[Wright(2015)]{Wright:2015bn}
S.~J. Wright.
\newblock Coordinate descent algorithms.
\newblock \emph{Mathematical Programming}, 151\penalty0 (1):\penalty0 3--34,
  2015.

\bibitem[Yang(2013)]{Yang:2013vl}
T.~Yang.
\newblock Trading computation for communication: Distributed stochastic dual
  coordinate ascent.
\newblock In \emph{Neural Information Processing Systems}, 2013.

\bibitem[Yang et~al.(2013)Yang, Zhu, Jin, and Lin]{Yang:2013ui}
T.~Yang, S.~Zhu, R.~Jin, and Y.~Lin.
\newblock Analysis of distributed stochastic dual coordinate ascent.
\newblock \emph{arXiv.org}, Dec. 2013.

\bibitem[Yen et~al.(2015)Yen, Lin, and Lin]{Yen:2015vy}
I.~E.-H. Yen, S.-W. Lin, and S.-D. Lin.
\newblock A dual augmented block minimization framework for learning with
  limited memory.
\newblock In \emph{Neural Information Processing Systems}, 2015.

\bibitem[Yu et~al.(2012)Yu, Hsieh, Chang, and Lin]{Yu:2012fp}
H.-F. Yu, C.-J. Hsieh, K.-W. Chang, and C.-J. Lin.
\newblock Large linear classification when data cannot fit in memory.
\newblock \emph{ACM Transactions on Knowledge Discovery from Data}, 5\penalty0
  (4):\penalty0 1--23, 2012.

\bibitem[Yu et~al.(2010)Yu, Vishwanathan, G{\"u}nter, and
  Schraudolph]{Yu:2010vw}
J.~Yu, S.~Vishwanathan, S.~G{\"u}nter, and N.~N. Schraudolph.
\newblock A quasi-{N}ewton approach to nonsmooth convex optimization problems
  in machine learning.
\newblock \emph{Journal of Machine Learning Research}, 11:\penalty0 1145--1200,
  2010.

\bibitem[Yuan et~al.(2010)Yuan, Chang, Hsieh, and Lin]{Yuan:2010ub}
G.-X. Yuan, K.-W. Chang, C.-J. Hsieh, and C.-J. Lin.
\newblock A comparison of optimization methods and software for large-scale
  l1-regularized linear classification.
\newblock \emph{Journal of Machine Learning Research}, 11:\penalty0 3183--3234,
  2010.

\bibitem[Yuan et~al.(2012)Yuan, Ho, and Lin]{Yuan:2012wi}
G.-X. Yuan, C.-H. Ho, and C.-J. Lin.
\newblock An improved {GLMNET} for {L}1-regularized logistic regression.
\newblock \emph{Journal of Machine Learning Research}, 13:\penalty0 1999--2030,
  2012.

\bibitem[Zhang et~al.(2012)Zhang, Lee, and Shin]{Zhang:2012bp}
C.~Zhang, H.~Lee, and K.~G. Shin.
\newblock Efficient distributed linear classification algorithms via the
  alternating direction method of multipliers.
\newblock In \emph{International Conference on Artificial Intelligence and
  Statistics}, 2012.

\bibitem[Zhang and Lin(2015)]{Zhang:2015vj}
Y.~Zhang and X.~Lin.
\newblock Stochastic primal-dual coordinate method for regularized empirical
  risk minimization.
\newblock In \emph{International Conference on Machine Learning}, 2015.

\bibitem[Zhang et~al.(2013)Zhang, Duchi, and Wainwright]{Zhang:2013wq}
Y.~Zhang, J.~C. Duchi, and M.~J. Wainwright.
\newblock Communication-efficient algorithms for statistical optimization.
\newblock \emph{Journal of Machine Learning Research}, 14:\penalty0 3321--3363,
  2013.

\bibitem[Zheng et~al.(2017)Zheng, Wang, Xia, Xu, and Zhang]{zheng2017general}
S.~Zheng, J.~Wang, F.~Xia, W.~Xu, and T.~Zhang.
\newblock A general distributed dual coordinate optimization framework for
  regularized loss minimization.
\newblock \emph{Journal of Machine Learning Research}, 18:\penalty0 1--52,
  2017.

\bibitem[Zinkevich et~al.(2010)Zinkevich, Weimer, Smola, and
  Li]{Zinkevich:2010tj}
M.~A. Zinkevich, M.~Weimer, A.~J. Smola, and L.~Li.
\newblock Parallelized stochastic gradient descent.
\newblock \emph{Neural Information Processing Systems}, 2010.

\end{thebibliography}

\end{document}